\documentclass[twoside]{article}

\usepackage[accepted]{aistats2024}

\usepackage[round]{natbib}

\usepackage{amsmath}
\usepackage{amssymb}
\usepackage{mathtools}
\usepackage{amsthm}
\usepackage{microtype}
\usepackage{graphicx}
\usepackage{subfigure}
\usepackage{booktabs} 
\usepackage{xcolor}

\usepackage{algorithm}
\usepackage{algpseudocode}

\usepackage[hidelinks]{hyperref}
\usepackage{comment}

\newcommand{\muh}{\hat{\mu}}
\newcommand{\pih}{\hat{\pi}}
\newcommand{\psiu}{\psi^{u}}
\newcommand{\psil}{\psi^{l}}

\newcommand{\lb}{\text{l}}
\newcommand{\ub}{\text{u}}

\newcommand{\varphih}{\hat{\varphi}}
\newcommand{\varphil}{\varphi^{\lb}}
\newcommand{\varphiu}{\varphi^{\ub}}

\newcommand{\dhl}{\hat{d}_l}
\newcommand{\dhu}{\hat{d}_u}
\newcommand{\dl}{d_{l}}
\newcommand{\du}{d_{u}}
\renewcommand{\d}{d}
\renewcommand{\dh}{\hat{d}}

\newcommand{\thetal}{\theta^{\lb}}
\newcommand{\thetau}{\theta^{\ub}}
\newcommand{\thetahl}{\hat{\theta}^{\lb}}
\newcommand{\thetahu}{\hat{\theta}^{\ub}}
\newcommand{\thetah}{\hat{\theta}}

\newcommand{\lambdal}{\lambda^{\lb}}
\newcommand{\lambdau}{\lambda^{\ub}}

\newcommand{\psihl}{\hat{\psi}^{\lb}}
\newcommand{\psihu}{\hat{\psi}^{\ub}}
\newcommand{\psih}{\hat{\psi}}
\newcommand{\Ph}{\hat{P}}
\renewcommand{\P}{P}
\newcommand{\abs}[1]{\left\lvert#1\right\rvert}
\newcommand{\norm}[1]{\left\|#1\right\|}

\usepackage[capitalize,noabbrev]{cleveref}
\usepackage{thmtools}
\usepackage{tikz}
\usepackage{tikz-qtree}
\usetikzlibrary{bayesnet,arrows,shapes.arrows,shapes.geometric,shapes.multipart,decorations.pathmorphing,positioning,shapes.swigs,}

\theoremstyle{plain}

\newtheorem{proposition}{Proposition}

\newtheorem{lemma}{Lemma}

\theoremstyle{definition}
\newtheorem{definition}{Definition}
\newtheorem{assumption}{Assumption}
\theoremstyle{remark}

\Crefname{assumption}{Assumption}{Assumptions}
\crefname{assumption}{Assumption}{Assumptions}
\Crefname{equation}{}{}
\crefname{equation}{}{}

\begin{document}

\twocolumn[

\aistatstitle{Auditing Fairness under Unobserved Confounding}

\aistatsauthor{ Yewon Byun \And Dylan Sam \And  Michael Oberst \And Zachary C. Lipton \And Bryan Wilder }
\vspace{3mm}
\aistatsaddress{ Machine Learning Department, Carnegie Mellon University} ]

\begin{abstract}

Many definitions of fairness or inequity involve unobservable causal quantities that cannot be directly estimated without strong assumptions. For instance, it is particularly difficult to estimate notions of fairness that rely on hard-to-measure concepts such as risk (e.g., quantifying whether patients at the same risk level have equal probability of treatment, regardless of group membership).
Such measurements of risk can be accurately obtained when no unobserved confounders have jointly influenced past decisions and outcomes. However, in the real world, this assumption rarely holds. 
In this paper, we show that, surprisingly, one can still compute meaningful bounds on treatment rates for high-risk individuals (i.e., conditional on their true, \textit{unobserved} negative outcome), even when entirely eliminating or relaxing the assumption that we observe all relevant risk factors used by decision makers. We use the fact that in many real-world settings (e.g., the release of a new treatment) we have data from prior to any allocation to derive unbiased estimates of risk. 
This result enables us to audit unfair outcomes of existing decision-making systems in a principled manner.
We demonstrate the effectiveness of our framework with a real-world study of Paxlovid allocation, provably identifying that observed racial inequity cannot be explained by unobserved confounders of the same strength as important observed covariates.

\end{abstract}

\section{INTRODUCTION}

A fundamental problem in the outcomes of decision-making systems across a variety of domains, such as healthcare, housing assistance, and the criminal justice system, is the presence of inequities across demographic lines \citep{nelson2002unequal, tonry2010social, artiga2020disparities, buchmueller2020aca, Shinn2022AllocatingHS}.
To reduce these unfair outcomes, it is essential that we are first able to identify and quantify them appropriately.
We primarily consider settings where we desire a resource to be allocated at equal rates (across groups) to those who would otherwise experience adverse events: a formalization of the idea that we want to allocate to ``high-risk'' individuals.
In healthcare, these members could be individuals who would die without treatment, or in housing, individuals who would become homeless if not provided housing assistance.
We refer to the allocation rate to these types of individuals as the ``treatment rate among the needy'' (see~\cref{def:treatment_rate_among_the_needy}).\footnote{There is a wealth of literature on causal measures of fairness, and our chosen metric, when used to quantify inequity, can be seen as a special case of \textit{counterfacutal equalized odds}, or more specifically, the ``opportunity rate'' as defined in Definition 4.2 of \citet{mishler2021fairness}.} 
Yet, this notion is counterfactual and difficult to measure---once an individual is treated, we cannot say what would have happened had they been denied treatment.

Equity, quantified in these terms, can be estimated from data, but \textit{only} if we observe all confounders---variables that jointly influence both the decision to allocate and the outcome under no allocation. 
Our work fits in the broader literature on causal fairness, a literature that has produced a variety of causality-informed measures of equity, which we further discuss in Section \ref{sec:fairness}. 
Throughout the literature, it is frequently assumed that all confounders \textit{are} observed, permitting the identification of causal fairness measures from data \citep{kusner2017counterfactual, mishler2021fairness, nilforoshan2022causal}.\footnote{There are exceptions to this pattern---\citet{rambachan2022counterfactual}, for instance, is closer in spirit to our work, proposing a sensitivity analysis framework for causal fairness metrics under unobserved confounding. We discuss differences to our approach in detail in~\cref{sec:fairness}.}

\begin{figure*}[t]
    \centering
    \includegraphics[width=.95\textwidth]{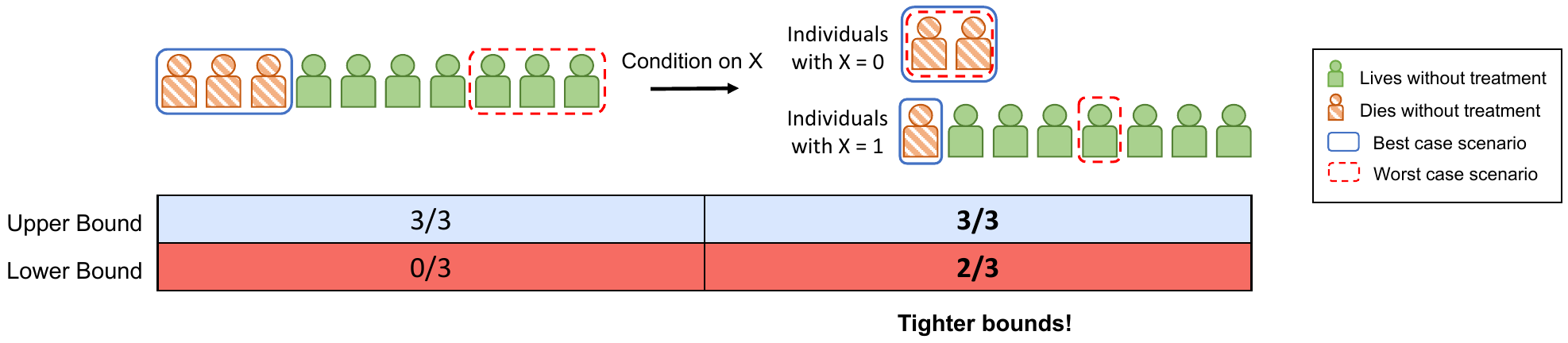}
    \vspace{-3mm}
    \caption{A conceptual figure to build intuition for our main partial identification result (Theorem \ref{theorem:no_assumption}). For simplicity, let the treatment rate $P(T=0)=0.3$ and mortality rate $P(Y(0)=1)=0.3$. We denote upper (blue) / lower (red) bounds on treatment rate among the needy $P(T=1 | Y(0)=1, X)$. In the marginal case (left), we see that the bounds are vacuous.  However, when we exploit covariate information (right) (e.g., selecting 2 samples from X=0 and 1 sample from X=1), we observe bounds that are much tighter.}
    \label{fig:intuition}
\end{figure*}

Yet, in the real world, this assumption rarely holds. In reality, resources are often allocated based on indicators of need or risk that we do not observe in our datasets (``unobserved confounders''), which could lead us to understate or overstate the amount of inequity.
On the one hand, similar rates of allocation across groups could mask inequity, given unobserved differences in need. On the other hand, these differences could explain apparent inequities in allocation across groups. 

To make progress in the face of unobserved confounding, we use the fact that in many real-world settings, we have data from a period where no individuals received resources (e.g., prior to a new drug entering the market, or a similar region where housing assistance is unavailable).  Such data allows us to derive unbiased estimates of what would happen to individuals without the allocation of resources, under an assumption that this baseline risk generalizes to the setting where resources are available. 
Unfortunately, if unobserved confounders exist, we still cannot exactly identify the treatment rate of needy individuals.

In this setting, we show that, interestingly, one can derive meaningful bounds on the treatment rate among the needy, without any assumptions on the strength of unobserved confounders. \Cref{fig:intuition} builds intuition for this result, which is given in~\cref{theorem:no_assumption}. We provide bias-corrected estimators for our bounds that are consistent and asymptotically normal, and extend recent results in the partial identification literature to handle the non-smooth nature of our estimators (\cref{prop:normality-max}) that make them more precise. As a result, our bounds can incorporate machine learning (ML) estimators that converge at slower than parametric rates, while retaining the benefits of asymptotic normality (e.g., confidence intervals). Finally, we derive even tighter bounds that incorporate assumptions on the plausible strength of unobserved confounding with a sensitivity model (\cref{theorem:identification_gamma}) and corresponding estimators that attain similar asymptotic properties to those discussed above.

These results are of immediate practical interest, allowing us to audit unfair outcomes in existing decision-making systems, while accounting for the effects of potential unobserved confounders.
We demonstrate the effectiveness of our framework in a real-world case study, to audit inequities in the allocation of Paxlovid to COVID-19 patients.
Our framework provably identifies and quantifies racial inequity 
in Paxlovid allocation, even if unobserved confounders had effects (on decisions and outcomes) similar to that of 
important observed covariates.
Finally, we provide empirical comparisons on both semi-synthetic and synthetic data, demonstrating that our framework more accurately and more efficiently (i.e., requiring less data) captures ground truth treatment rates, improving upon alternative methods.
In short, our work provides principled conditions under which machine learning estimators can be used as a tool to both identify and quantify inequity in existing decision-making systems for important limited resources.

\section{RELATED WORK}

\textbf{Fairness and Causality}
\label{sec:fairness}
The literature on fairness and decision-making is vast, and we will not claim to summarize it here. Of particular relevance to our work is the literature on \textit{causal fairness} \citep{kilbertus2017avoiding, plecko2022causal, loftus2018causal}, where fairness metrics are defined with respect to e.g., counterfactual outcomes.  

Even in this sub-literature, there are a wealth of ways to characterize fairness, such as counterfactual fairness\footnote{See \citet{nilforoshan2022causal}, \citet{zhang2018fairness}, and Section 4.4.1 of \citet{plecko2022causal} for discussions of the nuances of various definitions of counterfactual fairness.} \citep{kusner2017counterfactual, nilforoshan2022causal, wu2019counterfactual}, counterfactual equalized odds \citep{mishler2021fairness}, and so on \citep{khademi2019fairness}.  
Another line of work, with some overlaps with the concepts above, studies fairness along certain paths in causal graphs \citep{nabi2018fair, schroder2023causal}.
Our choice of metric is similar in spirit to counterfactual equalized odds and is precisely equivalent to the notion of ``opportunity rate'' given by \citet{mishler2021fairness} (see Def. 4.2 in their paper).

There are a variety of research directions pursued in the causal fairness literature, such as learning predictive models that lead to fair decisions, or giving conditions under which various notions of causal fairness can be ``identified'' from data---that is, written in terms of the observed distribution, instead of counterfactual outcomes.  Our work is most similar to a nascent line of work considering scenarios where these measures \textit{cannot} be identified, but can nonetheless be bounded. 
For instance, \citet{kilbertus2020sensitivity} measures the sensitivity of counterfactual fairness to misspecified causal structures via unmeasured confounding.
Perhaps closest to our work is that of \citet{rambachan2022counterfactual}, who provide bounds on causal fairness measures under a different sensitivity analysis framework.  Their framework assumes a bound on the differences in conditional means of the potential outcome, whereas our results are derived from bounds on treatment propensities (see~\cref{sec:sensitivity}). 
However, the most notable distinction from their work is that we derive bounds that partially identify inequity \textit{without} any assumptions on the strength of confounders, under our assumption on the availability of pre- and post-availability data (see~\cref{sec:preliminaries}).

\textbf{Partial Identification and Sensitivity Analysis} 
In statistics and econometrics, \textit{partial identification} refers to the derivation of bounds on causal quantities when the exact value cannot be identified from assumptions \citep{manski2003partial}. \textit{Sensitivity analysis} often refers to the derivation of bounds under assumptions about the ``strength'' of unobserved confounding.  
Sensitivity analysis has been pursued under a variety of models, dating back to \citet{Cornfield1959-4f}.
We will not attempt to summarize the literature here, except to note a few ideas that we draw upon. 
First, one insight in our analysis is that incorporating covariate information can improve the tightness of our bounds, an insight similarly leveraged in recent work \citep{yadlowsky2018bounds, levis2023covariate}. Second, our sensitivity model can be viewed as a variant of the sensitivity model introduced by \citet{tan2006distributional}, although our causal quantity of interest differs substantially, requiring the derivation of novel bounds.  Finally, we draw inspiration from the sensitivity analysis literature to assess the plausibility of our sensitivity parameters via an informal comparison to the strength of observed confounders \citep{frank2000benchmarking,Hsu2013-16}.

\section{PRELIMINARIES}\label{sec:preliminaries}

\paragraph{Notation} 
We use upper-case letters to denote random variables (e.g., $X$), and lower-case letters to denote their realizations (e.g., $x$). We use $X \in \mathcal{X}$ to denote covariates, $T \in \{0, 1\}$ to denote treatment, and $Y \in \{0, 1\}$ to denote a binary outcome. We let $Y = 1$ denote an adverse outcome (e.g., mortality), and $Y = 0$ denote a benign outcome (e.g., survival). We additionally define the potential outcome $Y(0)$ as the outcome of each individual without treatment.

Given our interest in identifying inequity across different subpopulations, we use $G \in \{1, \ldots, K\}$ to denote subpopulation membership, where group membership is a known function of $X$. We define some quantities (e.g., \cref{def:treatment_rate_among_the_needy} and some bounds) in terms of the overall population for simplicity of notation, where the extension to group-wise quantities is straightforward.

We further use $D \in \{0, 1\}$ to denote whether a sample belongs to \textit{pre-availability} or \textit{post-availability} data. 
Pre-availability data ($D = 0$) is drawn from a setting where the resource is not available (e.g., a time period before a drug entered the market) and post-availability data ($D = 1$) from a setting where the resource is available. 
We consider our data to be drawn from a common distribution $P$, where $P(\cdot \mid D = 0)$ and  $P(\cdot \mid D = 1)$ denote pre- and post-availability distributions respectively.

\paragraph{Availability of Treatment and Outcome Data} 
During the pre-availability period ($D = 0$), the treatment is not available by definition, and so $D = 0 \implies T = 0$.  In post-availability data (where $D = 1$), we do not assume access to outcome data---for instance, the outcome of interest may be a long-term outcome not immediately measurable in the post-availability period. As a result, when $D = 1$, we set $Y$ to an arbitrary value.  Because $T$ is fixed when $D = 0$, and because $Y$ is unknown when $D = 1$, we observe data as follows
\begin{equation*}
(X, T, Y) =
\begin{cases}
   (X, 0, Y) & \text{ if $D = 0$ (pre-availability)} \\
   (X, T, \sim) & \text{ if $D = 1$ (post-availability)}
\end{cases}
\end{equation*}
where $\sim$ indicates that $Y$ is not observed.

\section{ANALYSIS OF INEQUITY}

\subsection{Equity in Treatment Allocation}
We first define a notion of effective allocation.
\begin{definition}[Treatment Rate Among the Needy]\label{def:treatment_rate_among_the_needy}
\begin{equation}\label{eq:treatment_rate_among_the_needy}
P(T = 1 \mid Y(0) = 1, D = 1)
\end{equation}
\end{definition}
\Cref{eq:treatment_rate_among_the_needy} captures the proportion of individuals who receive treatment when it is available ($D = 1$), among those who would experience an adverse event $Y(0) = 1$ if they do not receive treatment. This is similar to the notion of ``opportunity rate" in the work of \citet{mishler2021fairness}.
\Cref{def:treatment_rate_among_the_needy} suggests a measure of inequity in treatment allocation, when applied to specific subgroups.

\begin{definition}\label{def:inequity}
We define \textit{inequity in treatment rate among the needy} for a pair of subpopulations $g \neq g'$ as 
\begin{align}
    &|P(T = 1 | Y(0) = 1, D = 1, G = g) \nonumber \\
    &\qquad - P(T = 1 | Y(0) = 1, D = 1, G = g')\label{eq:target}|
\end{align}
\end{definition}

When $Y$ corresponds to mortality, \cref{def:inequity} captures the notion that patients who would die if treatment were withheld should receive a potentially life-saving intervention at equal rates across subgroups $G$.\footnote{We note that one could define other metrics based on potential outcomes, such as seeking equal allocation across individuals who would not only die if treatment were withheld, but who would also survive if given treatment, e.g., $P(T = 1 \mid Y(0) = 1, Y(1) = 0)$.  However, for novel treatments (like Paxlovid in our example) treatment guidelines often focus on treating high-risk patients in the absence of definitive evidence that some patients have substantially different responses to treatment. Moreover, estimating or bounding such quantities would require substantially stronger assumptions than those presented here.}

\subsection{Identification under Strong Assumptions}

We note that it is not possible to directly estimate~\cref{eq:treatment_rate_among_the_needy} and~\cref{eq:target}. In the post-availability setting, we never observe the outcome $Y$. Further, we can never simultaneously observe $T = 1$ and $Y(0)$, which necessitates the use of strong additional assumptions to re-write this quantity in terms of quantities that we do observe.
For instance, one could estimate~\cref{eq:treatment_rate_among_the_needy} directly if one were willing to assume that $X$ captures all variables that influence both treatment assignment and $Y(0)$, often referred to as the assumption of no unmeasured confounding.
\begin{assumption}[No Unmeasured Confounding]\label{asmp:no_unmeasured_confounding}
The untreated outcome is independent of treatment in the post-availability period, given observed covariates, i.e., 
\begin{equation*}
Y(0) \perp T \mid X, D = 1
\end{equation*}
\end{assumption}
A main thesis of this work is that~\cref{asmp:no_unmeasured_confounding} may not be realistic in most real-world settings. \Cref{asmp:no_unmeasured_confounding} is violated if treatment is allocated on the basis of variables other than $X$ (e.g., unobserved financial means), which in turn provide more information on how likely a patient is to experience an adverse outcome without treatment. Given that this assumption is often violated in practice, we will later discuss bounding~\cref{eq:treatment_rate_among_the_needy} under a relaxed version of~\cref{asmp:no_unmeasured_confounding} (\cref{sec:sensitivity}), and even in the case where we drop~\cref{asmp:no_unmeasured_confounding} entirely (\cref{sec:no_assumption_main}).

For now, we state assumptions relating the pre- and post-availability periods, which we maintain throughout.

\begin{assumption}[Consistency]\label{asmp:consistency_pretreatment}
In pre-availability data, we directly observe the untreated potential outcome, i.e., $D = 0 \implies Y = Y(0)$.
\end{assumption}
\Cref{asmp:consistency_pretreatment} is analogous to the assumption of consistency in causal inference and captures the fact that no treatment is available in the pre-availability period, so all outcomes are untreated outcomes by definition.

\begin{assumption}[Covariate Stability]\label{asmp:covariate_stability}
Within each subgroup, the distribution of covariates of needy patients is the same across pre- and post-availability periods, i.e., 
\begin{equation*}
    X \perp D \mid Y(0) = 1, G
\end{equation*}
\end{assumption}
In other words, the observable characteristics $X$ of needy patients (where $Y(0) = 1$) are distributed the same across the pre- and post-availability periods (e.g., needy patients are young/old at same rates across the two periods).

Given these assumptions, one can directly identify the treatment rate among the needy from~\cref{eq:treatment_rate_among_the_needy}.

\begin{restatable}{proposition}{Identification}\label{prop:identification}
Under~\cref{asmp:no_unmeasured_confounding,asmp:consistency_pretreatment,asmp:covariate_stability}, \cref{eq:treatment_rate_among_the_needy} (conditioned on $G$) can be written as the following functional of the observed distribution $P$
\begin{align}\label{eq:identification}
&P(T = 1 \mid Y(0) = 1, D = 1, G = g) \\
&\quad = E[P(T = 1 \mid X, D = 1) \mid Y = 1, D = 0, G = g] \nonumber 
\end{align}
\end{restatable}
This result (with some notational differences) is a known fact in the literature \citep{coston2020counterfactual, mishler2021fairness}. For completeness, we provide the proof in~\cref{sec:identification}.
Notably, as discussed above, this result requires the hard-to-justify assumption that there are no unmeasured confounding variables (\cref{asmp:no_unmeasured_confounding}). In the following sections, we develop bounds under different relaxations of this assumption.

\subsection{Partial Identification under Arbitrary Unmeasured Confounding}\label{sec:no_assumption_main}
To begin, we consider the case where we consider~\cref{asmp:no_unmeasured_confounding} to be unrealistic and drop it entirely.  We demonstrate that it is still possible to obtain informative bounds on the treatment rate in~\cref{eq:treatment_rate_among_the_needy} that can be estimated from data, and provide intuition as to why informative bounds are possible to obtain.  We proceed under one additional assumption linking the pre- and post-availability periods. 
\begin{assumption}[Stable Baseline Risk]\label{asmp:stable_basline_risk}
Across the pre- and post-availability periods, the conditional baseline risk (i.e., the risk of an adverse outcome without treatment) does not change, i.e.,
\begin{equation*}
    Y(0) \perp D \mid X
\end{equation*}
\end{assumption}
This is the key assumption relating the pre- and post-availability periods, which allows us to estimate the baseline risk in the post-availability period, by leveraging data from the pre-availability period.
We expect it to be satisfied when the underlying mechanistic or biological determinants of risk are unchanged around the time period a treatment was introduced. To build further intuition, we depict a causal graph (\cref{fig:example_causal_graph}) where~\cref{asmp:covariate_stability,asmp:stable_basline_risk} hold, but no unmeasured confounding (\cref{asmp:no_unmeasured_confounding}) fails to hold.

To build intuition for our first main result, consider an extreme case where \textit{all} individuals will die without treatment. Here, the treatment rate among the needy is simply given by the observed treatment rate. Our result gives informative bounds in less extreme scenarios: For instance, suppose we knew that out of 100 patients, 90 would die if untreated, and that we have treated 50 patients. In this case, the worst-case scenario is that we have treated all 10 patients who would not die if untreated, but we must have treated at least 40 of the patients who would die if untreated. 

To begin to formalize this idea, consider the simplified scenario where $Y(0) \perp D$, a stronger version of~\cref{asmp:stable_basline_risk}. Further, observe that $P(T = 1, Y(0) = 1 \mid D = 1) \leq P(T = 1 \mid D = 1)$, which yields  
\begin{align*}
    P(T = 1|Y(0) = 1, D = 1) \leq \frac{P(T = 1 \mid D = 1)}{P(Y(0) = 1 \mid D = 0)}
\end{align*}
where we switch $D = 1$ for $D = 0$ in the denominator due to the assumed independence.\footnote{See Appendix \ref{sec:no_assumption} for the complete derivation of the upper bound, as well as the derivation of the lower bound.} Note that $P(T = 1 \mid D = 1)$ and $P(Y(0) = 1 \mid D = 0)$ are both observable (in the post- and pre- periods respectively). Unfortunately, this bound may be vacuous on its own (e.g. if $P(Y(0) = 1)$ is small). We can sharpen it by noting that the same inequality holds at every value of the covariates $X$, under~\cref{asmp:stable_basline_risk}, such that
{
\begin{align*}
    P(T = 1|Y(0) = 1, D = 1, X) \leq \frac{P(T = 1 \mid X, D = 1)}{P(Y(0) = 1 \mid X, D = 0)}
\end{align*}
}

Now, given calibrated classifiers for the treatment and outcome (to estimate the numerator and denominator), the bound will become tighter. Averaging over the appropriate distribution for $X$ then yields a tighter overall bound. To further build intuition, see~\cref{fig:intuition}.  This idea is formalized in the following theorem.\footnote{We also provide an alternative expression of the bounds in Appendix \ref{sec:alt_no_assumption}, clarifying how these bounds can be viewed as a weighted average of the conditional bounds over $X$, to build further intuition.}

\begin{restatable}[Bounds under arbitrary unmeasured confounding]{theorem}{Bounds}
\label{theorem:no_assumption}
Consider the setting described in~\cref{sec:preliminaries}. Under Assumptions \ref{asmp:consistency_pretreatment}, \ref{asmp:covariate_stability} and \ref{asmp:stable_basline_risk}, and if $P(Y(0) = 1 \mid D = 1, X = x) > 0$, then
\begin{equation*}
 \psi^l \leq P(T = 1 | Y(0) = 1 , D = 1) \leq \psi^u
\end{equation*}
where
\begin{align*}
    &\psi^l \coloneqq E[\max\{\theta_1^l(X), \theta_2^l(X)\}] \quad \psi^u \coloneqq E[\min\{\theta_1^u(X), \theta_2^u(X)\}]\\
    & \theta_1^l(X) \coloneqq \frac{P(D=0 | X)}{P(Y=1, D=0)}\Big( P(T = 1 | D=1, X) \\
  & \qquad \qquad \qquad + P(Y=1 | D=0, X) - 1 \Big) \\  
  &\theta_2^l(X) \coloneqq 0 \\
  &\theta_1^u(X) \coloneqq \frac{P(D=0 | X)P(T = 1| D=1, X)}{P(Y=1, D=0)} \\
  & \theta_2^u(X) \coloneqq \frac{P(D=0 | X) P(Y =1 | D=0, X)}{P(Y = 1, D=0)}
\end{align*}
\end{restatable}

The proof is given in Appendix \ref{sec:no_assumption}. At a high level, 
the max/min structure in each bound ($\psi^l, \psi^u$) arises from the complementary fact that the probabilities are lower- and upper-bounded by both a quantity we develop and by 0 and 1 (respectively).
The max/min takes the tighter of these two bounds at every level of the covariates. 
The underlying intuition for our result is that we incorporate information about $X$ in our bound, and estimate these quantities (e.g., $\theta_1^l(X), \theta_2^l(X), \theta_1^u(X), \theta_2^u(X)$) with machine learning (ML) models. With these estimates, we compute the expectation over the conditional distribution over $X$ to produce our bounds $\psi^l$ and $\psi^u$. This results in a tighter bound, when compared to using population-level bounds, e.g., only looking at the marginals over $T$ and $Y$. These bounds are likely to be meaningful (non-vacuous) when treatment rates are relatively lower than rates of adverse untreated outcomes (for the upper bound) and when the rates of treatment and adverse untreated outcomes are not too low (for the lower bound).

We remark that Theorem \ref{theorem:no_assumption} expresses our upper and lower bounds only in terms of functions of the observed data, i.e., potential outcomes do not appear in the expression. This establishes that the bounds are identified from the observed data (and without any assumption on the presence of confounders).

\begin{figure}[t]
\begin{center}
\begin{tikzpicture}[
  obs/.style={circle, draw=gray!90, fill=gray!30, very thick, minimum size=5mm}, 
  uobs/.style={circle, draw=gray!90, fill=gray!10, dotted, minimum size=5mm}, 
  bend angle=30]
  \node[obs] (X) {$X$} ;
  \node[obs] (Y) [right=of X] {$Y$} ;
  \node[obs] (T) [above=of Y]  {$T$};
  \node[uobs] (U) [above=of X] {$C$} ;
  \node[obs] (D) [left=of U]  {$D$};
  \draw[-latex, thick] (X) -- (Y);
  \draw[-latex, thick] (X) -- (T);
  \draw[-latex, thick] (D) to[bend left] (T);
  \draw[-latex, thick] (T) -- (Y);
  \draw[-latex, thick] (U) -- (Y);
  \draw[-latex, thick] (U) -- (X);
  \draw[-latex, thick] (U) -- (T);
\end{tikzpicture}
\end{center}
\caption{A causal graph consistent with Assumptions~\ref{asmp:covariate_stability} and~\ref{asmp:stable_basline_risk}, even given unobserved (light gray) confounders $C$. Dark gray variables are observed. This causal structure is sufficient, but not necessary, for our assumptions to hold: See~\cref{app:causal_swig} for more details.}%
\label{fig:example_causal_graph}
\end{figure}
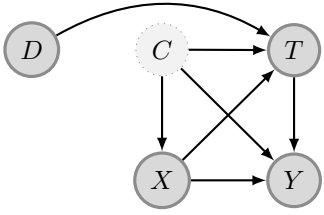

\paragraph{Estimation of Partial Identification Bounds}

Given the identification results in \cref{theorem:no_assumption}, we are ready to construct estimators of the upper and lower bounds $\psi^u, \psi^l$.  We define the following short-hand for the relevant conditional distributions
\begin{align}
\mu(X) &\coloneqq E[Y \mid D = 0, X ] \\
\pi(X) &\coloneqq E[T \mid D = 1, X ]\\
g(X) &\coloneqq E[D = 0 \mid X ]
\end{align}
These conditional expectations can be estimated by training classifiers $\muh, \pih$ on the pre- and post-availability data respectively, and $\hat{g}$ to distinguish the two. These are referred to as ``nuisance functions'', quantities that we have to estimate as part of estimating our bounds, but which are not of intrinsic interest. The simplest strategy would be to ``plug-in'' such estimators wherever the corresponding conditional expectation appears in the expression for $\psi^l$ or $\psi^u$. However, it is difficult to provide guarantees for this plug-in estimator, as ML models generally converge slower than $O(n^{-\frac{1}{2}})$, creating substantial bias in our estimate of the bound. 

Our proposed method is instead based on influence functions and semiparametric estimation to find and subtract a first-order approximation to the bias of the plug-in estimator. 
The corresponding estimators will then converge at $O(n^{-\frac{1}{2}})$ rates even if the ML models converge more slowly (as nonparametric methods typically will) and enable us to give valid confidence intervals based on asymptotic normality. 

To present our estimator for the upper bound, define $d^u(x) \in \arg \min\{\theta_1^u(x), \theta_2^u(x)\}$ to be the identity of a bound achieving the minimum value at $x$, with $\hat{d}^u(x)$ being the same quantity estimated from the plugin estimate of $\hat{\theta}$. Our proposed estimator is
\begin{align}
    &\varphi^u(P, d) \coloneqq \theta_{d(X)}^u(X) + \lambda^u_{d(X)}(X, Y, D, T) \label{eq:bias_corrected_upper} \\
    &\hat{\psi}^u(\hat{P}) \coloneqq E_{\hat{P}}\left[\varphi(\hat{P}, \hat{d}^u)\right] \nonumber
\end{align}
where we will employ the common strategy of estimating the expectations and the nuisance functions in $\hat{\theta}$ and $\hat{\lambda}$ on independent samples (averaging over $K$-fold cross-validation). We now present the detailed construction and analysis of this estimator, including the bias-correction term $\lambda$ and asymptotic guarantees. Our estimator for the lower bound is defined analogously, using an $\arg\max$.

\paragraph{Bias-corrected estimators}
The influence function can be used to provide a first-order approximation to the bias of the estimator; correcting for this bias will (hopefully) leave a remainder that depends in second order on error of the nuisance functions. Let $\theta_1^u = E[\theta_1^u(X)]$ and $\theta_2^u = E[\theta_2^u(X)]$. We now derive the influence functions corresponding to $\theta_1^u$ and $\theta_2^u$.

\begin{restatable}{lemma}{InfluenceUpper}\label{lemma:influence_upper}
    The influence functions for $\theta_1^u$ and $\theta_2^u$ are given by
    \small
    \begin{align*}
        & IF(\theta_1^u) = \\
        & \frac{1}{P(Y=1, D=0)} \Big(-\frac{1[Y=1, D=0]}{P(Y=1, D=0)} E_P[g(X)\pi(X)] + \\
        & \qquad + g(X)\pi(X) + 1[D=1](T - \pi(X)) \frac{g(X)}{1 - g(X)} + \\
        & \qquad + \pi(X)(1[D=0] - g(X))\Big) \\
        & IF(\theta_2^u) = \frac{1[D=0]}{P(Y=1, D=0)}\Big( \mu(X) \\
        & \qquad - E[\mu(X) | D=0] \left( \frac{1[Y=1]}{P(Y=1|D=0)}\right) + (Y - \mu(X)) \Big).
    \end{align*}
    \normalsize
\end{restatable}
To obtain first-order bias-corrected estimators, we set $\lambda^u$ in~\cref{eq:bias_corrected_upper} to the values
\begin{align*}
    &\lambda_1^u = IF(\theta_1^u) & \text{and} &  &\lambda_2^u = IF(\theta_2^u).
\end{align*}

Similarly, let $\theta_1^l = E[\theta_1^l(X)]$ and  $\theta_2^l = E[\theta_2^l(X)]$. Then, we can derive their influence functions as follows

\begin{restatable}{lemma}{InfluenceLower}\label{lemma:influence_lower}
    The influence functions for $\theta_1^l$ and $\theta_2^l$ are given by
    \small
    \begin{align*}
        & IF(\theta_1^l) = IF(\theta_1^u) +  \frac{1}{P(Y=1,D=0)} \Big( \frac{-1[Y=1,D=0]}{P(Y=1,D=0)} \\
        & \qquad \qquad \qquad \cdot E_{P}[g(X)(\mu(X) - 1)] + 1[D=0](Y-1) \Big)\\
        & IF(\theta_2^l) = 0.
    \end{align*}
    \normalsize
\end{restatable}
To obtain first-order bias-corrected estimators, we set
\begin{align*}
    &\lambda_1^l= IF(\theta_1^l) & \text{and} &  &\lambda_2^l = 0.
\end{align*}
We leave the proofs of the influence functions and bias-corrected estimators of our upper and lower bounds to Appendix~\ref{appx:influence_upper} and~\ref{appx:influence_lower}, respectively.

\paragraph{Asymptotics for nonsmooth bounds}

Providing inferential guarantees for our final bounds is complicated by the fact that each is the expectation of a nonsmooth function, i.e., averaging over a max or min operator. As we have derived influence functions for the smooth functions $\theta_1^l$ and $\theta_2^l$, simple asymptotic normality results (albeit for weaker bounds) can be obtained by dropping the min and max, averaging only over one or the other separately. These can be obtained via standard techniques, and are given in Appendix~\ref{sec:plugin}.

For the stronger nonsmooth bounds, we will need an additional assumption to guarantee asymptotic normality and provide valid confidence intervals. Our results build on a framework introduced by \citet{levis2023covariate} in the context of estimating bounds in instrumental variable models. They show that bounds with a similar expectation-of-max structure can be estimated under a margin condition which requires that the terms appearing in the max (or min) are separated with sufficiently high probability. We generalize their framework beyond the instrumental variable setting to provide conditions for estimation of any expectation-of-max structure where the terms inside the max admit first-order bias-corrected estimators. Suppose we want to estimate a bound of the form $E[\max_{j = 1...J} \theta_j(X)]$ (in general, we can allow more than two components, although this is all we use so far). We adopt a margin assumption similar to \citet{levis2023covariate}, itself inspired by similar assumptions used in a variety of other statistical settings \citep{audibert2007fast,luedtke2016statistical,kennedy2020sharp}. 
\begin{assumption} \label{assumption:margin}
  For some fixed $\alpha > 0$, $P\left[\min_{j \neq d(X)} \theta_{d(X)}(X) - \theta_j(X) \leq t\right] \lesssim t^{\alpha}$ 
\end{assumption}

This condition will be satisfied when the distribution of $\theta_{d(X)}(X) - \theta_j(X)$ has bounded density near 0. With this condition, we obtain the following result for the stronger estimator of the upper bound.
\begin{restatable}[Asymptotic Normality of Estimators]{theorem}{Normality}
\label{prop:normality-max}
    Let $\hat{\theta}$ denote the plugin estimate of any of the individual components of each bound. Under the conditions that Assumption \ref{assumption:margin} is satisfied, $\mu$ and $g$ are lower bounded, and each $\hat{\theta}$ is consistent (i.e., $||\hat{\theta} - \theta|| = o_P(1)$), the error of the estimator satisfies. 
    \begin{align*}
        \hat{\psi}^u - \psi^u  = O_P&\Big(||\hat{\theta}^u_j - \theta^u_j||_\infty^{1 + \alpha} \\ &+ \max_{j = 1, \ldots, J}E_P[\hat{\theta}^u_j + \hat{\lambda}^u_j - \theta^u_j]\Big) + O_P(n^{-\frac{1}{2}}) 
    \end{align*}
    Provided that $\hat{\pi}$, $\hat{\mu}$, and $\hat{g}$ converge at a $o_P(n^{-\frac{1}{4}})$ rate, and the plugin estimators satisfy $||\hat{\theta}^u_j - \theta^u_j||_\infty^{1 + \alpha} = o_P(n^{-\frac{1}{2}})$, then $\hat{\psi}_u$ is asymptotically normal with 
    \begin{align*}
        \sqrt{n}(\hat{\psi}^u - \psi^u ) \to N(0, Var(\varphi(P, d))).
    \end{align*}
\end{restatable}

The full proof is contained in Appendix \ref{sec:margin}. The error in our bias-corrected estimator of $\theta_1^u$ reduces to a sum of (i) a term on the order of $P(Y=1, D=0) - \hat{P}(Y=1, D=0)$, which converges at a parametric rate and (ii) a product of differences in $(\hat{\pi} - \pi)$ and $(\hat{g} - \hat{g})$. As such, we only require each of these estimators to converge at the slower rate of $o_P(n^{-\frac{1}{4}})$ to achieve our fast rate. The plugin estimators $\hat{\theta}$ may also converge at slower rates depending on $\alpha$ in the margin condition, e.g., if $\alpha \geq 1$ then $o_P(n^{-\frac{1}{4}})$ suffices for them as well. 

As asymptotic normality holds, we can construct standard confidence intervals for the relevant estimator $\hat{\psi}$. First, we can obtain a consistent estimator of the variance of $\hat{\psi}$ as
\begin{align*}
    \hat{\sigma}^2 & = \frac{1}{n} \sum_{i=1}^n \left( \hat{\psi}(X_i) - \frac{1}{n}\sum_{j=1}^n \hat{\psi}(X_j) \right)^2 ,
\end{align*}
which results in a confidence interval
\begin{align*}
    \widehat{CI}_{\psi^{l}, \psi^{u}} = \left[ \hat{\psi^{l}} - z_{1 - \alpha / 2} \frac{\hat{\sigma}^l}{\sqrt{n}}, \hat{\psi^{u}} + z_{1 - \alpha / 2} \frac{\hat{\sigma}^u}{\sqrt{n}}\right],
\end{align*}

where $\hat{\sigma}^l$ and $\hat{\sigma}^u$ are the estimated standard deviations of $\psi^l$ and $\psi^u$.
Therefore, we can use our estimators, and their corresponding confidence intervals to provide confidence bounds on the treatment allocation rates of interest. When these intervals are disjoint and non-overlapping across groups, our results suggest the presence of inequity (i.e., when our assumptions hold).

\subsection{Sensitivity Analysis under Bounded Confounding}\label{sec:sensitivity}

If it is plausible to impose an assumption that confounding in treatment assignment is bounded, we can in turn obtain tighter bounds on our estimand of interest. We introduce a sensitivity parameter $\gamma$ that captures the extent of the impact of the potential outcome $Y(0)$ on treatment assignment $T$, similar in spirit to the sensitivity model used by \citet{tan2006distributional}, adapted to our problem setting. This model allows us to, under the assumption that confounding is limited, assess whether there are verifiable discrepancies in allocation rates across subgroups. With this framework, we can vary $\gamma$ over a range of values to determine to how much confounding our finding is robust. 

\begin{definition}\label{assumption:gamma}
    We define a \textit{sensitivity parameter} $\gamma$ as
\begin{equation*}
    \frac{1}{\gamma} \leq \frac{P(T = 1 | Y(0)=0, D=1, X)}{P(T=1| Y(0) = 1, D=1, X)} \leq \gamma.
\end{equation*}
\end{definition}

We note that $\gamma = \infty$ is equivalent to arbitrary unmeasured confounding. In this scenario, we can recover the result in Theorem \ref{theorem:no_assumption}. 
Assuming a finite value of $\gamma$, we obtain the following stronger upper and lower bounds:
\begin{restatable}[Bounds with $\gamma$]{theorem}{IdentificationGamma}\label{theorem:identification_gamma}
    Using Definition \ref{assumption:gamma}, we achieve the following set of bounds
    \footnotesize
    $$\psi^{l, \gamma} \leq P(T = 1|Y(0) = 1, D = 1) \leq \psi^{u, \gamma}$$ 
    \normalsize
    where
\footnotesize
\begin{align*}
    & \psi^{l, \gamma} \coloneqq E[\max\{ \theta_1^{l},  \theta_2^{l}, \theta_3^{l, \gamma}\}]\\
    & \psi^{u, \gamma} \coloneqq E[\min\{ \theta_1^{u}, \theta_2^{u}, \theta_3^{u, \gamma}\}]\\
    & \theta_3^{l, \gamma} \coloneqq \frac{\scriptstyle P(T = 1 | D=1, X)}{\scriptstyle P(Y(0) = 1 |D=1, X) + \gamma(1 -  P(Y(0)= 1 | D=1, X))} \\
    & \theta_3^{u, \gamma} \coloneqq \frac{\scriptstyle P(T = 1|D=1, X)}{\scriptstyle P(Y(0) = 1|D=1, X) + \frac{1}{\gamma} (1 - P(Y(0) = 1|D=1, X))} \\
\end{align*}
\normalsize
\normalsize
\end{restatable}

We defer the proof to Appendix \ref{appendix:gamma_identification}. This is again similar in style to Theorem \ref{theorem:no_assumption}, where we incorporate covariate information to achieve tighter bounds $\psi^{l, \gamma}$ and $\psi^{u, \gamma}$. We note that these bounds converge to our earlier ones as $\gamma \to \infty$, and at $\gamma = 1$ (i.e., no confounding with respect to $Y(0)$), which implies point identification at $P(T = 1|Y(0) = 1, D=1, X) = P(T = 1|D=1, X)$. Notably, the upper and lower bound takes a min/max (respectively) over the two terms appearing in the bound in~\cref{theorem:no_assumption} (valid without any assumption on confounding) and a new quantity (that is dependent on $\gamma$).

We present a high-level overview of methods and results for the construction of estimators for the bounds in~\cref{theorem:identification_gamma}, with details in Appendix \ref{appx:gamma_estimation}.\footnote{We also offer an alternative expression of the bounds in Appendix \ref{sec:alt_gamma}, analagous to Appendix \ref{sec:alt_no_assumption} for~\cref{theorem:no_assumption}.} The picture is very similar to before: we can construct first order bias-corrected estimators for each term appearing inside the max and min, by adding the expectation (over $\hat{P})$ of the influence functions we have derived. Again, one option, requiring fewer assumptions, is to take expectations over each term individually and then choose the stronger of the bounds after averaging (applying a multiple-comparisons correction such as union bound to the level of the CIs). Under Assumption \ref{assumption:margin}, we further obtain that the expectation-of-max estimator is also asymptotically normal with sufficiently fast convergence rates for the estimators of the nuisance functions, which we defer to Appendix \ref{sec:margin}.

\subsection{Benchmarking Sensitivity Analysis}

The sensitivity parameter $\gamma$ is an assumption, not something we can estimate from data. To assess the plausibility of different $\gamma$ values, we can compute an analogous $\gamma'$ for an \textit{observed} random variable (e.g., diabetes) which is held out of the covariate set $X$. This quantity \textit{can} be estimated from data to perform benchmarking.

We determine this value of $\gamma'$ by training a discriminative model to compute the following inequality, where $X'$ denotes all covariates $X$ except $Z$, 
$$\frac{1}{\gamma'} \leq \frac{P(T = 1 | Z=0, D=1, X')}{P(T=1 | Z=1, D=1, X')} \leq \gamma',$$ where $Z$ is the random variable that represents if the patient has the covariate of interest.

\section{RESULTS}

We apply our analysis framework to understand treatment allocation inequity in the real-world setting of Paxlovid allocation for high-risk COVID-19 outpatients.
Further, we provide comparisons on semi-synthetic and synthetic settings, to demonstrate that our framework more accurately and more efficiently captures ground truth treatment rates, improving upon alternative methods.\footnote{We release our code at \\  \href{https://github.com/lasilab/inequity-bounds.git}{\texttt{https://github.com/lasilab/inequity-bounds.git}}.}
In our experiments, we use a logistic regression model for our estimators. For each estimator, we perform cross-fitting/sample-splitting over 5 disjoint folds. In all of our reported bounds, we use a 95\% confidence interval; in the case where our bounds use two quantities, we use a 97.5\% confidence interval for each, so that the resulting confidence interval is 95\% (via an application of the union bound).

\subsection{Dataset and Cohort Definition}

We use the NCATS NC3 cohort \citep{10.1093/jamia/ocaa196}, consisting of national line-level data of $18,438,581$ total patients, including $7,149,421$ confirmed COVID-19 positive patients, pooled from $76$ different data sharing centers across the United States.  
We focus our analysis on outpatients
with a positive SARS-CoV-2 test result, satisfying eligibility requirements (see Appendix \ref{cohort}).

\subsection{Real-world Study Results}
\begin{figure}[t]
\begin{center}
\centerline{\includegraphics[height=2.05in]{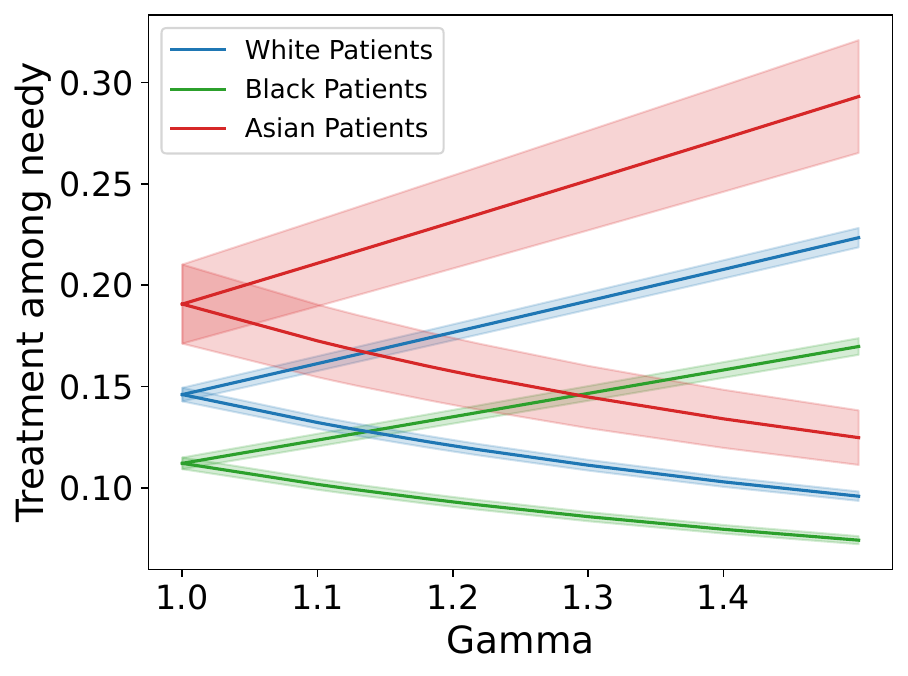}}
\vspace{-3mm}
\caption{Upper and lower bounds for treatment rate among the needy $P(T = 1 |Y(0) = 1, D=1, G=g)$ are computed for each racial group $g$, with varying values of $\gamma \in [1,1.5]$. The shaded area represents a 95\% confidence interval.}
\label{icml-historical}
\end{center}
\vskip -0.3in
\end{figure}

Under arbitrary unobserved confounding, we obtain vacuous bounds as expected, as this setting violates the conditions under which~\cref{theorem:no_assumption} provides meaningful results. Specifically, the treatment probability exceeds the probability of the adverse untreated outcome (i.e., mortality), leading to a vacuous upper bound, while the very low probability of the adverse untreated outcome results in a vacuous lower bound (see Table \ref{tab:rates}).
Under bounded unobserved confounding as in~\cref{assumption:gamma}, with parameter $\gamma$, we are able to identify non-overlapping bounds for our quantity of interest $P(T = 1 | Y(0) = 1, D=1)$ for particular subgroups. We identify non-overlapping bounds between Black and White patients ($\gamma \leq 1.11$) as well as Black and Asian patients ($\gamma \leq 1.21$). Hence, treatment rates for Black patients that would die without treatment are strictly lower than treatment rates for White and Asian patients, up to $\gamma = 1.11$, $\gamma = 1.21$ respectively,
\textbf{highlighting substantial inequity}. For a better interpretation of $\gamma$, we perform the following benchmarking analysis. 

\begin{figure}[t]
\begin{center}
\centerline{\includegraphics[height=2.05in]{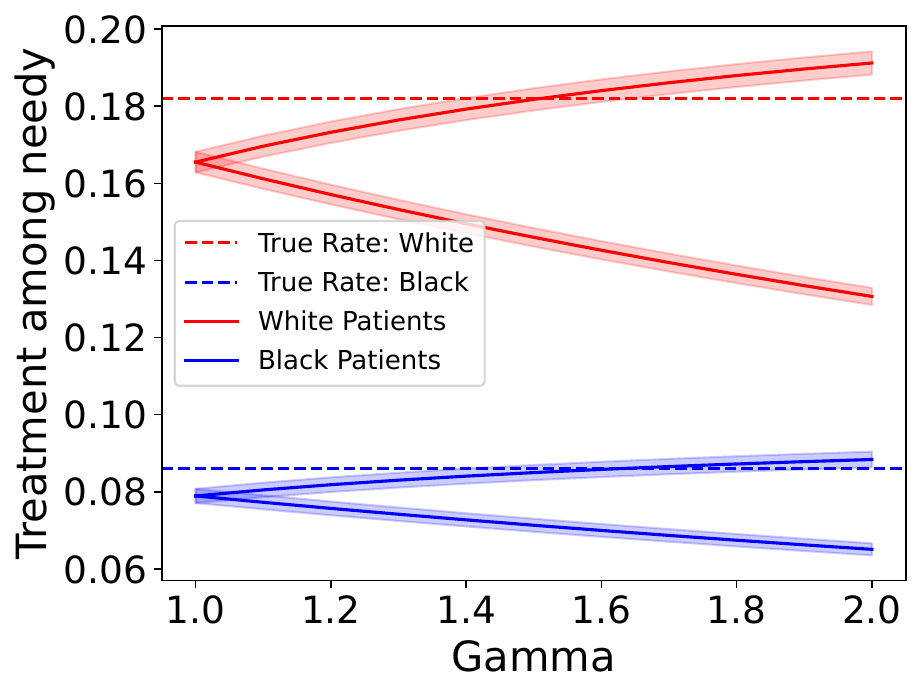}}
\caption{(Semi-Synthetic Data) Upper and lower bounds for treatment rate among the needy, with 95\% confidence intervals, for each racial group $g$, with varying values of $\gamma \in [1,2]$ (true value of $\gamma = 1.5$).}
\label{fig:semi_syn}
\end{center}
\vspace{-6mm}
\end{figure}

\subsection{Benchmarking Sensitivity Analysis}\label{sec5.2.3}

In our benchmarking, we select diabetes as our covariate of interest, based on its well-documented association with high risk of severe COVID-19 \citep{cdc_web}. 
We again proceed by training a classifier to predict treatment, letting $Z$ be diabetes and $X'$ be all other covariates. Then, we compute the following ratio on post-availability test data, using counterfactual features of having diabetes ($Z = 1$) or not having diabetes ($Z = 0$) for each patient:
\begin{align*}
    \frac{1}{\gamma'} \leq \frac{P(T = 1 |Z = 0, D=1, X')}{P(T = 1 | Z = 1, D=1, X')} \leq \gamma'.
\end{align*}

We observe that the smallest value of $\gamma'$ that satisfies the above equation for all post-availability test data is $1.09$. Therefore, our result in identifying disparities in allocation (e.g., non-overlapping bounds for (1) Black and White $\gamma \approx 1.11$ and (2) Black and Asian $\gamma \approx 1.21$ ) is \textbf{robust to unobserved confounding variables} that exhibit an influence on COVID-19 treatment allocation up to the impact of a patient's diabetes, which is evidenced to be associated with high risk of severe COVID-19 \citep{cdc_web}. In other words, unless there exists an unobserved confounder more influential than diabetes, there exists a disparity in treatment between both (1) Black and White patients and (2) Black and Asian patients.

\subsection{Semi-synthetic and Synthetic Settings}

\begin{figure}[t]
\begin{center}

\centerline{\includegraphics[height=2.05in]{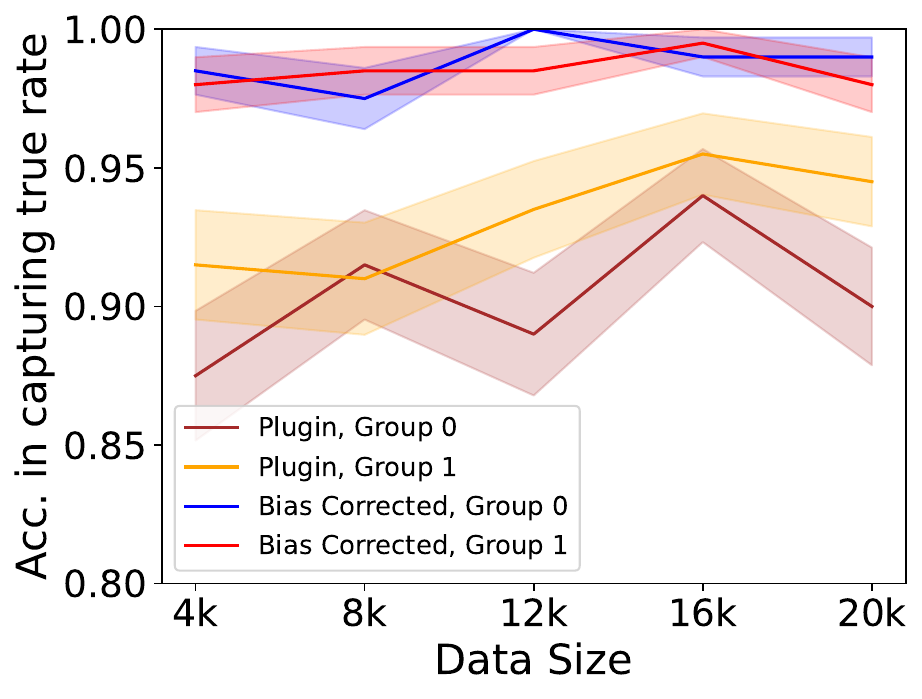}} 
\caption{(Synthetic Data) Accuracy of our bias-corrected bounds compared to their plugin counterparts in capturing the true treatment rates. We use the derived 95\% confidence interval from our bias-corrected estimates for both methods.}
\label{fig:syn_acc}
\end{center}
\vspace{-5mm}
\end{figure}
We generate both semi-synthetic and synthetic tasks from the Folktables dataset comprised of US Census data \citep{ding2021retiring}. In these tasks, we know the ground truth rates of treatment among the needy, so we can study whether our bounds are indeed valid and empirically compare them to alternative approaches. In the semi-synthetic setting, we use two racial groups of White and Black patients, and we simulate both $Y$ and $T$. We define $Y = T * Y(1) + (1 - T) * Y(0)$ and sample $Y(0) \sim \text{Bernoulli}(\sigma(|x|_1 + 2))$ and $Y(1) \sim \text{Bernoulli}(\sigma(|x|_1 + 2) / 2)$, where $\sigma$ represents the sigmoid function.
To produce a known value of $\gamma,$ we use $Y(0)$ to confound the generation of $T$. We sample $T \sim \text{Bernoulli}(p)$. For White patients, $p = \sigma(|x|_1 - 1)$ and for Black patients, $p = \sigma(|x|_1 - 2)$. If $Y(0) = 1$, we divide $p$ by 1.5, making $\gamma = 1.5$.

We generate our fully synthetic task in a similar fashion, where our covariates are sampled from a 2D Gaussian of $\mathcal{N}(0, 0.2) \times \mathcal{N}(0, 0.1)$. In this task, we control $\gamma = 1.5$, similar to the semi-synthetic task.

In our semi-synthetic experiments, we observe that our estimates of our bounds \textbf{successfully capture the true treatment rates} among the needy, given the true amount of unobserved confounding (i.e., $\gamma = 1.5$) (Figure \ref{fig:semi_syn}). 
To the best of our knowledge, no other approach provides valid bounds in this setting. To illustrate the benefit of our approach over alternative approaches (e.g., plug-in estimates),
we run synthetic experiments (Figure \ref{fig:syn_acc}) over 100 different trials given limited data (4000 $\sim$ 20000 samples). We capture the rates at which our bounds and the plug-ins capture the true rates given the actual value of $\gamma = 1.5$. 
We observe that our bounds capture the true rates at a \textbf{significantly higher rate given limited data} compared to alternative plug-in-based approaches.

\section{DISCUSSION}

In this work, we introduce a principled approach that allows practitioners to audit need-based inequity in existing decision-making systems. We consider a causal notion of equity whereby allocation rates should be equalized across groups when conditioning on the population who would suffer an adverse outcome without resource allocation.
We demonstrate that one can robustly quantify need-based inequity when relaxing the assumption on no unmeasured confounding, and surprisingly, can still obtain informative bounds when entirely eliminating this assumption.
We show that our bounds can be estimated with flexible machine learning models (e.g., nonparametric models, random forests, etc.), while providing valid confidence intervals.
Furthermore, we apply our method to analyze a real-world case study of Paxlovid allocation to high-risk COVID-19 patients,
and we find that observed inequity between racial groups cannot be explained by unobserved confounders at the same influence of important observable covariates.
More broadly, we remark that our setting and design are quite general and can easily be applied to different settings such as the creation of new services, government programs, and so on. Equivalently, this can be applied to policies, benefits, or treatments that roll out in one location and not the other.

\subsubsection*{Acknowledgements}
We would like to thank Angel Desai for the helpful discussion during the early phase of this project. We would also like to thank the anonymous reviewers for their valuable feedback. YB was supported in part by the AI2050 program at Schmidt Sciences (Grant G-22-64474) and also gratefully acknowledges the NSF (IIS2211955), UPMC, Highmark Health, Abridge, Ford Research, Mozilla, the PwC Center, Amazon AI, JP Morgan Chase, the Block Center, the Center for Machine Learning and Health, and the CMU Software Engineering Institute (SEI) via Department of Defense contract FA8702-15-D-0002, for their generous support of ACMI Lab's research. 
DS was supported by the Bosch Center for Artificial Intelligence, the ARCS Foundation, and the National Science Foundation Graduate Research Fellowship under Grant No. DGE2140739.

The analyses described in this publication were conducted with data or tools accessed through the NCATS N3C Data Enclave https://covid.cd2h.org and N3C Attribution \& Publication Policy v 1.2-2020-08-25b supported by NCATS U24 TR002306, Axle Informatics Subcontract: NCATS-P00438-B. This research was possible because of the patients whose information is included within the data and the organizations (https://ncats.nih.gov/n3c/resources/data-contribution/data-transfer-agreement-signatories) and scientists who have contributed to the on-going development of this community resource [https://doi.org/10.1093/jamia/ocaa196]. 

Disclaimer: The N3C Publication committee confirmed that this manuscript is in accordance with N3C data use and attribution policies; however, this content is solely the responsibility of the authors and does not necessarily represent the official views of the National Institutes of Health or the N3C program.

The N3C data transfer to NCATS is performed under a Johns Hopkins University Reliance Protocol IRB00249128 or individual site agreements with NIH. The N3C Data Enclave is managed under the authority of the NIH; information can be found at https://ncats.nih.gov/n3c/resources.

\bibliography{ref}
\bibliographystyle{plainnat}

\onecolumn

\appendix
\section{Additional Proofs \& Statements}\label{appx:proof}

In this section, we provide additional remarks as well as the omitted proofs for the Propositions, Lemmas, and Theorems in the main paper, except for \cref{prop:normality-max}. The proof of \cref{prop:normality-max} is separately located in \cref{sec:margin}.

\subsection{Representing Assumptions in a Causal Graph}\label{app:causal_swig}

In~\cref{fig:example_causal_graph}, we gave an illustrative causal graph and claimed that this causal structure is sufficient, but not necessary, for our assumptions to hold.  A more precise characterization is given here, using the framework of single-world intervention graphs (SWIGs), developed by \citet{Richardson2013-18}.  Single-world intervention graphs are a useful tool for relating assumptions that use potential outcome notation to those that use the framework of causal directed acyclic graphs.

\begin{figure}[t]
\begin{center}
\begin{tikzpicture}[
  obs/.style={circle, draw=gray!90, fill=gray!30, very thick, minimum size=8mm}, 
  uobs/.style={circle, draw=gray!90, fill=gray!10, dotted, minimum size=8mm}, 
  bend angle=30]
  \tikzset{line width=1.5pt, outer sep=0pt,
  ell/.style={draw,fill=white, inner sep=2pt,
  line width=1.5pt},
  swig vsplit={gap=5pt,
  inner line width right=0.5pt}};
  \node[obs] (X) {$X$} ;
  \node[obs] (Y) [right=of X] {$Y_0$} ;
  \node[name=T,above=of Y,shape=swig vsplit,fill=gray!30,draw=gray!90,minimum size=8mm]{
    \nodepart{left}{$T$}
    \nodepart{right}{$T = 0$}
  };
  \node[uobs] (U) [above=of X] {$C$} ;
  \node[obs] (D) [left=of U]  {$D$};
  \draw[-latex,thick,>=stealth]
    (X) edge (Y)
    (X) edge (T)
    (D) edge[bend left] (T)
    (T) edge[out=-30,in=30] (Y)
    (U) edge (Y)
    (U) edge (X)
    (U) edge (T);
\end{tikzpicture}
\end{center}
\caption{Example of a single-world intervention graph (SWIG) \citep{Richardson2013-18}, mirroring~\cref{fig:example_causal_graph}, that satisfies Assumptions~\ref{asmp:covariate_stability} and~\ref{asmp:stable_basline_risk}, but where unobserved confounding is present. Note that we use $Y_0$ in lieu of $Y(0)$ for consistency with typical SWIG notation, but these notations equivalently represent the potential outcome under $T = 0$. We use node-splitting notation, where all outgoing edges from $T$ propagate the chosen value $T = 0$, all incoming edges go to $T$, and there is no connection between $T$ and $T = 0$.  This graph illustrates causal relationships in the `single world' where we intervene upon $T$ and set it to the chosen value.}%
\label{fig:example_causal_graph_swig}
\end{figure}
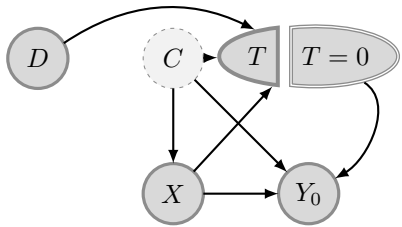

\Cref{fig:example_causal_graph_swig} applies the intervention $T = 0$ to the causal graph given in~\cref{fig:example_causal_graph}, via a `node splitting' operation (see \citet{Richardson2013-18} for more details), where the node $T$ is split, all incoming edges go to $T$, and all outgoing edges propagate the value $T = 0$, yielding $Y_0$ instead of $Y$ in this example. This graph allows us to characterize the causal relationships between the potential outcome $Y_0$ and other variables, in the `single world' where we intervene upon $T$ and set it to the desired value $T = 0$.  Note that in the resulting graph, the nodes $T$ and $T = 0$ are not connected. From d-separation in the graph given in~\cref{fig:example_causal_graph_swig}, we can observe that both~\cref{asmp:covariate_stability} and~\cref{asmp:stable_basline_risk} hold, namely that
\begin{align*}
   X &\perp D \mid Y_0 & \text{and} && Y_0 &\perp D \mid X,
\end{align*}
where the former implies~\cref{asmp:covariate_stability} and the latter is equivalent to~\cref{asmp:stable_basline_risk}. However, our assumptions are only a subset of the implications of this causal structure.  For instance, this causal structure would imply similar relationships for $Y_1$, which does not appear in our assumptions.  Hence our claim that this causal structure is sufficient, but not necessary, for our assumptions to hold.

\subsection{Proof of~\cref{prop:identification}}
\label{sec:identification}
\Identification*
\begin{proof}
\begin{align*}
    & P(T = 1 \mid Y(0) = 1, D = 1, G = g) \\
    &\quad =  \int_{x} P(T = 1 \mid X = x, Y(0) = 1, D = 1, G = g) \\
    &\qquad \qquad \cdot P(X = x \mid Y(0) = 1, D = 1, G = g) dx \\
    &\quad =  \int_{x} P(T = 1 \mid X = x, D = 1) \\
    &\qquad \qquad \cdot P(X = x \mid Y = 1, D = 0, G = g) dx 
\end{align*}
where the first equality follows from standard rules of probability, and the second equality invokes our three assumptions given above.
\end{proof}

\subsection{Proof of Theorem~\ref{theorem:no_assumption}}
\label{sec:no_assumption}
\Bounds*

\begin{proof}
First, we note that by Fréchet inequalities, for events $T=1, Y(0)=1$, we have that
\begin{align*}
    P(T=1,Y(0)=1|D=1, X) \leq \min\{P(T=1|D=1, X),P(Y(0)=1|D=1, X)\} \\
    P(T=1,Y(0)=1|D=1, X) \geq \max\{0, P(T=1|D=1, X)+P(Y(0)=1|D=1, X)-1\}
\end{align*}

To derive bounds on treatment rate among the needy, we divide by $P(Y(0)=1|D=1, X)$, giving us
\begin{align*}
    P(T=1 | Y(0)=1, D=1, X) \leq \min\left\{\frac{P(T=1|D=1, X)}{P(Y(0)=1|D=1, X)}, 1 \right\} \\
    P(T=1 | Y(0)=1, D=1, X) \geq \max\left\{0, \frac{P(T=1|D=1, X)+P(Y(0)=1|D=1, X)-1}{P(Y(0)=1|D=1, X)}\right\}
\end{align*}

Next, we remark that our quantity of interest is given by
\begin{align*}
    P(T = 1 | Y(0) = 1, D=1) = E_X[ P(T = 1 | Y(0) = 1, D=1, X) | Y(0) = 1, D=0],
\end{align*}
where we can switch from $D=1$ to $D=0$ in our conditional expectation due to \cref{asmp:covariate_stability}. 
Therefore, we get that our target is now given by
\begin{align*}
    & E\left[ \max\left\{0, \frac{P(T = 1 | D=1, X) - P(Y(0) = 0 | D=1, X)}{P(Y(0) = 1 | D=1, X)} \right\} \mid Y=1, D=0 \right] \\
    & \qquad \leq P(T = 1 | Y(0) = 1, D=1) \leq E \left[ \min\left\{1, \frac{P(T = 1 | D=1, X)}{P(Y(0) = 1 | D=1, X)}\right\} \mid Y=1, D=0 \right]
\end{align*}

Finally, we can convert this to be computed over the \textit{unconditional} expectation as follows. The upper bound is given by a min over two terms. The term involving 1 simplifies to
\begin{align*}
    E[1 | Y=1, D=0] & = \int_x 1 \cdot P(X = x | Y=1, D=0)  \\
    & = \int_x P(Y=1, D=0 | X=x) \frac{P(X=x)}{P(Y=1, D=0)} \\
    & = \frac{1}{P(Y=1, D=0)} E\left[ P(Y=1, D=0 | X) \right] \\
    & = \frac{1}{P(Y=1, D=0)} E\left[ P(Y=1 | D=0, X) \cdot P(D=0 | X) \right]
\end{align*}

The other term is given by
\begin{align*}
    E\left[ \frac{P(T = 1 | D=1, X)}{P(Y(0) = 1 | D=1, X)} \mid Y=1, D=0) \right]
\end{align*}
Note that with \cref{asmp:stable_basline_risk}, we can replace the denominator with $P(Y(0) = 1 | D=0, X)$ as $Y(0)$ and $D$ are independent conditioning on covariates $X$. Then, we have that
\begin{align*}
        E\left[ \frac{P(T = 1 | D=1, X)}{P(Y(0) = 1 | D=1, X)} \mid Y=1, D=0 \right] & = E\left[  \frac{P(T = 1 | D=1, X)}{P(Y(0) = 1 | D=0, X)} \mid Y=1, D=0 \right] \\
        & = \int_x \frac{P(T = 1 | D=1, X)}{P(Y(0) = 1 | D=0, X)} P(Y=1, D=0 | X=x) \frac{P(X=x)}{P(Y=1, D=0)} \\
        & = \frac{1}{P(Y=1, D=0)} E\left[  \frac{P(T = 1 | D=1, X)}{P(Y(0) = 1 | D=0, X)} P(Y=1, D=0 | X) \right] \\
        & = \frac{1}{P(Y=1, D=0)} E\left[P(D=0 | X) P(T = 1 | D=1, X) \right]
\end{align*}

Next, we can consider the lower bound. The lower bound is given by a max of two terms. The zero term is trivially 0. The other term is given by

\begin{align*}
    E\left[ \frac{P(T = 1 | D=1, X) - P(Y(0) = 0 | D=1, X)}{P(Y(0)=1 | D=1, X)} \mid Y=1, D=0 \right]
\end{align*}

We can again switch $D=1$ to $D=0$ in both the $Y(0)$ term in the numerator and in the term in the denominator using \cref{asmp:stable_basline_risk}. Then, we get that
\begin{align*}
    & E\left[ \frac{P(T = 1 | D=1, X) - P(Y(0) = 0 | D=0, X)}{P(Y=1 | D=0, X)} \mid Y=1, D=0 \right] \\
    & = \int_x \frac{P(T = 1 | D=1, X) - P(Y(0) = 0 | D=0, X)}{P(Y=1 | D=0, X)} \cdot  P(Y=1, D=0 | X=x) \frac{P(X=x)}{P(Y=1, D=0)} \\
    & = \frac{1}{P(Y=1, D=0)} E\Big[ P(D=0 | X) P(Y=1 | D=0, X)  \cdot \frac{P(T = 1 | D=1, X) - P(Y(0) = 0 | D=0, X)}{P(Y=1 | D=0, X)} \Big] \\
    & =  \frac{1}{P(Y=1, D=0)} E\Big[ P(D=0 | X)  \cdot \left(P(T = 1 | D=1, X) + P(Y(0) = 1 | D=0, X) - 1 \right) \Big]
\end{align*}
\end{proof}

as desired.

\subsection{Alternative Expression of Theorem~\ref{theorem:no_assumption}}
\label{sec:alt_no_assumption}

We provide an alternative formulation of the bounds presented in Theorem~\ref{theorem:no_assumption} to offer additional intuition.

By leveraging the Fréchet inequalities, we had the following bounds on the treatment rate among the needy:
\begin{align*}
    P(T=1 \mid Y(0)=1, D=1, X) &\leq \min\left\{\frac{P(T=1 \mid D=1, X)}{P(Y(0)=1 \mid D=1, X)}, 1 \right\}, \\
    P(T=1 \mid Y(0)=1, D=1, X) &\geq \max\left\{0, \frac{P(T=1 \mid D=1, X) + P(Y(0)=1 \mid D=1, X) - 1}{P(Y(0)=1 \mid D=1, X)}\right\}.
\end{align*}

To align the distribution of \(X\) with that of patients experiencing adverse events during the pre-availability period, we re-weight this quantity using 
\[
w(X) := \frac{P(Y=1, D=0 \mid X)}{P(Y=1, D=0)}
\]

Our bounds from Theorem~\ref{theorem:no_assumption} can be simply expressed as the bounds from Fréchet inequalities, each weighted by $w(X)$, which results in this alternative expression:
\begin{gather*}
    \psi^l \coloneqq E[w(X) \cdot \max\{0, \theta_1^{lw}(X)\}], \quad
    \psi^u \coloneqq E[w(X) \cdot \min\{\theta_1^{uw}(X), 1\}], \\
    \text{where} \quad \theta_1^{lw}(X) = \frac{P(T=1 \mid D=1, X) + P(Y=1 \mid D=0, X) - 1}{P(Y=1 \mid D=0, X)}, \quad \theta_1^{uw}(X) = \frac{P(T=1 \mid D=1, X)}{P(Y(0)=1 \mid D=1, X)}.
\end{gather*}

\begin{proof} By simple expansion, we can be show that this reformulation is equivalent to the original bounds in Theorem~\ref{theorem:no_assumption}.
    First, for the upper bound term $\theta_1^{uw}(X)$, we have by \cref{asmp:covariate_stability,asmp:stable_basline_risk}
    \begin{align*}
        \theta_1^{uw}(X) \cdot w(X) & =  \frac{P(T=1|D=1,X)}{P(Y(0)=1 | D=1,X)} \cdot \frac{P(Y=1,D=0|X)}{P(Y=1,D=0)} \\
        & = \frac{P(T=1|D=1,X)P(D=0|X)}{P(Y=1,D=0)}
    \end{align*}

Second, for the lower bound term $\theta_1^{lw}(X)$, we have by \cref{asmp:covariate_stability,asmp:stable_basline_risk}
\begin{align*}
    \theta_1^{lw}(X) \cdot w(X) & = \frac{P(T=1|D=1,X) + P(Y=1|D=0,X) - 1}{P(Y=1|D=0,X)} \cdot \frac{P(Y=1,D=0|X)}{P(Y=1,D=0)} \\
    & = (P(T=1|D=1,X) + P(Y=1|D=0,X) - 1) \cdot \frac{P(D=0|X)}{P(Y=1,D=0)}
\end{align*}

Therefore, our bounds can intuitively be viewed as a weighted average of the conditional bounds over $X$. 
\end{proof}

\subsection{Alternative Expression of Theorem~\ref{theorem:identification_gamma}}
\label{sec:alt_gamma}

Similarly, we provide an alternative formulation of the bounds presented in Theorem~\ref{theorem:identification_gamma}, explicitly illustrating how these bounds can be interpreted as a weighted average of the conditional bounds over $X$. 

Our bounds presented in Theorem~\ref{theorem:identification_gamma} can be cast in a similar way as the above. For simplicity, we use the simplified notation $\mu(X) := P(Y=1|D=0,X)$, $\pi(X) := P(T=1|D=1,X)$, and $g(X):= P(D=0|X)$. We have the same re-weighting term $w(X)$ (in simplified notation):
\begin{align*}
    w(X) := \frac{\mu(X)g(x)}{P(Y=1,D=0)}
\end{align*}

With this, our bounds from Theorem~\ref{theorem:identification_gamma} can be expressed as follows:
\begin{gather*}
    \psi^{l,\gamma} \coloneqq E[w(X) \cdot \max\{0, \theta_1^{lw}(X), \theta_2^{l\gamma w}(X)\}], \quad
    \psi^{u,\gamma} \coloneqq E[w(X) \cdot \min\{1, \theta_1^{uw}(X), \theta_2^{u\gamma w}(X)\}], \\
     \text{where} \quad \theta_2^{l\gamma w}(X) = \frac{\gamma^{-1}\cdot \pi(X)}{(\gamma^{-1}-1)\mu(X) + 1}, \quad \theta_2^{u \gamma w}(X) = \frac{\gamma \cdot \pi(X)}{(\gamma -1)\mu(X) + 1}.
\end{gather*}

\begin{proof} By simple expansion, we can see that the upper and lower bound terms $\theta_2^{u\gamma w}(X)$ and $\theta_2^{l\gamma w}(X)$ is equivalent to the corresponding terms $\theta_3^{u,\gamma}(X)$ and $\theta_3^{l,\gamma}(X)$ in the original formulation (i.e., \eqref{eq:theta_3_u_gamma} and ~\eqref{eq:theta_3_l_gamma}).
\end{proof}

\subsection{Analysis of Plugin Estimators}\label{sec:plugin}

We now present some analysis of a standard plugin estimator, which will be useful in proofs in the error analysis of our bias-corrected estimators. At a high level, this section demonstrates that a simple plugin estimator for the (ratio) estimand of $E\left[\frac{\pi(X)}{\mu(X)} \right]$ achieves a rate that is a combination of the rates of our estimators of $\pi$ and $\mu$, plus an additional term that is the variance of our plugin estimator.

First, we will prove a technical lemma that bounds the expected error of a ratio estimator that directly takes a ratio of plugins.

\begin{lemma}\label{lemma:ratio_lemma}
    Let $R = \frac{\pi}{\mu}$, and $\hat{R} = \frac{\hat{\pi}}{\hat{\mu}}$. Then, we have that 
    \begin{align*}
    | E_{x\sim P} [R] - E_{x\sim P}[\hat{R}] | \leq \frac{2}{\delta^2} \left( E_{x \sim P} [|\pi - \hat{\pi}|] + E_{x\sim P}[|\hat{\mu} - \mu|] \right),
\end{align*}
    for some $0 \leq \delta \leq \mu, \hat{\mu}$.
\end{lemma}

\begin{proof}

We first observe that
\begin{align*}
    | E_{x\sim P} [R] - E_{x\sim P}[\hat{R}] | & \leq E_{x \sim P} \left[ | R - \hat{R} |\right] \\
    & = E_{x\sim P} \left[ | \frac{\pi \hat{\mu} - \hat{\pi} \mu}{\mu \hat{\mu}} | \right] \\
    & \leq \frac{1}{\delta^2} E_{x\sim P}\left[ |\pi \hat{\mu} - \hat{\pi} \mu | \right]
\end{align*} 

Let $x$ be an arbitrary data point. We observe that 
\begin{align*}
    \min \{\pi \mu - \hat{\pi} \hat{\mu}, \hat{\pi} \hat{\mu} - \pi \mu  \} + \pi \hat{\mu} - \hat{\pi} \mu \leq \pi \hat{\mu} - \hat{\pi} \mu \leq + \max \{\pi \mu - \hat{\pi} \hat{\mu}, \hat{\pi} \hat{\mu} - \pi \mu  \} + \pi \hat{\mu} - \hat{\pi} \mu,
\end{align*}
since one of $\pi \mu - \hat{\pi} \hat{\mu}, \hat{\pi} \hat{\mu} - \pi \mu$ must be non-positive, and one must be non-negative.

We first consider the term of $\pi \mu - \hat{\pi} \hat{\mu}$. This satisfies that
\begin{align*}
    \pi\mu - \hat{\pi}\hat{\mu} + \pi \hat{\mu} - \hat{\pi} \mu & = \mu(\pi - \hat{\pi}) + \hat{\mu}(\pi - \hat{\pi}) \\
    & = (\mu + \hat{\mu}) (\pi - \hat{\pi})
\end{align*}

Then, noting that $\mu, \hat{\mu} \in [0, 1]$, we have that
\begin{align*}
        |\pi\mu - \hat{\pi}\hat{\mu} + \pi \hat{\mu} - \hat{\pi} \mu|\leq 2 | \pi - \hat{\pi}|
\end{align*}
Next, we can consider the other case of the term $\hat{\pi} \hat{\mu} - \pi \mu$. We have that
\begin{align*}
    \hat{\pi}\hat{\mu} - \pi \mu + \pi \hat{\mu} - \hat{\pi} \mu & = (\hat{\mu} - \mu) (\pi + \hat{\pi}),
\end{align*}
and with $\pi, \hat{\pi} \in [0, 1]$, we get that
\begin{align*}
    | \hat{\pi}\hat{\mu} - \pi \mu + \pi \hat{\mu} - \hat{\pi} \mu| \leq 2|\pi + \hat{\pi}|,
\end{align*}

Therefore, we observe that 
\begin{align*}
    |\pi \hat{\mu} - \hat{\pi}\mu | & \leq 2 \max\{|\hat{\mu} - \mu|, |\pi - \hat{\pi} | \} \\
    E_{x\sim P}\left[|\pi \hat{\mu} - \hat{\pi}\mu | \right] & \leq 2 E_{x\sim P}\left[\max\{|\hat{\mu} - \mu|, |\pi - \hat{\pi} | \}\right] \\
    & \leq 2 \left( E_{x \sim P} [|\pi - \hat{\pi}|] + E_{x\sim P}[|\hat{\mu} - \mu|] \right)
\end{align*}

Plugging this in gives us the result that
\begin{align*}
    | E_{x\sim P} [R] - E_{x\sim P}[\hat{R}] | \leq \frac{2}{\delta^2} \left( E_{x \sim P} [|\pi - \hat{\pi}|] + E_{x\sim P}[|\hat{\mu} - \mu|] \right)
\end{align*}
\end{proof}

Now, we can consider the estimator $E_{\hat{P}}[\frac{T}{\hat{\mu}(X)}]$. To verify consistency, note that as $\hat{P} \to P$ and $\hat{\mu} \to \mu$ we have
\begin{align*}
    E_{\hat{P}}\left[\frac{T}{\hat{\mu}(X)}\right] \to E_{P}\left[\frac{T}{\mu(X)}\right] 
\end{align*}
and using iterated expectation yields that
\begin{align*}
    E_{T, X \sim P}\left[\frac{T}{\mu(X)}\right] & = E_{X \sim P}\left[\frac{E[T|X]}{\mu(X)}\right] \\
    & = E_{X \sim P}\left[\frac{\pi(X)}{\mu(X)}\right].
\end{align*}
Next, we analyze the total expected error of this estimator. To start with, note that 
\begin{align*}
    E_{\hat{P}}\left[\frac{T}{\hat{\mu}(X)}\right] - E_{X \sim P}\left[\frac{\pi(X)}{\mu(X)}\right] = \left(E_{\hat{P}}\left[\frac{T}{\hat{\mu}(X)}\right] - E_{P}\left[\frac{T}{\hat{\mu}(X)}\right] \right) + \left(E_{P}\left[\frac{T}{\hat{\mu}(X)}\right] - E_{P}\left[\frac{\pi(X)}{\mu(X)}\right]\right).
\end{align*}
 Provided that we employ sample splitting, so that $\hat{\mu}$ is trained on an independent sample from the samples used to estimate the expectation $\hat{P}$, the first term is easily controlled in terms of the variance of $\frac{T}{\mu(X)}$. Specifically, suppose that $\hat{P}$ is estimated using $n$ samples. We have that
 \begin{align*}
E\left[\left\lvert E_{\hat{P}}\left[\frac{T}{\hat{\mu}(X)}\right] - E_{P}\left[\frac{T}{\hat{\mu}(X)}\right]\right\rvert \right] &= E\left[\left\lvert E_{\hat{P}}\left[\frac{T}{\hat{\mu}(X)}\right] - E\left[E_{\hat{P}}\left[\frac{T}{\hat{\mu}(X)}\right]\right\rvert \right]\right]\\ 
&\leq \sqrt{\text{Var}\left[E_{\hat{P}}\left[\frac{T}{\hat{\mu}(X)}\right]\right]} \\
& =\sqrt{\frac{\text{Var}\left[\frac{T}{\hat{\mu}(X)}\right]}{n}}
 \end{align*}
where the first equality follows because $E_{\hat{P}}\left[\frac{T}{\hat{\mu}(X)}\right]$ is an unbiased estimator for $E_{P}\left[\frac{T}{\hat{\mu}(X)}\right]$, the second line follows by Cauchy-Schwartz, and the third because the samples in $\hat{P}$ are independent. 

For the second term, note that since $ E_{T, X \sim P}\left[\frac{T}{\mu(X)}\right] = E_{X \sim P}\left[\frac{E[T|X]}{\mu(X)}\right]$, we can apply Lemma \ref{lemma:ratio_lemma} with $\hat{\pi} = \pi$ to obtain that 
\begin{align*}
    \left\lvert E_{P}\left[\frac{T}{\hat{\mu}(X)}\right] - E_{P}\left[\frac{\pi(X)}{\mu(X)}\right]\right\rvert \leq \frac{2}{\gamma^2} E_{P}[|\mu(X) - \hat{\mu}(X)|]. 
\end{align*}

Combining the bounds on the first and second terms using the triangle inequality yields
\begin{align*}
      E\left[\left\lvert E_{\hat{P}}\left[\frac{T}{\hat{\mu}(X)}\right] - E_{X \sim P}\left[\frac{\pi(X)}{\mu(X)}\right] \right\rvert\right] \leq \sqrt{\frac{\text{Var}\left[\frac{T}{\hat{\mu}(X)}\right]}{n}} + \frac{2}{\gamma^2} E_{P}[|\mu(X) - \hat{\mu}(X)|].
\end{align*}
Note that a high-probability bound could be obtained by using a Bernstein bound for the first term combined with any high-probability generalization guarantee for the ML model in the second term. 

An analogous argument for the alternate plugin estimator $E_{\hat{P}}\left[\frac{\hat{\pi}(X)}{\hat{\mu}(X)}\right]$ yields the bound on its expected error
\begin{align*}
      E\left[\left\lvert E_{\hat{P}}\left[\frac{\pi(X)}{\hat{\mu}(X)}\right] - E_{X \sim P}\left[\frac{\pi(X)}{\mu(X)}\right] \right\rvert\right] \leq \sqrt{\frac{\text{Var}\left[\frac{\hat{\pi}(X)}{\hat{\mu}(X)}\right]}{n}} + \frac{2}{\gamma^2} \left(E_{P}[|\mu(X) - \hat{\mu}(X)|] + E_{P}[|\pi(X) - \hat{\pi}(X)|]\right).
\end{align*}

Comparing these two bounds, we observe a form of bias-variance tradeoff. In the second bound, we accumulate additional potential error from the estimation of $\hat{\pi}$ instead of directly plugging in the samples $T$. However, we often expect that $\hat{\pi}$ will have lower variance than $T$ since estimated treatment probabilities will take less extreme values than binary treatment indicators, in which case the variance term will be smaller for the second estimator.

\subsection{Proof of Lemma~\ref{lemma:influence_upper}}\label{appx:influence_upper}

Next, we will derive the influence functions for our upper bounds under no additional assumptions. Recall that our estimands are given by
\begin{align*}
  &\theta_1^u(X) \coloneqq \frac{P(D=0 | X)P(T = 1| D=1, X)}{P(Y=1, D=0)}, \\
  & \theta_2^u(X) \coloneqq \frac{P(D=0 | X) P(Y =1 | D=0, X)}{P(Y = 1, D=0)}
\end{align*}

Our relevant conditional distributions (i.e., our nuisance functions) are given by
\begin{align*}
    \mu(X) &\coloneqq E[Y=1 |D=0, X], \\
    \pi(X) &\coloneqq E[T = 1 | D=1, X], \\
    g(X) &\coloneqq E[D=0 | X].
\end{align*}

We now proceed to derive the influence functions for our upper bound under no additional assumptions.

\InfluenceUpper*

\begin{proof}
    First, we will derive the influence function for $\theta^{u}_{1}$. 

    \begin{align*}
    I F\left(\theta^{u}_{1}\right)= & I F\left(\frac{1}{P(Y=1, D=0)}\right) E_P[g(X) \pi(X)]+\frac{1}{P(Y=1, D=0)} I F\left(E_P[g(X) \pi(X)]\right) \\
    =- & \frac{1[Y=1, D=0]-P(Y=1, D=0)}{P(Y=1, D=0)^2} E_P[g(X) \pi(X)] \\
    & +\frac{1}{P(Y=1, D=0)} \sum_x(1[X=x]-p(x))(g(x) \pi(x)) \\
    & +\frac{1}{P(Y=1, D=0)} \sum_x p(x)\left(\frac{1[X=x]}{P(X=x)}(1[D=0]-g(x))\right) \pi(x) \\
    & +\frac{1}{P(Y=1, D=0)} \sum_x p(x) g(x)\left(\frac{1[D=1, X=x]}{P(D=1, X)}(T-\pi(x))\right) \\
    =- & \frac{1[Y=1, D=0]}{P(Y=1, D=0)^2} E_P[g(X) \pi(X)]+\frac{E_P[g(X) \pi(X)]}{P(Y=1, D=0)} \\
    & +\frac{g(X) \pi(X)}{P(Y=1, D=0)}-\frac{E[g(X) \pi(X)]}{P(Y=1, D=0)} \\
    & +\frac{\pi(X)(1[D=0]-g(X))}{P(Y=1, D=0)} \\
    & +\frac{1[D=1](T-\pi(X))}{P(Y=1, D=0)} \frac{g(X)}{1-g(X)} \\
    & = \frac{1}{P(Y=1, D=0)}\left(-\frac{1[Y=1, D=0]}{P(Y=1, D=0)} E_P[g(X) \pi(X)]\right. \\
& \left.\quad+g(X) \pi(X)+1[D=1](T-\pi(X)) \frac{g(X)}{1-g(X)}+\pi(X)(1[D=0]-g(X))\right)
    \end{align*}

Next, we derive the influence function for $\theta_2^u$. 

\begin{align*}
    IF(\theta_2^u) &= IF \left( \frac{1}{P(Y=1|D=0)} E[\mu(X)|D=0] \right) \\ 
    &=  \frac{1}{P(Y = 1 \mid D = 0)} \left( \frac{-1 \cdot IF(E[Y \mid D = 0])}{P(Y = 1 \mid D = 0)} E[\mu(X) \mid D = 0] + IF(E[\mu(X) \mid D = 0]) \right) \\ 
    &=  \frac{1}{P(Y = 1 \mid D = 0)} \left( \frac{-1 \cdot \frac{1[D = 0]}{P(D = 0)}\left(Y - E[Y \mid D = 0]\right)}{P(Y = 1 \mid D = 0)}E[\mu(X) \mid D = 0]\right. \\ 
    &\left. \qquad \qquad + IF(\sum_{x,d} p(x, d) \frac{1[d = 0]}{p(d)} \mu(x)) \right) \\
    &= \frac{1}{P(Y = 1 \mid D = 0)} \left( \frac{-1 \cdot 1[D = 0] \left(Y - E[Y \mid D = 0]\right)}{P(Y = 1 \mid D = 0)P(D = 0)}E[\mu(X) \mid D = 0]\right. \\ 
   &\qquad + \sum_{x,d} IF(p(x, d)) \frac{1[d = 0]}{p(d)} \mu(x) \\
    &\qquad + \sum_{x,d} p(x, d) 1[d = 0] IF\left(\frac{1}{p(d)}\right) \mu(x) + \left. \sum_{x,d} p(x, d) 1[d = 0] \frac{1}{p(d)} IF(\mu(x)) \right)
\end{align*}

This further simplifies as

\begin{align*}
  IF(\theta_2^u) &= \frac{1}{P(Y = 1 \mid D = 0)} \left( \frac{-1 \cdot 1[D = 0] \left(Y - E[Y \mid D = 0]\right)}{P(Y = 1 \mid D = 0)P(D = 0)}E[\mu(X) \mid D = 0] \right. \\ 
           &\qquad + \sum_{x,d} (1[X = x, D = d] -  p(x, d)) \frac{1[d = 0]}{p(d)} \mu(x) \\
           &\qquad - \sum_{x,d} p(x, d) 1[d = 0] \frac{1[D = d] - p(d)}{{p(d)}^2} \mu(x) \\
           &\qquad + \left. \sum_{x,d} p(x, d) 1[d = 0] \frac{1}{p(d)} \frac{1[D = 0, X = x]}{P(D = 0 \mid X)P(X)}\left(Y - E[Y \mid D = 0, X]\right) \right)
\end{align*}

We'll consider each of the four terms, one at a time, and ignore the initial $P{(Y = 1 \mid D = 0)}^{-1}$ term for now.
\begin{align*}
  &\frac{-1 \cdot 1[D = 0] \left(Y - E[Y \mid D = 0]\right)}{P(Y = 1 \mid D = 0)P(D = 0)} E[\mu(X) \mid D = 0]  \\
  &=   -1 \cdot \frac{1[D = 0]}{P(D = 0)} \frac{Y - E[Y \mid D = 0]}{P(Y = 1 \mid D = 0)} E[\mu(X) \mid D = 0] 
\end{align*}

Now we will consider the second term 
\begin{align*}
  \sum_{x,d} (1[X = x, D = d] -  p(x, d)) \frac{1[d = 0]}{p(d)} \mu(x) &= \sum_{x} (1[X = x, D = 0] - p(x, D=0))\frac{\mu(x)}{p(d=0)} \\
&= \sum_{x} \frac{1[X = x, D = 0]}{p(D =0)} \mu(x) - \sum_{x} p(x, D=0)\frac{\mu(x)}{p(D=0)} \\
&= \frac{1[D = 0]}{p(D =0)} \mu(X) - E[\mu(X) \mid D = 0]
\end{align*}

Now we will consider the third term 
\begin{align*}
- \sum_{x,d} p(x, d) 1[d = 0] \frac{1[D = d] - p(d)}{{p(d)}^2} \mu(x) &= - \sum_{x} p(x, D=0) \frac{1[D = 0] - p(D=0)}{{p(D=0)}^2} \mu(x) \\
&= - \sum_{x} p(x \mid D=0) \frac{1[D = 0] - p(D=0)}{{p(D=0)}} \mu(x) \\
&= - \left(\frac{1[D = 0]}{p(D = 0)} - 1\right) E[\mu(X) \mid D = 0]
\end{align*}

Now we will consider the fourth term, where we (in the first line) replace all instances of $d$ (lowercase) with $0$, and remove the sum over $d$, which eliminates the $1[d = 0]$ term.  Similarly in the next line we remove the indicator $X = x$ by replacing all instances of $x$ with $X$, and removing the sum over $X$.

\begin{align*}
  &\sum_{x,d} p(x, d) 1[d = 0] \frac{1}{p(d)} \frac{1[D = 0, X = x]}{P(D = 0 \mid X)P(X)}\left(Y - E[Y \mid D = 0, X]\right) \\
  &\quad = \sum_{x} p(x, D=0) \frac{1}{p(D = 0)} \frac{1[D = 0, X = x]}{P(D = 0 \mid X)P(X)}\left(Y - E[Y \mid D = 0, X]\right) \\
  &\quad = p(X \mid D=0) \frac{1[D = 0]}{P(D = 0 \mid X)P(X)}\left(Y - E[Y \mid D = 0, X]\right) \\
  &\quad = 1[D = 0] \frac{p(X, D=0)}{p(D = 0)P(D = 0 \mid X)P(X)} \left(Y - E[Y \mid D = 0, X]\right) \\
  &\quad = 1[D = 0] \frac{p(D=0 \mid X)}{p(D = 0)P(D = 0 \mid X)} \left(Y - E[Y \mid D = 0, X]\right) \\
  &\quad = \frac{1[D = 0]}{P(D = 0 \mid X)} \frac{p(D=0 \mid X)}{p(D = 0)} \left(Y - E[Y \mid D = 0, X]\right) 
\end{align*}

Putting it all together gives us the following
\begin{align*}
  IF(\theta_2^u) &= \frac{1}{P(Y = 1 \mid D = 0)} \left(-1 \cdot \frac{1[D = 0]}{P(D = 0)} \frac{Y - E[Y \mid D = 0]}{P(Y = 1 \mid D = 0)} E[\mu(X) \mid D = 0] \right.\\
           & \qquad + \frac{1[D = 0]}{p(D =0)} \mu(X) - E[\mu(X) \mid D = 0] \\
           & \qquad - \left(\frac{1[D = 0]}{p(D = 0)} - 1\right) E[\mu(X) \mid D = 0]\\
 & \qquad + \left. \frac{1[D = 0]}{P(D = 0 \mid X)} \frac{p(D=0 \mid X)}{p(D = 0)} \left(Y - E[Y \mid D = 0, X]\right) \right)
\end{align*}

which simplifies with some cancellations in the second and third lines
\begin{align*}
  IF(\theta_2^u) &= \frac{1}{P(Y = 1 \mid D = 0)} \left(-1 \cdot \frac{1[D = 0]}{P(D = 0)} \frac{Y - E[Y \mid D = 0]}{P(Y = 1 \mid D = 0)} E[\mu(X) \mid D = 0]  \right.\\
           & \qquad + \frac{1[D = 0]}{p(D =0)} \left(\mu(X) - E[\mu(X) \mid D = 0]\right)\\
    &\qquad + \left. \frac{1[D = 0]}{P(D = 0 \mid X)} \frac{p(D=0 \mid X)}{p(D = 0)} \left(Y - E[Y \mid D = 0, X]\right) \right)
\end{align*}

This further simplifies by factoring out the term involving $E[\mu(X) \mid D = 0]$
\begin{align*}
  IF(\theta_2^u) &= \frac{1[D = 0]}{P(Y = 1, D = 0)} \left(\mu(X) - E[\mu(X) \mid D = 0] \left(1 + \frac{Y - E[Y \mid D = 0]}{P(Y = 1 \mid D = 0)}\right) + (Y-E[Y|D=0,X])\right)\\
\end{align*}

This further simplifies by $E[Y|D=0] = P(Y=1|D=0)$ and $\frac{P(Y=1|D=0)}{P(Y=1|D=0)} = 1$. 
\begin{align*}
    IF(\theta_2^u) &= \frac{1[D = 0]}{P(Y = 1, D = 0)} \left(\mu(X) - E[\mu(X) \mid D = 0] \left(\frac{Y}{P(Y = 1 \mid D = 0)}\right) + (Y-E[Y|D=0,X])\right)
\end{align*}

This gives us the following final result
\begin{align*}
    IF(\theta_2^u) &= \frac{1[D = 0]}{P(Y = 1, D = 0)} \mu(X) \\
    & - \frac{1[D = 0]}{P(Y = 1, D = 0)}  E[\mu(X) \mid D = 0] \left(\frac{1[Y=1]}{P(Y = 1 \mid D = 0)}\right) \\
    &+ \frac{1[D = 0]}{P(Y = 1, D = 0)} (Y-E[Y|D=0,X])
\end{align*}

\end{proof}

Next, we move on to discussing our estimator of this upper bound, using our derived influence function. Our procedure (as is standard in literature \citep{kennedy2022semiparametric}) is to use a first order correction of our simple plugin estimator by adding in the expectation of our influence function.

\begin{proposition} \label{prop:estimator_ub_constant_term}
    Our one-step estimator of $\theta_2^u$ is given by 
    \begin{align*}
        \hat{\theta}_2^u = 1.
    \end{align*}
\end{proposition}
\begin{proof}
    We compute the one-step estimator as
    \begin{align*}
        \hat{\theta}_2^u(\hat{P}) = \theta_2^u(\hat{P}) + E_{\hat{P}}[IF( \theta_2^u(\hat{P})].
    \end{align*}
    
    The first term is given by 
    \begin{align*}
        \theta_2^u(\hat{P}) = \frac{1}{\hat{P}(Y=1|D=0)}E_{\hat{P}}[\hat{\mu}(X) | D=0]
    \end{align*}
    
    and the second term is given by 
    \begin{align*}
        E_{\hat{P}}[IF( \theta_2^u(\hat{P})] &= E_{\hat{P}} \left[ \frac{1[D = 0]}{\hat{P}(Y = 1, D = 0)} \hat{\mu}(X) \right] \\
        &- E_{\hat{P}} \left[ \frac{1[D = 0]}{\hat{P}(Y = 1, D = 0)}  E_{\hat{P}}[\hat{\mu}(X) \mid D = 0] \left(\frac{1[Y=1]}{\hat{P}(Y = 1 \mid D = 0)}\right) \right] \\
        &+ E_{\hat{P}} \left[ \frac{1[D = 0]}{\hat{P}(Y = 1, D = 0)} (Y-E_{\hat{P}}[Y|D=0,X]) \right]
    \end{align*}
    
    The first two terms cancel out, using the same logic (that we used to cancel terms out for proving that an influence function has mean zero).
    
    Therefore, we get that 
    \begin{align*}
        \hat{\theta}_2^u(\hat{P}) &= \frac{1}{\hat{P}(Y=1|D=0)}E_{\hat{P}}[\hat{\mu}(X) | D=0] + E_{\hat{P}} \left[ \frac{1[D = 0]}{\hat{P}(Y = 1, D = 0)} (Y-\hat{\mu}(X)) \right] \\
        &=  \frac{1}{\hat{P}(Y=1|D=0)}E_{\hat{P}}[(\hat{\mu}(X) + Y - \hat{\mu}(X)) |D=0] \\
        &= \frac{1}{\hat{P}(Y=1|D=0)}E_{\hat{P}}[Y |D=0] = 1 .
    \end{align*}

    Thus, the estimator for this term is constant.
\end{proof}

\begin{proposition} \label{prop:estimator_ub}
    Our one-step estimator of $\hat{\theta}_1^u(\hat{P})$ is given by
    \begin{align*}
        \hat{\theta}_1^u(\hat{P}) & =  E_P\left[ \frac{\hat{g}(X)\hat{\pi}(X)}{\hat{P}(Y=1, D=0)} + \frac{\hat{\pi}(X)(1[D=0]-\hat{g}(X))}{\hat{P}(Y=1, D=0)} \right] + E_P\left[ \frac{1[D=1] (T - \hat{\pi}(X))}{\hat{P}(Y=1, D=0)}\frac{\hat{g}(X)}{1 - \hat{g}(X)}\right]    \end{align*}
\end{proposition}

\begin{proof}

    We can compute our one-step estimator by
        $\hat{\theta}_1^u(\hat{P}) =\theta_1^u(\hat{P}) + E_{\hat{P}}[IF( \theta_1^u(\hat{P}))]$.

    The first term is given by
    \begin{align*}
        \theta_1^u(\hat{P}) & = \frac{1}{\hat{P}(Y=1, D=0)} E_{\hat{P}}[\hat{g}(X) \hat{\pi}(X)] 
    \end{align*}

    The second term is given by
    \begin{align*}
        E_{\hat{P}}[IF( \theta_1^u(\hat{P}))] & = E_{\hat{P}} \left[ -\frac{1[Y=1, D=0]}{\hat{P}{(Y=1, D=0)}^2} E_{\hat{P}}[\hat{g}(X)\hat{\pi}(X)]  + \frac{\hat{g}(X)\hat{\pi}(X)}{\hat{P}(Y=1, D=0)} + \frac{\hat{\pi}(X)(1[D=0]-\hat{g}(X))}{\hat{P}(Y=1, D=0)} \right] \\
         & \qquad + E_{\hat{P}}\left[ \frac{1[D=1] (T - \hat{\pi}(X))}{\hat{P}(Y=1, D=0)}\frac{\hat{g}(X)}{1 - \hat{g}(X)}\right]
    \end{align*}
    We remark that the first part here is given by
    \begin{align*}
        E_{\hat{P}} \left[ -\frac{1[Y=1, D=0]}{\hat{P}{(Y=1, D=0)}^2} E_{\hat{P}}[\hat{g}(X)\hat{\pi}(X)]\right] & = -E_{\hat{P}}\left[ E_{\hat{P}}[\hat{g}(X)\hat{\pi}(X)] | Y=1, D=0 \right] \\
        & = E_{\hat{P}}[\hat{g}(X)\hat{\pi}(X)] = \theta_1^u(\hat{P})
    \end{align*}
    which cancels out with the first term above. Thus, we derive the estimator as
    \begin{align*}
        \hat{\theta}_1^u(\hat{P}) & =  E_{\hat{P}}\left[ \frac{\hat{g}(X)\hat{\pi}(X)}{\hat{P}(Y=1, D=0)} + \frac{\hat{\pi}(X)(1[D=0]-\hat{g}(X))}{\hat{P}(Y=1, D=0)} \right]  + E_{\hat{P}}\left[ \frac{1[D=1] (T - \hat{\pi}(X))}{\hat{P}(Y=1, D=0)}\frac{\hat{g}(X)}{1 - \hat{g}(X)}\right] \\
         & = E_{\hat{P}}\left[ \frac{\hat{\pi}(X)(1[D=0])}{\hat{P}(Y=1, D=0)} \right]  + E_{\hat{P}}\left[ \frac{1[D=1] (T - \hat{\pi}(X))}{\hat{P}(Y=1, D=0)}\frac{\hat{g}(X)}{1 - \hat{g}(X)}\right] 
    \end{align*}
\end{proof}

Next, we perform error analysis for our derived one-step estimator of the upper bound. 

\begin{lemma}[Error of one-step estimator of upper bound under arbitrary unobserved confounding]\label{lemma:error_ub}
        Let the error of our one-step estimator be given by
        \begin{equation}
         R(\hat{P}, P) = \theta_1^u(\hat{P}) - \theta_1^u(P) + E_P\left[ IF(\theta_1^u(\hat{P})) \right] 
        \end{equation}
        Then, we have that
        \begin{align*}
            R(\hat{P}, P) = o_P(n^{-\frac{1}{2}}), 
        \end{align*}
        when our estimates of $\pi$ and $g$ converge at rates of $o_P(n^{-\frac{1}{4}})$.
\end{lemma}

\begin{proof}
    \begin{align*}
R(\hat{P}, P) & = \theta_1^u(\hat{P})-\theta_1^u(P)+E_P\left[I F\left(\theta_1^u(\hat{P})\right)\right] \\
& =  \frac{E_{\hat{P}}[\hat{g}(X) \hat{\pi}(X)]}{\hat{P}(Y=1, D=0)}-\frac{E_P[g(X) \pi(X)]}{P(Y=1, D=0)}+\frac{1}{\hat{P}(Y=1, D=0)} E_P\left[-\frac{1[Y=1, D=0]}{\hat{P}(Y=1, D=0)} E_{\hat{P}}[\hat{g}(X) \hat{\pi}(X)]\right. \\
& \qquad \left.+\hat{g}(X) \hat{\pi}(X)+1[D=1](T-\hat{\pi}(X)) \frac{\hat{g}(X)}{1-\hat{g}(X)}+\hat{\pi}(X)(1[D=0]-\hat{g}(X))\right] \\
& =  \frac{E_{\hat{P}}[\hat{g}(X) \hat{\pi}(X)]}{\hat{P}(Y=1, D=0)}-\frac{E_P[g(X) \pi(X)]}{P(Y=1, D=0)}-\frac{P(Y=1, D=0)}{\hat{P}(Y=1, D=0)} \frac{E_{\hat{P}}[\hat{P}(X) \hat{\pi}(X)]}{\hat{P}(Y=1, D=0)} \\
& \qquad + \frac{1}{\hat{P}(Y=1, D=0)} E_P\left[\hat{g}(X) \hat{\pi}(X)+1[D=1](T-\hat{\pi}(X)) \frac{\hat{g}(X)}{1-\hat{g}(X)}+\hat{\pi}(X)(1[D=0]-\hat{g}(X))\right] \\
& =  \left(1-\frac{P(Y=1, D=0)}{\hat{P}(Y=1, D=0)}\right) \frac{E_{\hat{P}}[\hat{g}(X) \hat{\pi}(X)]}{\hat{P}(Y=1, D=0)}-\frac{E_P[g(X) \pi(X)]}{P(Y=1, D=0)} \\
& \qquad + \frac{1}{\hat{P}(Y=1, D=0)} E_P\left[\hat{g}(X) \hat{\pi}(X)+1[D=1](T-\hat{\pi}(X)) \frac{\hat{g}(X)}{1-\hat{g}(X)}+\hat{\pi}(X)(1[D=0]-\hat{g}(X))\right] \\
& =  \left(1-\frac{P(Y=1, D=0)}{\hat{P}(Y=1, D=0)}\right) \frac{E_{\hat{P}}[\hat{g}(X) \hat{\pi}(X)]}{\hat{P}(Y=1, D=0)}-\frac{E_P[g(X) \pi(X)]}{P(Y=1, D=0)} \\
& \qquad + \frac{1}{\hat{P}(Y=1, D=0)} E_P\left[\hat{g}(X) \hat{\pi}(X)+(1-g(X))(\pi(X)-\hat{\pi}(X)) \frac{\hat{g}(X)}{1-\hat{g}(X)}+\hat{\pi}(X)(g(X)-\hat{g}(X))\right]
\end{align*}
where we have used that $E_P[T \cdot 1[D=1]]=E_P[1[T=1, D=1]]=E_P[P(T=1, D=1 \mid X)]=$ $E_P[P(T=1 \mid D=1, X) P(D=1 \mid X)]=E_P[T(X)(1-g(X))]$. 
Now, to deal with the first few terms, we are going to add zero (on the second line after the equality below). 
\begin{align*}
& \theta_1^u(\hat{P})-\theta_1^u(P)+E_P\left[\operatorname{IF}\left(\theta_1^u(\hat{P})\right)\right] \\
&= \left(1-\frac{P(Y=1, D=0)}{\hat{P}(Y=1, D=0)}\right) \frac{E_{\hat{P}}[\hat{g}(X) \hat{\pi}(X)]}{\hat{P}(Y=1, D=0)} \\
& \qquad -\frac{E_P[g(X) \pi(X)]}{P(Y=1, D=0)}+\frac{E_P[g(X) \pi(X)]}{\hat{P}(Y=1, D=0)}+\frac{E_P[g(X) \pi(X)]}{\hat{P}(Y=1, D=0)}+\frac{E_P[\hat{g}(X) \hat{\pi}(X)]}{\hat{P}(Y=1, D=0)} \\
& \qquad +\frac{1}{\hat{P}(Y=1, D=0)} E_P\left[(1-g(X))(\pi(X)-\hat{\pi}(X)) \frac{\hat{g}(X)}{1-\hat{g}(X)}+\hat{\pi}(X)(g(X)-\hat{g}(X))\right]
\end{align*}.
The second line after the equality can be re-written as
\begin{align*}
& -\frac{E_P[g(X) \pi(X)]}{P(Y=1, D=0)}+\frac{E_P[g(X) \pi(X)]}{\hat{P}(Y=1, D=0)}+\frac{E_P[g(X) \pi(X)]}{\hat{P}(Y=1, D=0)}+\frac{E_P[\hat{g}(X) \hat{\pi}(X)]}{\hat{P}(Y=1, D=0)} \\
& =-\left(1-\frac{P(Y=1, D=0)}{\hat{P}(Y=1, D=0)}\right) \frac{E_P[g(X) \pi(X)]}{P(Y=1, D=0)}+\frac{1}{\hat{P}(Y=1, D=0)} E_P[g(X) \pi(X)-\hat{g}(X) \hat{\pi}(X)]
\end{align*}
So that the entire expression can be written as
\begin{align*}
& \theta_1^u(\hat{P})-\theta_1^u(P)+E_P\left[I F\left(\theta_1^u(\hat{P})\right)\right] \\
& =\left(1-\frac{P(Y=1, D=0)}{\hat{P}(Y=1, D=0)}\right)\left(\frac{E_{\hat{P}}[\hat{g}(X) \hat{\pi}(X)]}{\hat{P}(Y=1, D=0)}-\frac{E_P[g(X) \pi(X)]}{P(Y=1, D=0)}\right) \\
& \qquad +\frac{1}{\hat{P}(Y=1, D=0)} E_P[g(X) \pi(X)-\hat{g}(X) \hat{\pi}(X)] \\
& \qquad +\frac{1}{\hat{P}(Y=1, D=0)} E_P\left[(1-g(X))(\pi(X)-\hat{\pi}(X)) \frac{\hat{g}(X)}{1-\hat{g}(X)}+\hat{\pi}(X)(g(X)-\hat{g}(X))\right]
\end{align*}
The first line is a product of estimation error of $P(Y=1,D=0)$ and the estimation error of the original plug-in estimator. We remark that the estimation error of of this product overall achieves a fast rate of $o_P(n^{-1/2})$, assuming that our estimator of $P(Y=1, D=0)$ has a rate of $o_P(n^{-1/2})$, which is relatively straightforward since it can be estimated by a simple sample average of the indicator variable $1[Y=1, D=0]$.  
\begin{align}
    \underbrace{\left(1-\frac{P(Y=1, D=0)}{\hat{P}(Y=1, D=0)}\right)}_{=O_P\left(n^{-1 / 2}\right)} \underbrace{\left(\frac{E_{\hat{P}}[\hat{g}(X) \hat{\pi}(X)]}{\hat{P}(Y=1, D=0)}-\frac{E_P[g(X) \pi(X)]}{P(Y=1, D=0)}\right)}_{=o_P(1)}=o_P\left(n^{-1 / 2}\right) \label{eq:convergence_p_phat}
\end{align}

Finally, we can analyze the last two lines from above. Ignoring the $E_P$ and the common multiplier of $\frac{1}{\hat{P}(Y=1, D=0)}$, we have that
\begin{align*}
    & g(X) \pi(X) - \hat{g}(X) \hat{\pi}(X) + (1 - g(X)) (\pi(X) - \hat{\pi}(X)) \frac{\hat{g}(X)}{1 - \hat{g}(X)} + \hat{\pi}(X) (g(X) -\hat{g}(X)) \\
    & \qquad = (1 - g(X))\hat{g}(X) (\pi(X) - \hat{\pi}(X)) + (1 - \hat{g}(X)) g(X) (\hat{\pi}(X) - \pi(X)) 
\end{align*}
where we can ignore the $\frac{1}{(1 - \hat{g}(X))}$ in the denominator. This further simplifies as
\begin{align*}
    & = (\hat{\pi}(X) - \pi(X)) \Big[(1 - \hat{g}(X))g(X) - (1 - g(X))\hat{g}(X)) \Big] \\
    & = (\hat{\pi}(X) - \pi(X))(g(X) - \hat{g}(X))
\end{align*}

We can observe that this is given by a product-of-errors structure in terms of our estimator of $\pi$ and of $g$. This in turn, implies that our overall estimator has asymptotic normality (and converges at a rate of $o_{P}(n^{-1/2})$) if our estimators of $\pi(X)$ and $g(X)$ converge at $o_{P}(n^{-1/4})$ rates.  

\end{proof}

\subsection{Proof of Lemma~\ref{lemma:influence_lower}}\label{appx:influence_lower}

Next, we will derive the influence functions for our lower bounds under no additional assumptions. 
Recall that our estimands are given by
\begin{align*}
      & \theta_1^l(X) \coloneqq \frac{P(D=0 | X)}{P(Y=1, D=0)}\Big( P(T = 1 | D=1, X) + P(Y | D=0, X=x) - 1 \Big), \\
      & \theta_2^l(X) \coloneqq 0 
\end{align*}
Our relevant conditional distributions (i.e., our nuisance functions) are given by
\begin{align*}
    \mu(X) &\coloneqq E[Y=1 |D=0, X], \\
    \pi(X)&\coloneqq E[T = 1 | D=1, X], \\
    g(X)  & \coloneqq E[D=0 | X].
\end{align*}

We now proceed to derive the influence functions for our lower bound under no additional assumptions.

\InfluenceLower*

Our estimand for our lower bound will be as follows,
\begin{align*}
    \theta_1^l &= E \left[ \frac{\pi(X)}{\mu(X)} - \frac{(1-\mu(X))}{\mu(X)} \bigg| Y = 1, D=0 \right] \\
    &= E \left[ \frac{\pi(X)}{\mu(X)} \bigg| Y = 1, D=0 \right] - E \left[ \frac{(1-\mu(X))}{\mu(X)}\bigg| Y = 1, D=0 \right]
\end{align*}

Looking at the second term,
\begin{align*}
    & E \left[ \frac{(1-\mu(X))}{\mu(X)}\bigg| Y = 1, D=0 \right] = \int_{x} p(x|Y=1,D=0) \frac{1- \mu(x)}{\mu(x)} dx \\
    &= \int_{x} P(Y=1,D=0 | x) \frac{p(x)}{P(Y=1,D=0)} \frac{1- \mu(x)}{\mu(x)} dx \\
    &= \frac{1}{P(Y=1,D=0)} E \left[P(Y=1|D=0,x)P(D=0|X) \frac{1-P(Y=1|D=0,x)}{P(Y=1|D=0,x)} \right] \\
    &= \frac{1}{P(Y=1,D=0)} E \left[P(D=0|X) (1-P(Y=1|D=0,x))\right] \\
    &= \frac{1}{P(Y=1,D=0)} E \left[P(D=0|X) (1-\mu(x))\right] \\
    &= \frac{1}{P(Y=1,D=0)}E \left[ g(X)(1-\mu(X)) \right]
\end{align*}

where again note $g(X) = P(D=0|X)$. Putting it together with the first term,
\begin{align*}
    \theta_1^l & = \frac{E[\pi(X) | D=0]}{P(Y=1|D=0)} - \frac{E \left[P(D=0|X) (1-\mu(X))\right]}{P(Y=1,D=0)} \\
    &= \frac{E[g(X)\pi(X)]}{P(Y=1,D=0)} - \frac{1}{P(Y=1,D=0)}E \left[ g(X)(1-\mu(X)) \right] \\
    & = \frac{E[(\pi(X) + \mu(X) - 1)g(X)]}{P(Y=1, D=0)}
\end{align*}

Now, we will derive the influence function for our lower bound $\theta_1^l$.   
First, we observe that the influence function of $\theta_1^l$ can be written as follows
\begin{equation}
    IF(\theta_1^l) = IF(\theta_1^u) + IF \left( \frac{E_{P} \left[ (\mu(X) - 1)g(X) \right]}{P(Y=1,D=0)} \right)
\end{equation}

Taking the second term, we have that 
\begin{align*}
    & IF \left( \frac{E_{P} \left[ (\mu(X) - 1) g(X) \right]}{P(Y=1,D=0)} \right) \\
    & = IF(\frac{1}{P(Y = 1, D = 0)})E_{P}[g(X)(\mu(X)-1)] + \frac{1}{P(Y = 1, D = 0)} IF[E_{P}[g(X)(\mu(X) - 1)]] \\
    &= -\frac{1[Y=1,D=0]-P(Y=1,D=0)}{P{(Y=1,D=0)}^{2}}E_{P}[g(X)(\mu(X)-1)]\\
    &+ \frac{1}{P(Y = 1, D = 0)} \sum_{x}(1[X=x] - p(x))(g(x)(\mu(x) - 1)) \\ 
    &+ \frac{1}{P(Y = 1, D = 0)} \sum_{x} p(x) \left( \frac{1[X=x]}{P(X=x)} (1[D=0]-g(x)) \right) (\mu(x) - 1) \\
    &+  \frac{1}{P(Y = 1, D = 0)} \sum_{x} p(x)g(x) \left( \frac{1[D=0, X=x]}{P(D=0,X)}(Y-\mu(x)) \right)
\end{align*}

where in the last term, we cancel out $1$, since $IF(1) = 0$. Further simplifying gives,

\begin{align*}
    & IF \left( \frac{E_{P} \left[ (\mu(X) - 1) g(X) \right]}{P(Y=1,D=0)} \right) \\
    &= \frac{-1[Y=1,D=0]}{P{(Y=1,D=0)}^{2}}E_{P}[g(X)(\mu(X) - 1)] + \frac{1}{P(Y=1,D=0)}E_{P}[g(X)(\mu(X) - 1)] \\
    & \qquad + \frac{g(X)(\mu(X)-1)}{P(Y=1,D=0)} - \frac{E_{P}[g(X)(\mu(X) - 1)]}{P(Y=1,D=0)} \\ 
    & \qquad + \frac{(\mu(X) - 1)(1[D=0] - g(X))}{P(Y=1,D=0)} \\
    & \qquad + \frac{1[D=0](Y - \mu(X))}{P(Y=1,D=0)}
\end{align*}

with some cancellations in the first and second line, and some re-ordering of the third and fourth terms, we can then write that 

\begin{align*}
     IF \left( \frac{E_{P} \left[ (\mu(X) - 1) g(X) \right]}{P(Y=1,D=0)} \right)&  = \frac{1}{P(Y=1,D=0)} \left( \frac{-1[Y=1,D=0]}{P(Y=1,D=0)} E_{P}[g(X)(\mu(X) - 1)] \right. \\
    & \left. \qquad + g(X)(\mu(X) - 1) + 1[D=0](Y-\mu(X)) + (\mu(X) - 1)(1[D=0] - g(X)) \right) \\
    &= \frac{1}{P(Y=1,D=0)} \left( \frac{-1[Y=1,D=0]}{P(Y=1,D=0)} E_{P}[g(X)(\mu(X) - 1)] + 1[D=0](Y-1) \right)
\end{align*}

Therefore, the final influence function is given by
\begin{align*}
    IF(\theta_1^l) & = IF(\theta_1^u) + \frac{1}{P(Y=1,D=0)} \left( \frac{-1[Y=1,D=0]}{P(Y=1,D=0)} E_{P}[g(X)(\mu(X) - 1)] + 1[D=0](Y-1) \right)
\end{align*}

Now, we will compute the one-step estimator as follows.

\begin{proposition} \label{prop:estimator_lb}
    Our one-step estimator of $\theta_1^l$ is given by
   
\begin{align*}
    \hat{\theta}_1^l(\hat{P}) &=
     \frac{1}{\hat{P}(Y=1,D=0)}  \Big( 1[D=1](T -\pi(X))\frac{g(X)}{1-g(X)} +  1[D=0]\pi(X) + 1[D=0](Y- 1) \Big).
\end{align*}
\end{proposition}

\begin{proof}

We compute the one-step estimator as 
    $\hat{\theta}_1^l(\hat{P}) = \theta_1^l(\hat{P}) + E_{\hat{P}}[IF( \theta_1^l(\hat{P}))].$
    
The first term is given by

\begin{align*}
    \theta_1^l(\hat{P}) = \frac{E[g(X)(\pi(X) + \mu(X)-1)]}{P(Y=1,D=0)}
\end{align*}

and the second term is given by

\begin{align*}
    E_{\hat{P}}[IF( \theta_1^l(\hat{P}))] &= 
     \frac{1}{\hat{P}(Y=1,D=0)} \left( \frac{-1[Y=1,D=0]}{\hat{P}(Y=1,D=0)} E_{\hat{P}}[\hat{g}(X)(\hat{\pi}(X))] \right. \\
    & \left. + \hat{g}(X)\hat{\pi}(X) + 1[D=1](T - \hat{\pi}(X))\frac{\hat{g}(X)}{1-\hat{g}(X)} + \hat{\pi}(X)(1[D=0] - \hat{g}(X)) \right) \\
    & +  \frac{1}{\hat{P}(Y=1,D=0)} \left( \frac{-1[Y=1,D=0]}{\hat{P}(Y=1,D=0)} E_{\hat{P}}[\hat{g}(X)(\hat{\mu}(X) - 1)] \right. \\
    & \left. + \hat{g}(X)(\hat{\mu}(X) - 1) + 1[D=0](Y-\hat{\mu}(X)) + (\hat{\mu}(X) - 1)(1[D=0] - \hat{g}(X)) \right) \\
\end{align*}

We see that the first and third term cancels out with $\theta_1^l(\hat{P})$. Thus, we have that 

\begin{align*}
    \hat{\theta}_1^l(\hat{P}) &= 
     \frac{1}{\hat{P}(Y=1,D=0)} \left(\hat{g}(X)\hat{\pi}(X) + 1[D=1](T -\hat{\pi}(X))\frac{\hat{g}(X)}{1-\hat{g}(X)} + \hat{\pi}(X)(1[D=0] - \hat{g}(X)) \right) \\
    & +  \frac{1}{\hat{P}(Y=1,D=0)} \left(
     \hat{g}(X)(\hat{\mu}(X) - 1) + 1[D=0](Y-\hat{\mu}(X)) + (\hat{\mu}(X) - 1)(1[D=0] - \hat{g}(X)) \right) \\
     & = \frac{1}{\hat{P}(Y=1,D=0)}  \left( 1[D=1](T - \hat{\pi}(X))\frac{\hat{g}(X)}{1-\hat{g}(X)} +  1[D=0]\hat{\pi}(X) \right. \\
     & \left. \qquad + 1[D=0](Y-\hat{\mu}(X)) + 1[D=0](\hat{\mu}(X) - 1) \right)
\end{align*}

Therefore, our first-order unbiased estimator is given by 

\begin{align*}
    & \hat{\theta}_1^l(\hat{P}) \\
    & = \frac{1}{\hat{P}(Y=1,D=0)}  \left( 1[D=1](T - \hat{\pi}(X))\frac{\hat{g}(X)}{1-\hat{g}(X)} +  1[D=0]\hat{\pi}(X)  + 1[D=0](Y-\hat{\mu}(X) + \hat{\mu}(X) - 1) \right) \\
    & = \frac{1}{\hat{P}(Y=1,D=0)}  \Big( 1[D=1](T - \hat{\pi}(X))\frac{\hat{g}(X)}{1-\hat{g}(X)} +  1[D=0]\hat{\pi}(X) + 1[D=0](Y-1) \Big).
\end{align*}

\end{proof}

\begin{lemma}[Error of one-step estimator of lower bound under arbitrary unobserved confounding]\label{lemma:error_lb}
        Let the error of our one-step estimator be given by 
        \begin{equation*}
          R(\hat{P}, P) = \theta_1^l(\hat{P}) - \theta_1^l(P) + E_P\left[ IF(\theta_1^l(\hat{P})) \right]
        \end{equation*}
        Then, we have that
        \begin{align*}
            R(\hat{P}, P) = o_P(n^{-\frac{1}{2}}),
        \end{align*}
        when our estimates of $g$ and $\pi$ converge at rates of $o_P(n^{-\frac{1}{4}})$
\end{lemma}

\begin{proof}
    We will analyze the remainder term of the one-step estimator. We leverage the fact that $\theta_1^l$ is the sum of an additional term and $\theta_2^u$ and that influence functions are additive:  
    \begin{align*}
        R(\hat{P}, P) &= \theta_1^l(\hat{P}) - \theta_1^l(P) + E_{P}[IF(\theta_1^l(\hat{P}))] \\
        & = \theta_2^u(\hat{P}) - \theta_2^u(P) + E_{P}[IF \theta_2^u(\hat{P})] \\
        & \quad + E_{\hat{P}} \left[ \frac{\hat{g}(X)(\hat{\mu}(X) - 1)}{\hat{P}(Y=1,D=0)} \right] - E_{P} \left[ \frac{g(X)(\mu(X) - 1)}{P(Y=1,D=0)}\right] \\
        & \quad + \frac{1}{\hat{P}(Y=1, D=0)} E_{P}\left[ \frac{-1[Y=1, D=0]}{\hat{P}(Y=1, D=0)} E_{\hat{P}}[\hat{g}(X)(\hat{\mu}(X) - 1)] + 1[D=0](Y-1)\right]
    \end{align*}
    We note that from our error analysis in Lemma \ref{lemma:error_ub}, the error term from the terms involving $\theta_2^u$ all converge at fast rates when our estimates of $\pi$ and $g$ converge at rates of $o_P(n^{-\frac{1}{4}})$. Thus, it suffices to look at the remaining terms (and drop the asymptotic term after the first line):
    \begin{align*}
        R(\hat{P}, P) &= o_P(n^{-\frac{1}{2}})\\
        & \quad + E_{\hat{P}} \left[ \frac{\hat{g}(X)(\hat{\mu}(X) - 1)}{\hat{P}(Y=1,D=0)} \right] - E_{P} \left[ \frac{g(X)(\mu(X) - 1)}{P(Y=1,D=0)}\right] \\
        & \quad + \frac{1}{\hat{P}(Y=1, D=0)} E_{P}\left[ \frac{-1[Y=1, D=0]}{\hat{P}(Y=1, D=0)} E_{\hat{P}}[\hat{g}(X)(\hat{\mu}(X) - 1)] + 1[D=0](Y-1)\right] \\
        & = E_{\hat{P}} \left[ \frac{\hat{g}(X)(\hat{\mu}(X) - 1)}{\hat{P}(Y=1,D=0)} \right] - E_{P} \left[ \frac{g(X)(\mu(X) - 1)}{P(Y=1,D=0)}\right] \\
        & \quad -\frac{P(Y=1, D=0)}{\hat{P}(Y=1, D=0)} E_{\hat{P}}\left[ 
 \frac{\hat{g}(X)(\hat{\mu}(X) - 1)}{\hat{P}(Y=1, D=0)}\right] + E_{P}\left[ \frac{1[D=0](Y-1)}{\hat{P}(Y=1, D=0)}\right]
    \end{align*}
    Rearranging terms gives us that
    \begin{align*}
        R(\hat{P}, P) &= E_{\hat{P}} \left[ \frac{\hat{g}(X)(\hat{\mu}(X) - 1)}{\hat{P}(Y=1,D=0)} \right] -\frac{P(Y=1, D=0)}{\hat{P}(Y=1, D=0)} E_{\hat{P}}\left[ 
 \frac{\hat{g}(X)(\hat{\mu}(X) - 1)}{\hat{P}(Y=1, D=0)}\right] \\
        & \quad + E_{P}\left[ \frac{1[D=0](Y-1)}{\hat{P}(Y=1, D=0)}\right] - E_{P} \left[ \frac{g(X)(\mu(X) - 1)}{P(Y=1,D=0)}\right] \\
        & = \left(1 - \frac{P(Y=1, D=0)}{\hat{P}(Y=1, D=0)} \right) E_{\hat{P}} \left[ \frac{\hat{g}(X)(\hat{\mu}(X) - 1)}{\hat{P}(Y=1,D=0)} \right] \\
        & \quad + E_P\left[ \frac{g(X) (\mu(X) - 1)}{\hat{P}(Y=1, D=0)} - \frac{g(X)(\mu(X) - 1)}{P(Y=1, D=0)}  \right] \\
        & = \left(1 - \frac{P(Y=1, D=0)}{\hat{P}(Y=1, D=0)} \right) E_{\hat{P}} \left[ \frac{\hat{g}(X)(\hat{\mu}(X) - 1)}{\hat{P}(Y=1,D=0)} \right] \\
        & \quad + \left(\frac{P(Y=1, D=0) - \hat{P}(Y=1, D=0)}{P(Y=1, D=0)\hat{P}(Y=1, D=0)}\right) E_P\left[ g(X) (\mu(X) - 1) \right]
    \end{align*}

    We finally note that our estimator of $P(Y=1, D=0)$ has a rate of $O_P(n^{-\frac{1}{2}})$. Thus, we get that both of the above terms will have fast rates and that our overall error term will converge at a rate of $o_P(n^{-\frac{1}{2}})$, given estimators of $g, \pi$ that converge at rates of $o_P(n^{-\frac{1}{4}})$, where the convergence rates of $g, \pi$ are necessary for the remainder terms from $\theta_2^u$. Note that there is no reliance on the convergence rate of $\mu$ as in our one-step correct estimator, we only use $Y$ and not $\mu(X)$. 
    
\end{proof}

\subsection{Algorithm for Estimators in Propositions \ref{prop:estimator_lb}, \ref{prop:estimator_ub}, \ref{prop:estimator_ub_gamma}, and \ref{prop:estimator_lb_gamma}}

We perform estimation of our upper and lower bounds as follows, using cross-fitting:

\begin{enumerate}
    \item We first split our data $M = \{(X, T, Y, D) \}$ into $M_0 = \{ (X_i, T_i, Y_i, D_i) | \forall i \text{ where } D_i = 0\}$ and $M_1 = \{(X_i, T_i, Y_i, D_i) | \forall i \text{ where } D_i = 1 \}$, which represent pre-availability and post-availability data. 
    \item Next, we split these data into $N$ disjoint folds of equal sample size to perform cross-fitting. 
    \item For each fold $k$, we estimate the upper and lower bounds in Lemmas~\ref{lemma:influence_upper} and~\ref{lemma:influence_lower} and average them: 
    \begin{align*}
        \hat{\psi}^u = \frac{1}{K}\sum_{k=1}^{K} \hat{\psi}_k^u, \quad \hat{\psi}^l = \frac{1}{K}\sum_{k=1}^{K} \hat{\psi}_k^l,
    \end{align*} 
    where $\hat{\psi}^u_k, \hat{\psi}^l_k$ represent our estimates of the upper and lower bounds evaluated on fold $k$ and where our nuisance functions used in estimating $\psi$ are trained on all folds except $k$. We do the same for our sensitivity analysis bounds of $\hat{\psi}^{u, \gamma}$ and $\hat{\psi}^{l, \gamma}$.
    \item In computing $\hat{\psi}^{u}, \hat{\psi}^{u, \gamma}$, $\hat{\psi}^{l}$, $\hat{\psi}^{l, \gamma}$ on fold $k$, and we estimate the following nuisance functions: 
        \begin{itemize}
            \item Estimate $\pi(x)$ on $M_{1,\neg k}$ and evaluate on $M_{0,k} \cup M_{1,k}$.
            \item Estimate $\mu(x)$ on $M_{0, \neg k}$ and 
            evaluate on $M_{0,k} \cup M_{1,k}$.
            \item Estimate $g(x)$ on $M_{0,\neg k} \cup M_{1, \neg k}$ and evaluate on $M_{0,k} \cup M_{1, k}$
            \item Estimate $P(Y=1,D=0)$ on $M_{0,\neg k} \cup M_{1, \neg k}$ and evaluate on $M_{0,k} \cup M_{1, k}$. 
        \end{itemize}
\end{enumerate}

\subsection{Identification of Bounds with a Sensitivity Analysis Model}\label{appendix:gamma_identification}

Next, we will derive our results under certain assumptions on the strengths of underlying confounders by adopting a sensitivity analysis model.
We impose a condition on confounding in treatment assignment, 
\begin{align*}
    \frac{1}{\gamma} \leq \frac{P(T = 1|Y(0) = 0, D=1, X)}{P(T=1|Y(0) = 1, D=1, X)} \leq \gamma.
\end{align*}

We now represent the result from the main body in the identification of our bounds under our sensitivity model.

\IdentificationGamma*

\begin{proof}
    
Now, using the same expansion of $P(T = 1| D = 1, X)$ as before, we have
\begin{align*}
    P(T = 1|D=1, X) & = P(Y(0) = 1|D=1, X)P(T = 1|Y(0) = 1, D=1, X) \\
    & \qquad + P(Y(0) = 0|D=1, X)P(T = 1|Y(0) = 0, D=1, X).
\end{align*}
Consider first the upper bound. As before, we know that $P(T = 1|Y(0) = 0, D = 1, X) \geq 0$. However, given the sensitivity assumption, we also have $P(T = 1|Y(0) = 0, D=1, X) \geq \frac{1}{\gamma}P(T = 1|Y(0) = 1, D=1, X)$. Since $\frac{1}{\gamma}P(T = 1|Y(0) = 1, D=1, X) \geq 0$, the tightest bound combining these two constraints is 
\begin{align*}
    P(T = 1|D=1, X) & \geq P(Y(0) = 1|D=1, X)P(T = 1|Y(0) = 1, D=1, X) \\
    & \qquad +  P(Y(0) = 0|D=1, X)\frac{1}{\gamma}P(T = 1|Y(0) = 1, D=1, X)
\end{align*}
which yields
\begin{align*}
    P(T = 1|Y(0) = 1, D=1, X) &\leq \frac{P(T = 1|D=1, X)}{P(Y(0) = 1|D=1, X) + \frac{1}{\gamma} P(Y(0) = 0|D=1, X)}\\
    &=\frac{P(T = 1|D=1, X)}{P(Y(0) = 1|D=1, X) + \frac{1}{\gamma} (1 - P(Y(0) = 1|D=1, X))}.
\end{align*}
As $\gamma \to 1$ (no unmeasured confounding), this bound converges to $P(T = 1|Y(0) = 1, D=1, X) \leq P(T = 1|D=1, X)$. As $\gamma \to \infty$ (arbitrary unmeasured confounding), it converges to our earlier bound without the sensitivity assumption. Thus, we remark that the term $\theta_1^{u}$ is unnecessary to include in the minimum operator in the upper bound as it is always dominated by the sensitivity analysis term $\theta_3^{u, \gamma}$, although we include it for clarity for the reader.

Turning now to the lower bound, we have that $P(T = 1|Y(0) = 0, D=1, X) \leq 1$ as before. The sensitivity assumption adds the further constraint $P(T = 1|Y(0) = 0, D=1, X) \leq \gamma P(T = 1|Y(0) = 1, D=1, X)$. The first constraint is not necessarily redundant because as $\gamma \to \infty$, $\gamma P(T = 1|Y(0) = 1, D=1, X)$ will exceed 1. Therefore, we obtain a tighter bound by taking the stronger of the two constraints:

\begin{align*}
    P(T = 1|Y(0) = 1, D=1, X) &\geq \max \Bigg\{ \frac{P(T = 1|D=1, X)}{P(Y(0) = 1|D=1, X) + \gamma (1 - P(Y(0) = 1|D=1, X))},  \\
    & \quad \quad \quad \quad \quad \frac{P(T = 1|D=1, X) - P(Y(0) = 0|D=1, X)}{P(Y(0) = 1|D=1, X)}\Bigg\}. 
\end{align*}
As $\gamma \to 1$, we have $P(T = 1|Y(0) = 1, D=1, X) \geq P(T = 1| D=1, X)$. Combined with the $\gamma \to 1$ upper bound, this shows we achieve point identification at the expected value under no unmeasured confounding.  As $\gamma \to \infty$, the first term in the max eventually becomes vacuous, and we revert to the bound from before. 
\end{proof}

\subsection{Estimation of Bounds with a Sensitivity Analysis Model }\label{appx:gamma_estimation}

Now that we have shown the identification results under our sensitivity analysis model in \cref{theorem:identification_gamma}, we can construct our estimator of the upper and lower bounds, given by $\psi^{u, \gamma}, \psi^{l, \gamma}$. First, we consider estimating the upper bound. 

Recall that our estimands are given by
\begin{align*}
    & \theta_1^{l, \gamma} \coloneqq \theta_1^l \\
    & \theta_2^{l, \gamma} \coloneqq \theta_2^l \\
    & \theta_3^{l, \gamma} \coloneqq \frac{P(T = 1 | D=1, X)}{P(Y(0) = 1 |D=1, X) + \gamma(1 -  P(Y(0)= 1 | D=1, X))} \\
    & \theta_1^{u, \gamma} \coloneqq \theta_1^{u} \\
    & \theta_2^{u, \gamma} \coloneqq \theta_2^{u} \\
    & \theta_3^{u, \gamma} \coloneqq \frac{P(T = 1|D=1, X)}{P(Y(0) = 1|D=1, X) + \frac{1}{\gamma} (1 - P(Y(0) = 1|D=1, X))} \\
\end{align*}

Our relevant conditional distributions (i.e., our nuisance functions) are given by
\begin{align*}
    \mu(X) &\coloneqq E[Y=1 |D=0, X], \\
    \pi(X) &\coloneqq E[T = 1 | D=1, X], \\
    g(X)   &\coloneqq E[D=0 | X]
\end{align*}

We now proceed to derive the influence functions for our upper and lower bounds under our sensitivity analysis model.

\begin{lemma}
    Let our estimand $\theta_3^{u, \gamma}(P)$ be given by 
    \begin{align}\label{eq:theta_3_u_gamma}
        \theta_3^{u, \gamma}(P) = \frac{1}{P(Y = 1 , D = 0)} E\left[ \frac{\gamma \pi(X) \mu(X)}{(\gamma - 1) \mu(X) + 1} \right]
    \end{align}
    Then, we have that our influence function is given by
    \begin{align*}
        IF\left(\theta_3^{u, \gamma}(P)\right) & = -\frac{1[Y=1, D=0]}{P{(Y=1, D=0)}^2} E_P[g(X)A(X)]  + \frac{g(X)A(X)}{P(Y=1, D=0)} + \frac{A(X)(1[D=0]-g(X))}{P(Y=1, D=0)}  \\
        & \qquad + \frac{1[D=1]}{P(Y=1, D=0)} \frac{\gamma \mu(X)}{((\gamma - 1)\mu(X) + 1)}(T - \pi(x))\frac{g(X)}{1 - g(X)} \\
         & \qquad + \frac{1[D=0]}{P(Y=1, D=0)}\frac{\gamma \pi(X)}{{((\gamma - 1)\mu(X) + 1)}^2}(Y - \mu(x))  
    \end{align*}

    where $A(X) = A_{\gamma}(X) = \frac{\gamma \pi(X) \mu(X)}{(\gamma - 1) \mu(X) + 1}$
    
\end{lemma}

\begin{proof}

    First, we will simplify the upper bound term. It can be written as
    \begin{align*}
        \frac{\pi(X)}{\mu(X) + \frac{1}{\gamma}(1 - \mu(X))} & = \frac{\gamma \pi(X)}{\gamma \mu(X) + (1 - \mu(X))} = \frac{\gamma \pi(X)}{(\gamma - 1) \mu(X) + 1}.
    \end{align*}

    Our target function of interest is given by
    \begin{align*}
        E\left[ \frac{\gamma \pi(X)}{(\gamma - 1) \mu(X) + 1} | Y = 1, D=0 \right] & = \int_x P(x | Y=1, D=0) \frac{\gamma \pi(x)}{(\gamma - 1) \mu(x) + 1} \\
        & = \int_x P(Y=1, D=0 | x) \frac{P(x)}{P(Y=1, D=0)} \frac{\gamma \pi(x)}{(\gamma - 1) \mu(x) + 1} \\
        & = \frac{1}{P(Y = 1, D=0)} E\left[  P(Y=1, D=0 | X) \frac{\gamma \pi(X)}{(\gamma - 1) \mu(X) + 1} \right] \\
        & = \frac{1}{P(Y = 1, D=0)} E\left[  P(Y=1 | D=0, x) P(D=0 | X) \frac{\gamma \pi(X)}{(\gamma - 1) \mu(X) + 1} \right] \\
        & = \frac{1}{P(Y = 1, D=0)} E\left[  \mu(X) g(X) \frac{\gamma \pi(X)}{(\gamma - 1) \mu(X) + 1} \right] \\
        & = \frac{1}{P(Y = 1, D=0)} E\left[ g(X) \frac{\gamma \pi(X)\mu(X)}{(\gamma - 1) \mu(X) + 1} \right]
    \end{align*}

    Let 
    \begin{equation*}
      A(X) = \frac{\gamma \pi(X) \mu(X)}{(\gamma - 1) \mu(X) + 1},
    \end{equation*}
    and let $g(X) = P(D=0 | X)$, as is done previously. We begin as follows 
    \begin{align*}
        IF(\theta_3^{u, \gamma}(P)) & = \frac{E_P[g(X)A(X)]}{P(Y=1, D=0)}
    \end{align*}

    We remark that this is the same form as in the proof for the upper bound with arbitrary unobserved confounding, except that we have switched $\pi(X)$ for $A(X)$. Therefore, we can apply an intermediate result
    \begin{align*}
        IF(\theta_3^{u, \gamma}(P)) & = -\frac{1[Y=1, D=0] - P(Y=1, D=0)}{P{(Y=1, D=0)}^2} E_P[g(X)A(X)] \\
        &\qquad + \frac{1}{P(Y=1, D=0)} \sum_{x} (1[X=x] - p(x))(g(x)A(x)) \\
        &\qquad + \frac{1}{P(Y=1, D=0)} \sum_{x} p(x)\left( \frac{1[X=x]}{P(X=x)}(1[D=0] - g(x)) \right) A(x) \\
        & \qquad + \frac{1}{P(Y=1, D=0)} \sum_x p(x)g(x) IF(A(x))
    \end{align*}

    This simplifies as follows (combining the first three lines)
    \begin{align*}
        IF\left(\theta_3^{u, \gamma}(P)\right) & = -\frac{1[Y=1, D=0]}{P{(Y=1, D=0)}^2} E_P[g(X)A(X)] +\frac{1}{P(Y=1, D=0)} E_P[g(X)A(X)] \\
        &\qquad + \frac{g(X)A(X)}{P(Y=1, D=0)} - \frac{E[g(X)A(X)]}{P(Y=1, D=0)}\\
        &\qquad + \frac{A(X)(1[D=0]-g(X))}{P(Y=1, D=0)}  \\
        & \qquad + \frac{1}{P(Y=1, D=0)} \sum_x p(x)g(x) IF(A(x)) \\
        & = -\frac{1[Y=1, D=0]}{P{(Y=1, D=0)}^2} E_P[g(X)A(X)]  + \frac{g(X)A(X)}{P(Y=1, D=0)} + \frac{A(X)(1[D=0]-g(X))}{P(Y=1, D=0)}  \\
        & \qquad + \frac{1}{P(Y=1, D=0)} \sum_x p(x)g(x) IF(A(x))
    \end{align*}

    Next, we address $IF(A(X))$. This is computed as
    \begin{align*}
        IF\left(\frac{\gamma \pi(X)\mu(X)}{(\gamma - 1)\mu(X) + 1}\right)  & = \frac{\gamma IF(\pi(X)\mu(X))((\gamma - 1)\mu(x) + 1)}{{((\gamma - 1)\mu(X) + 1)}^2} - \frac{\gamma \pi(X)\mu(X)(\gamma - 1)IF(\mu(X))}{{((\gamma - 1)\mu(X) + 1)}^2} \\
        & = \frac{\gamma IF(\pi(X)\mu(X))}{((\gamma - 1)\mu(X) + 1)} - \frac{\gamma \pi(X)\mu(X)(\gamma - 1)IF(\mu(X))}{{((\gamma - 1)\mu(X) + 1)}^2} \\
        & = \frac{\gamma IF(\pi(X))\mu(X) + \gamma \pi(X) IF(\mu(X))}{((\gamma - 1)\mu(X) + 1)} - \frac{\gamma \pi(X)\mu(X)(\gamma - 1)IF(\mu(X))}{{((\gamma - 1)\mu(X) + 1)}^2} \\
        & = \frac{\gamma IF(\pi(X))\mu(X)}{((\gamma - 1)\mu(X) + 1)}
        + \frac{\gamma \pi(X) IF(\mu(X))}{((\gamma - 1)\mu(X) + 1)} 
        - \frac{\gamma \pi(X)\mu(X)(\gamma - 1)IF(\mu(X))}{{((\gamma - 1)\mu(X) + 1)}^2} \\
        & = \frac{\gamma \mu(X)}{((\gamma - 1)\mu(X) + 1)} IF(\pi(X))
        +  \frac{\gamma \pi(X) }{((\gamma - 1)\mu(X) + 1)} IF(\mu(X))
         \\
         & \qquad - \frac{\gamma \pi(X)\mu(X)(\gamma - 1)}{{((\gamma - 1)\mu(X) + 1)}^2}IF(\mu(X))
    \end{align*}
    We now compute the influence functions for $\pi(X)$ and $\mu(X)$ in the last line.
    \begin{align*}
        IF\left(\frac{\gamma \pi(X)\mu(X)}{(\gamma - 1)\mu(X) + 1}\right) & = \frac{\gamma \mu(X)}{((\gamma - 1)\mu(X) + 1)} \left( \frac{1[D=1, X]}{P(D=1, X)} (T - \pi(X))\right) \\
        & \qquad +  \frac{\gamma \pi(X) }{((\gamma - 1)\mu(X) + 1)} \left( \frac{1[D=0, X]}{P(D=0, X)} (Y - \mu(X))\right) \\
         & \qquad - \frac{\gamma \pi(X)\mu(X)(\gamma - 1)}{{((\gamma - 1)\mu(X) + 1)}^2} \left( \frac{1[D=0, X]}{P(D=0, X)} (Y - \mu(X))\right) \\
    \end{align*}
    Then, plugging in this above and removing our indicator function on $X = x$ gives us 
    \begin{align*}
         \frac{1}{P(Y=1, D=0)} \sum_x p(x)g(x)&  IF(A(x)) \\
         & = \frac{1}{P(Y=1, D=0)} \Bigg( \frac{\gamma \mu(X)}{((\gamma - 1)\mu(X) + 1)} \left( \frac{1[D=1]P(X)g(X)}{P(D=1, X)} (T - \pi(X))\right) \\
         & \qquad + \frac{\gamma \pi(X) }{((\gamma - 1)\mu(X) + 1)} \left( \frac{1[D=0]P(X)g(X)}{P(D=0, X)} (Y - \mu(X))\right)  \\
         & \qquad - \frac{\gamma \pi(X)\mu(X)(\gamma - 1)}{{((\gamma - 1)\mu(X) + 1)}^2} \left( \frac{1[D=0]P(X)g(X)}{P(D=0, X)} (Y - \mu(X))\right) \Bigg)\\
         & = \frac{1}{P(Y=1, D=0)} \Bigg( \frac{\gamma \mu(X)}{((\gamma - 1)\mu(X) + 1)} \left( \frac{1[D=1]g(X)}{P(D=1| X)} (T - \pi(X))\right) \\
         & \qquad + \frac{\gamma \pi(X) }{((\gamma - 1)\mu(X) + 1)} \left( \frac{1[D=0]g(X)}{P(D=0 | X)} (Y - \mu(X))\right)  \\
         & \qquad - \frac{\gamma \pi(X)\mu(X)(\gamma - 1)}{{((\gamma - 1)\mu(X) + 1)}^2} \left( \frac{1[D=0]g(X)}{P(D=0 | X)} (Y - \mu(X))\right) \Bigg)\\
         & = \frac{1}{P(Y=1, D=0)} \Bigg( \frac{\gamma \mu(X)}{((\gamma - 1)\mu(X) + 1)} \left( 1[D=1](T - \pi(X))\right)\frac{g(X)}{1 - g(X)} \\
         & \qquad + \frac{\gamma \pi(X) }{((\gamma - 1)\mu(X) + 1)} \left( 1[D=0] (Y - \mu(X))\right)  \\
         & \qquad - \frac{\gamma \pi(X)\mu(X)(\gamma - 1)}{{((\gamma - 1)\mu(X) + 1)}^2} \left( 1[D=0](Y - \mu(X))\right) \Bigg) \\
         & = \frac{1[D=1]}{P(Y=1, D=0)} \frac{\gamma \mu(X)}{((\gamma - 1)\mu(X) + 1)}(T - \pi(X)) \frac{g(X)}{1 - g(X)}\\
         & \qquad + \frac{1[D=0]}{P(Y=1, D=0)} \frac{\gamma \pi(X) }{((\gamma - 1)\mu(X) + 1)} (Y - \mu(X))  \\
         & \qquad - \frac{1[D=0]}{P(Y=1, D=0)} \frac{\gamma \pi(X)\mu(X)(\gamma - 1)}{{((\gamma - 1)\mu(X) + 1)}^2}(Y - \mu(X))  
    \end{align*}

Then, we note that we can combine the two bottom lines, where 
    \begin{align*}
        \frac{\gamma \pi(X)}{(\gamma - 1)\mu(X) + 1} -\frac{\gamma \pi(X) \mu(X) (\gamma - 1)}{{((\gamma - 1) \mu(X) + 1)}^2} & = \frac{\gamma \pi(X) \mu(X) (\gamma - 1) + \gamma \pi(X)}{{((\gamma - 1)\mu(X) + 1)}^2} - \frac{\gamma \pi(X) \mu(X) (\gamma - 1)}{{((\gamma - 1) \mu(X) + 1)}^2}\\
        & =  \frac{\gamma \pi(X)}{{((\gamma - 1)\mu(X) + 1)}^2}
    \end{align*}
    which gives that
    \begin{align*}
         \frac{1}{P(Y=1, D=0)} \sum_x p(x)g(x) IF(A(x)) & = \frac{1[D=1]}{P(Y=1, D=0)} \frac{\gamma \mu(X)}{((\gamma - 1)\mu(X) + 1)}(T - \pi(X)) \frac{g(X)}{1 - g(X)} \\
         & + \frac{1[D=0]}{P(Y=1, D=0)}\frac{\gamma \pi(X)}{{((\gamma - 1)\mu(X) + 1)}^2}(Y - \mu(X))  
    \end{align*}

    Finally, we can put everything together to get
    \begin{align*}
        IF\left(\theta_3^{u, \gamma}(P)\right) & = -\frac{1[Y=1, D=0]}{P{(Y=1, D=0)}^2} E_P[g(X)A(X)]  + \frac{g(X)A(X)}{P(Y=1, D=0)} + \frac{A(X)(1[D=0]-g(X))}{P(Y=1, D=0)}  \\
        & \qquad + \frac{1[D=1]}{P(Y=1, D=0)} \frac{\gamma \mu(X)}{((\gamma - 1)\mu(X) + 1)}(T - \pi(X))\frac{g(X)}{1 - g(X)} \\
         & \qquad + \frac{1[D=0]}{P(Y=1, D=0)} \frac{\gamma \pi(X)}{{((\gamma - 1)\mu(X) + 1)}^2}(Y - \mu(X))  
    \end{align*}
    
\end{proof}

Next, we move on to discussing our estimator of the upper bound, using this influence function. 

\begin{proposition} \label{prop:estimator_ub_gamma}
    Our one-step estimator of $\theta_3^{u, \gamma}$ is given by
    \begin{align*}
        \hat{\theta}_3^{u, \gamma} 
        & =  \frac{1}{\hat{P}(Y=1, D=0)} E_P\left[ \hat{A}(X)(1[D=0]) \right]\\
        & \qquad + \frac{1}{\hat{P}(Y=1, D=0)} E_P\left[1[D=1] \frac{\gamma \hat{\mu}(X)}{((\gamma - 1)\hat{\mu}(X) + 1)}(T - \hat{\pi}(x)) \frac{\hat{g}(X)}{1 - \hat{g}(X)} \right] \\
        & \qquad +\frac{1}{\hat{P}(Y=1, D=0)} E_P\left[ 1[D=0]\frac{\gamma \hat{\pi}(X)}{{((\gamma - 1)\hat{\mu}(X) + 1)}^2}(Y - \hat{\mu}(x))  \right] 
    \end{align*}
\end{proposition}

\begin{proof}

    \begin{align*}
        \hat{\theta}_3^{u, \gamma}(\hat{P}) & =\theta_3^{u, \gamma}(\hat{P}) + E_{\hat{P}}[IF( \theta_3^{u, \gamma}(\hat{P}))]
    \end{align*}

    The first term is given by
    \begin{align*}
        \theta_3^{u, \gamma}(\hat{P}) & = \frac{1}{\hat{P}(Y=1, D=0)} E_{\hat{P}}[\hat{g}(X) \hat{A}(X)] 
    \end{align*}

    The second term is given by
    \begin{align*}
        E_{\hat{P}}[IF( \theta_3^{u, \gamma}(\hat{P}))] & = E_{\hat{P}} \left[ -\frac{1[Y=1, D=0]}{\hat{P}{(Y=1, D=0)}^2} E_{\hat{P}}[\hat{g}(X)\hat{A}(X)]  + \frac{\hat{g}(X)\hat{A}(X)}{\hat{P}(Y=1, D=0)} + \frac{\hat{A}(X)(1[D=0]-\hat{g}(X))}{\hat{P}(Y=1, D=0)} \right] \\
         & \qquad + E_{\hat{P}}\left[ \frac{1[D=1]}{\hat{P}(Y=1, D=0)} \frac{\gamma \hat{\mu}(X)}{((\gamma - 1)\hat{\mu}(X) + 1)}(T - \hat{\pi}(x))\frac{\hat{g}(X)}{1 - \hat{g}(X)}\right] \\
         &  \qquad + E_{\hat{P}}\left[\frac{1[D=0]}{\hat{P}(Y=1, D=0)}\frac{\gamma \hat{\pi}(X)}{{((\gamma - 1)\hat{\mu}(X) + 1)}^2}(Y - \hat{\mu}(x))  \right] 
    \end{align*}
    The first expectation term in the second term is exactly $\theta_3^{u, \gamma}(\hat{P})$, so it cancels out with the original first term. This gives us that
    \begin{align*}
        \hat{\theta_3}^{u, \gamma}(\hat{P}) & = \frac{1}{\hat{P}(Y=1, D=0)} E_{\hat{P}}[\hat{g}(X) \hat{A}(X)] + \frac{1}{\hat{P}(Y=1, D=0)} E_{\hat{P}}\left[ \hat{A}(X)(1[D=0] - \hat{g}(X)) \right]\\
        & \qquad + \frac{1}{\hat{P}(Y=1, D=0)} E_{\hat{P}}\left[1[D=1] \frac{\gamma \hat{\mu}(X)}{((\gamma - 1)\hat{\mu}(X) + 1)}(T - \hat{\pi}(X)) \frac{\hat{g}(X)}{1 - \hat{g}(X)}\right] \\
        & \qquad +\frac{1}{\hat{P}(Y=1, D=0)} E_{\hat{P}}\left[ 1[D=0] \frac{\gamma \hat{\pi}(X)}{{((\gamma - 1)\hat{\mu}(X) + 1)}^2}(Y - \hat{\mu}(X))  \right] \\
        & = \frac{1}{\hat{P}(Y=1, D=0)} E_{\hat{P}}[\hat{g}(X) \hat{A}(X)] + \frac{1}{\hat{P}(Y=1, D=0)} E_{\hat{P}}\left[ \hat{A}(X)(1[D=0] - \hat{g}(X)) \right]\\
        & \qquad + \frac{1}{\hat{P}(Y=1, D=0)} E_{\hat{P}}\left[1[D=1] \frac{\gamma \hat{\mu}(X)}{((\gamma - 1)\hat{\mu}(X) + 1)}(T - \hat{\pi}(X)) \frac{\hat{g}(X)}{1 - \hat{g}(X)}\right] \\
        & \qquad +\frac{1}{\hat{P}(Y=1, D=0)} E_{\hat{P}}\left[ 1[D=0] \frac{\gamma \hat{\pi}(X)}{{((\gamma - 1)\hat{\mu}(X) + 1)}^2}(Y - \hat{\mu}(x))  \right] \\
        & =  \frac{1}{\hat{P}(Y=1, D=0)} E_P\left[ \hat{A}(X)(1[D=0]) \right]\\
        & \qquad + \frac{1}{\hat{P}(Y=1, D=0)} E_P\left[1[D=1] \frac{\gamma \hat{\mu}(X)}{((\gamma - 1)\hat{\mu}(X) + 1)}(T - \hat{\pi}(x)) \frac{\hat{g}(X)}{1 - \hat{g}(X)} \right] \\
        & \qquad +\frac{1}{\hat{P}(Y=1, D=0)} E_P\left[ 1[D=0]\frac{\gamma \hat{\pi}(X)}{{((\gamma - 1)\hat{\mu}(X) + 1)}^2}(Y - \hat{\mu}(x))  \right].
    \end{align*}
\end{proof}

\begin{lemma}[Error of one-step estimator of upper bound with $\gamma$]\label{lemma:error_ub_gamma}
        Let the error of our one-step estimator be given by 
        \begin{equation*}
          R(\hat{P}, P) = \theta_3^{u, \gamma}(\hat{P}) - \theta_3^{u,\gamma}(P) + E_P\left[ IF(\theta_3^{u, \gamma}(\hat{P})) \right]
        \end{equation*}
        Then, we have that
        \begin{align*}
            R(\hat{P}, P) = o_P(n^{-\frac{1}{2}}),
        \end{align*}
        when $(\hat{\pi}  - \pi), (\hat{\mu}  - \mu), (\hat{g}  - g)$ have rates of at least $o_P(n^{-\frac{1}{4}})$. 
\end{lemma}

\begin{proof}
    \begin{align*}
        R(\hat{P}, P) & = \theta_3^{u, \gamma}(\hat{P})-\theta_3^{u, \gamma}(P)+E_P\left[I F\left(\theta_3^{u, \gamma}(\hat{P})\right)\right] \\
        & = \underbrace{E_{\hat{P}}\left[ \frac{\hat{g}(X)\hat{A}(X)}{\hat{P}(Y=1,D=0)} \right]}_{\textcolor{blue}{(a)}} - \underbrace{E_{P} \left[ \frac{g(X)A(X)}{P(Y=1,D=0)} \right]}_{\textcolor{blue}{(b)}} - \underbrace{\frac{P(Y=1,D=0)}{(\hat{P}(Y=1,D=0))^{2}}E_{\hat{P}}[\hat{g}(X)\hat{A}(X)]}_{\textcolor{blue}{(c)}} \\
        & + \underbrace{ E_{P} \left[ \frac{\hat{g}(X)\hat{A}(X)}{\hat{P}(Y=1,D=0)} \right]}_{\textcolor{blue}{(d)}} + E_{P} \left[ \frac{\hat{A}(X)(1[D=0]-\hat{g}(X)}{\hat{P}(Y=1,D=0)}\right] \\
        &+ E_{P} \left[ \frac{1[D=1]}{\hat{P}(Y=1,D=0)} \left( \frac{\gamma \hat{\mu}(X)}{(\gamma-1)\hat{\mu}(X) + 1} \right) \frac{\hat{g}(X)}{1-\hat{g}(X)} (T - \hat{\pi}(X)) \right] \\
        &+ E_{P} \left[ \frac{1[D=0]}{\hat{P}(Y=1,D=0)} \left( \frac{\gamma \hat{\pi}(X)}{((\gamma-1)\hat{\pi}(X) + 1)^2} \right) (Y - \hat{\mu}(X)) \right]
    \end{align*}
    First, we take the terms \textcolor{blue}{(a)} and \textcolor{blue}{(c)}, 
    \small
    \begin{align*}
        & E_{\hat{P}}\left[ \frac{\hat{g}(X)\hat{A}(X)}{\hat{P}(Y=1,D=0)} \right] - \frac{P(Y=1,D=0)}{(\hat{P}(Y=1,D=0))^{2}}E_{\hat{P}}[\hat{g}(X)\hat{A}(X)] = \left( 1 - \frac{P(Y=1,D=0)}{\hat{P}(Y=1,D=0)} \right) E_{\hat{P}} \left[ \frac{\hat{g}(X)\hat{A}(X)}{\hat{P}(Y=1,D=0)} \right] 
    \end{align*}
    \normalsize
    As shown in the proof of \cref{lemma:error_ub} in \cref{eq:convergence_p_phat}, this converges at $o_{P}(n^{-1/2})$ rate. 
    Now, we take the terms \textcolor{blue}{(b)} and \textcolor{blue}{(d)},
    \begin{align*}
        & - E_{P} \left[ \frac{g(X)A(X)}{P(Y=1,D=0)} \right] + E_{P} \left[ \frac{\hat{g}(X)\hat{A}(X)}{\hat{P}(Y=1,D=0)} \right] \\
        & = \left( 1 - \frac{P(Y=1,D=0)}{\hat{P}(Y=1,D=0)} \right) E_{P} \left[ \frac{g(X)A(X)}{\hat{P}(Y=1,D=0)} \right] + \underbrace{\frac{1}{\hat{P}(Y=1,D=0)}E_{P}[\hat{g}(X)\hat{A}(X) - g(X)A(X)]}_{\textcolor{blue}{(e)}}
    \end{align*}
    Again, we see that the first term converges at $o_{P}(n^{-1/2})$ rate. We defer analysis of the second term to later. 
    We will first turn to analyzing the remaining terms.
    \begin{align*}
        & E_{P} \left[ \frac{1[D=1]}{\hat{P}(Y=1,D=0)} \left( \frac{\gamma \hat{\mu}(X)}{(\gamma-1)\hat{\mu}(X) + 1} \right) \frac{\hat{g}(X)}{1-\hat{g}(X)} (T - \hat{\pi}(X)) \right] \\
        & + E_{P} \left[ \frac{1[D=0]}{\hat{P}(Y=1,D=0)} \left( \frac{\gamma \hat{\pi}(X)}{(((\gamma-1)\hat{\mu}(X) + 1)^{2}} \right) (Y - \hat{\mu}(X)) \right] +  E_{P} \left[ \frac{\hat{A}(X)(1[D=0]-\hat{g}(X))}{\hat{P}(Y=1,D=0)}\right] \\ 
        & = E_{P} \left[ \frac{1-g(X)}{\hat{P}(Y=1,D=0)} \left( \frac{\gamma \hat{\mu}(X)}{(\gamma-1)\hat{\mu}(X) + 1} \right) \frac{\hat{g}(X)}{1-\hat{g}(X)} (\pi(X) - \hat{\pi}(X)) \right] \\
        & + E_{P} \left[ \frac{g(X)}{\hat{P}(Y=1,D=0)} \left( \frac{\gamma \hat{\pi}(X)}{((\gamma-1)\hat{\mu}(X) + 1)^{2}} \right) (\mu(X) - \hat{\mu}(X)) \right] +  E_{P} \left[ \frac{\hat{A}(X)(g(X)-\hat{g}(X))}{\hat{P}(Y=1,D=0)}\right] \\
    \end{align*}
    The equality here holds through an application of iterated expectation over $X$ and with the observation $$E[1[D=1] T q(X)] = E[(1 - g(X))\pi(X) q(X)],$$
    for any function $q$ of $X$. 
    
    We now look at these remaining terms with the last remaining term from above \textcolor{blue}{(e)}, resulting in the following expression:
    \begin{align*}
        R(\hat{P},P) & = \frac{1}{\hat{P}(Y=1,D=0)}E_{P}[\hat{g}(X)\hat{A}(X) - g(X)A(X)] + E_{P} \left[ \frac{\hat{A}(X)(g(X)-\hat{g}(X))}{\hat{P}(Y=1,D=0)}\right] \\
        & + E_{P} \left[ \frac{1-g(X)}{\hat{P}(Y=1,D=0)} \left( \frac{\gamma \hat{\mu}(X)}{(\gamma-1)\hat{\mu}(X) + 1} \right) \frac{\hat{g}(X)}{1-\hat{g}(X)} (\pi(X) - \hat{\pi}(X)) \right] \\
        & + E_{P} \left[ \frac{g(X)}{\hat{P}(Y=1,D=0)} \left( \frac{\gamma \hat{\pi}(X)}{((\gamma-1)\hat{\mu}(X) + 1)^{2}} \right) (\mu(X) - \hat{\mu}(X)) \right]  \\
        & + o_P(n^{-\frac{1}{2}}),
    \end{align*}
    where the terms that disappear at the parametric rate are contained within the additional term of $o_P(n^{-\frac{1}{2}})$. This further simplifies to
    \begin{align*}
        R(\hat{P},P) & = \frac{1}{\hat{P}(Y=1,D=0)}E_{P}[g(X)\hat{A}(X) - g(X)A(X)] \\
        & + E_{P} \left[ \frac{1-g(X)}{\hat{P}(Y=1,D=0)} \left( \frac{\gamma \hat{\mu}(X)}{(\gamma-1)\hat{\mu}(X) + 1} \right) \frac{\hat{g}(X)}{1-\hat{g}(X)} (\pi(X) - \hat{\pi}(X)) \right] \\
        & + E_{P} \left[ \frac{g(X)}{\hat{P}(Y=1,D=0)} \left( \frac{\gamma \hat{\pi}(X)}{((\gamma-1)\hat{\mu}(X) + 1)^{2}} \right) (\mu(X) - \hat{\mu}(X)) \right] + o_P(n^{-\frac{1}{2}})
    \end{align*}
    Rewriting the above equation in terms of $A(X)$ and $\hat{A}(X)$, we have that
    \begin{align*}
        R(\hat{P},P) & = \underbrace{\frac{1}{\hat{P}(Y=1,D=0)}E_{P}[g(X)\hat{A}(X) - g(X)A(X)]}_{\textcolor{red}{(j)}} \\
        & + E_{P} \left[ \underbrace{\frac{1-g(X)}{\hat{P}(Y=1,D=0)} \left( \frac{\gamma \hat{\mu}(X) \pi(X)}{(\gamma-1)\hat{\mu}(X) + 1} \right) \frac{\hat{g}(X)}{1-\hat{g}(X)}}_{\textcolor{red}{(e)}} - \underbrace{\frac{(1 - g(X))\hat{g}(X)}{\hat{P}(Y=1, D=0)(1 - \hat{g}(X))} \hat{A}(X)}_{\textcolor{red}{(f)}} \right] \\
        & + E_{P} \left[ \frac{g(X)}{\hat{P}(Y=1,D=0)} \left( \underbrace{\frac{\gamma \hat{\pi}(X)\mu(X)}{((\gamma-1)\hat{\mu}(X) + 1)^{2}}}_{\textcolor{red}{(g)}} - \underbrace{\frac{1}{(\gamma - 1)\hat{\mu}(X) + 1}\hat{A}(X)}_{\textcolor{red}{(h)}} \right) \right] \\
        & + o_P(n^{-\frac{1}{2}})
    \end{align*}

    Now, we let $d = (\gamma - 1) \hat{\mu}(X) + 1$ and $d' = (\gamma - 1) \mu(X) + 1$ (i.e., the denominators of $A$ and $\hat{A}$, respectively). We first look at combining terms \textcolor{red}{(e)} and \textcolor{red}{(h)}
    \begin{align*}
        & \frac{1}{\hat{P}(Y=1, D=0)} \Big( E\left[ \frac{1-g(X)}{1-\hat{g}(X)} \hat{g}(X) \frac{\gamma \hat{\mu}(X) \pi(X)}{d} \right] - E\left[ \frac{g}{d} \hat{A}(X) \right] \Big) \\
        & \qquad \qquad = \frac{1}{\hat{P}(Y=1, D=0)} E\Big[ \frac{(1-g(X)) \hat{g}(X) \gamma \hat{\mu}(X) \pi(X)}{(1 - \hat{g}(X)) d}  - \frac{g(X) \gamma \hat{\mu}(X) \hat{\pi}(X)}{d d} \Big] 
    \end{align*}

    This simplifies to give us that
    \begin{align*}
        & = \frac{1}{\hat{P}(Y=1, D=0)} E\Big[ \frac{(1 - g(X))\hat{g}(X) (\gamma \hat{\mu}(X) \pi(X))((\gamma - 1) \hat{\mu}(X) + 1) - (1 - \hat{g}(X))g(X) \gamma \hat{\mu}(X) \hat{\pi}(X)}{(1 - \hat{g}(X))d d'}  \Big]\\
        & =  \frac{1}{\hat{P}(Y=1, D=0)} E\Big[ \frac{1}{(1 - \hat{g}(X))d d'} \Big( (1 - g(X))\hat{g}(X) (\gamma \hat{\mu}(X) \pi(X))(\gamma - 1) \hat{\mu}(X)  \\
        & \qquad \qquad \qquad \qquad + (1 - g(X))\hat{g}(X) \gamma \hat{\mu}(X) \pi(X) \\
        & \qquad \qquad \qquad \qquad  - (1 - \hat{g}(X))g(X) \gamma \hat{\mu}(X) \hat{\pi}(X)  \Big)  \Big]\\
        & =  \frac{1}{\hat{P}(Y=1, D=0)} E\Big[ \frac{1}{(1 - \hat{g}(X))d d'} \Big( (1 - g(X))\hat{g}(X) (\gamma \hat{\mu}(X) \pi(X))(\gamma - 1) \hat{\mu}(X)  \\
        & \qquad \qquad \qquad \qquad  - \hat{g}(X) g(X) \gamma \hat{\mu}(X) \hat{\pi}(X) + \hat{g}(X) g(X) \gamma \hat{\mu}(X) \hat{\pi}(X) \\
        & \qquad \qquad \qquad \qquad + g(X) \hat{g}(X) \gamma \hat{\mu}(X) \pi(X) - \hat{g}(X) \gamma \hat{\mu}(X) \pi(X)  \Big)  \Big] \\
        & = \frac{1}{\hat{P}(Y=1, D=0)} E\Big[ \frac{1}{(1 - \hat{g}(X))d d'} \Big( (1 - g(X))\hat{g}(X) (\gamma \hat{\mu}(X) \pi(X))(\gamma - 1) \hat{\mu}(X)  \\
        & \qquad \qquad \qquad \qquad  + \underbrace{\gamma \hat{\mu}(X) (\hat{g}(X) \pi(X) - g(X) \hat{\pi}(X)) + g(X) \hat{g}(X) \gamma \hat{\mu}(X)(\hat{\pi}(X) - \pi(X)) }_{\textcolor{red}{(i)}}\Big)  \Big]  
    \end{align*}
    We remark that we can simplify \textcolor{red}{(i)} (ignoring the common multiple of $\gamma \hat{\mu}(X)$ for now), 
    \begin{align*}
        &\hat{g}(X)\pi(X) - g(X) \hat{\pi}(X) + g(X) \hat{g}(X) \hat{\pi}(X) - g(X) \hat{g}(X) \pi(X) \\
        & \qquad \qquad = g(X) \hat{\pi}(X) (\hat{g}(X) - 1) - \hat{g}(X) \pi(X) (g(X) - 1)
    \end{align*}
    Plugging this in yields that
    \begin{align}
        & = \frac{1}{\hat{P}(Y=1, D=0)} E\Big[ \frac{1}{(1 - \hat{g}(X))d d'} \Big( (1 - g(X))\hat{g}(X) (\gamma \hat{\mu}(X) \pi(X))(\gamma - 1) \hat{\mu}(X) \notag \\
        & \qquad \qquad \qquad \qquad  +  \gamma \hat{\mu}(X) \Big(
            g(X) \hat{\pi}(X) (\hat{g}(X) - 1) - \hat{g}(X) \pi(X) (g(X) - 1)   \Big) \Big)  \Big]  \notag \\
        & =  \frac{1}{\hat{P}(Y=1, D=0)} E\Big[ \frac{1}{(1 - \hat{g}(X))d d'} (1 - g(X))\hat{g}(X) (\gamma \hat{\mu}(X) \pi(X))(\gamma - 1) \hat{\mu}(X) \Big]  \label{eq:part_1_gamma} \\
        & \qquad \qquad \qquad \qquad + \frac{1}{\hat{P}(Y=1, D=0)} E\Big[  \gamma \hat{\mu}(X) \Big( g(X) \hat{\pi}(X) (\hat{g}(X) - 1) - \hat{g}(X) \pi(X) (g(X) - 1)  \Big) \Big] \notag
    \end{align}

    Next, we look at combining terms \textcolor{red}{(f)} and \textcolor{red}{(g)}
    \begin{align*}
        & \frac{1}{\hat{P}(Y=1, D=0)} \left( E \left[ \frac{g(X)\gamma \hat{\pi}(X)\mu(X)}{dd} \right] - E \left[ \frac{1-g(X)}{1-\hat{g}(X)}\hat{g}(X)\hat{A}(X) \right] \right) \\
        &= \frac{1}{\hat{P}(Y=1, D=0)} E \left(\frac{(1-\hat{g}(X)) g(X)}{(1-\hat{g}(X))d} \frac{\gamma \hat{\pi}(X)\mu(X)}{d} - \frac{1-g(X)}{1-\hat{g}(X)}\hat{g}(X)\frac{\gamma \hat{\pi}(X) \hat{\mu}(X)d}{dd} \right) \\
    \end{align*}

    We first try to simplify the numerator inside the expectation. First, we substitute $d = (\gamma-1)\hat{\mu}(X)+1$ and expand the terms.
    \begin{align*}
        & (1-\hat{g}(X))g(X)\gamma \hat{\pi}(X)\mu(X) - (1-g(X))\hat{g}(X)\gamma \hat{\pi}(X) \hat{\mu}(X)d \\
        &= (1-\hat{g}(X))g(X)\gamma \hat{\pi}(X)\mu(X) - (1-g(X))\hat{g}(X)\gamma \hat{\pi}(X) \hat{\mu}(X)((\gamma-1)\hat{\mu}(X)+1) \\
        &= (1-\hat{g}(X))g(X)\gamma \hat{\pi}(X)\mu(X) - (1-g(X))\hat{g}(X)\gamma \hat{\pi}(X) \hat{\mu}(X)(\gamma-1)\hat{\mu}(X) - (1-g(X))\hat{g}(X)\gamma \hat{\pi}(X) \hat{\mu}(X) \\
        &= \underbrace{g(X)\gamma \hat{\pi}(X)\mu(X)}_{\textcolor{green}{(k)}} - \underbrace{\hat{g}(X)g(X)\gamma \hat{\pi}(X)\mu(X)}_{\textcolor{green}{(l)}} - \underbrace{(1-g(X))\hat{g}(X)\gamma \hat{\pi}(X) \hat{\mu}(X)(\gamma-1)\hat{\mu}(X)}_{\textcolor{green}{(m)}} \\
        & - \underbrace{\hat{g}(X)\gamma \hat{\pi}(X) \hat{\mu}(X)}_{\textcolor{green}{(n)}} + \underbrace{g(X)\hat{g}(X)\gamma \hat{\pi}(X) \hat{\mu}(X)}_{\textcolor{green}{(o)}}
    \end{align*}
    Grouping the \textcolor{green}{(k)} and \textcolor{green}{(n)} together, grouping \textcolor{green}{(l)} and \textcolor{green}{(o)} together, and keeping term \textcolor{green}{(m)}, we have 
    \begin{align*}
        & \gamma \hat{\pi}(X)(g(X)\mu(X) - \hat{g}(X)\hat{\mu}(X)) + \gamma \hat{\pi}(X)g(X)\hat{g}(X)(\hat{\mu}(X)-\mu(X)) - (1-g(X))\hat{g}(X)\gamma \hat{\pi}(X) \hat{\mu}(X)(\gamma-1)\hat{\mu}(X) \\
        & = \left( \gamma \hat{\pi} (g(X)\mu(X) - \hat{g}(X)\hat{\mu}(X) + g(X)\hat{g}(X)\hat{\mu}(X) - g(X)\hat{g}(X)\mu(X) \right) \\
        & \qquad - (1-g(X))\hat{g}(X)\gamma \hat{\pi}(X) \hat{\mu}(X)(\gamma-1)\hat{\mu}(X) \\
        & = (\gamma \hat{\pi}(\hat{g}(X)\hat{\mu}(X)(g(X)-1) - g(X)\mu(X)(\hat{g}(X)-1)) - (1-g(X))\hat{g}(X)\gamma \hat{\pi}(X) \hat{\mu}(X)(\gamma-1)\hat{\mu}(X) 
    \end{align*}

    Reintroducing the denominator results in the following expression
    \begin{align}
        & \frac{1}{\hat{P}(Y=1, D=0)} E \Big[ \frac{\gamma \hat{\pi}(X)}{(1-\hat{g}(X))dd}(\hat{g}(X)\hat{\mu}(X)(g(X)-1) - g(X)\mu(X)(\hat{g}(X)-1))  \label{eq:part_2_gamma} \\
        & \qquad \qquad \qquad \qquad  - \frac{(1-g(X))\hat{g}(X)\gamma \hat{\pi}(X) \hat{\mu}(X)(\gamma-1)\hat{\mu}(X)}{(1-\hat{g}(X))dd} \Big] \notag
    \end{align}

    Next, we combine terms from \cref{eq:part_1_gamma} and \cref{eq:part_2_gamma}, which gives us that (ignoring the $\frac{1}{\hat{P}(Y=1, D=0)}$ and the expectation for now)
    \begin{align}
        & \frac{1}{(1 - \hat{g}(X)) dd} ((1 - g(X)) \hat{g}(X) \gamma \hat{\mu}(X) \pi(X) (\gamma - 1) \hat{\mu}(X) \label{green_1} \\
        & \qquad + \frac{\gamma \hat{\mu}(X)}{(1 - \hat{g}(X)) dd} (g(X) \hat{\pi}(X)(\hat{g}(X) - 1) - \hat{g}(X) \pi(X) (g(X) - 1)) \label{blue_1} \\
        & \qquad + \frac{\gamma \hat{\pi}(X)}{(1 - \hat{g}(X)) dd} (\hat{g}(X) \hat{\mu}(X)(g(X) - 1) - g(X) \mu(X) (\hat{g}(X) - 1)) \label{blue_2} \\
        & \qquad - \frac{1}{(1 - \hat{g}(X)) dd} ((1 - g(X)) \hat{g}(X) \gamma \hat{\mu}(X) \hat{\pi}(X) (\gamma - 1) \hat{\mu}(X) \label{green_2}
    \end{align}

    Combining terms \cref{green_1} and \cref{green_2} gives us
    \begin{align*}
        \frac{1}{(1 - \hat{g}(X)) dd} ((1 - g(X)) \hat{g}(X) \gamma \hat{\mu}(X)(\gamma - 1) \hat{\mu}(X) (\pi(X) - \hat{\pi}(X)).    
    \end{align*}

    Combining terms \cref{blue_1} and \cref{blue_2} gives us
    \begin{align*}
        & \frac{1}{(1 - \hat{g}(X)) dd} (\gamma g(X) \hat{\pi}(X) (\hat{g}(X) - 1)) (\hat{\mu}(X) - \mu(X)) + \frac{1}{(1 - \hat{g}(X)) dd} \gamma \hat{\mu}(X)\hat{g}(X) (g(X) - 1) (\hat{\pi}(X) - \pi(X))
    \end{align*}

    The remaining term \textcolor{red}{(j)} from above is simplified as
    \small
    \begin{align}
        \frac{1}{\hat{P}(Y=1, D=0)} E[g(X) \hat{A}(X) - g(X) A(X)] & = \frac{1}{\hat{P}(Y=1, D=0)} E\left[g(X) \frac{1}{d d'} \gamma(\gamma - 1) \mu(X) \hat{\mu}(X) (\hat{\pi}(X) - \pi(X)) \right] \label{first_term_1}\\ 
        & \qquad + {\hat{P}(Y=1, D=0)} E\left[g(X) \frac{1}{d d'} \gamma  (\hat{\pi}(X) \hat{\mu}(X) - \pi(X) \mu(X)) \right],\label{first_term_2} 
    \end{align}
    \normalsize
    as we note that
    \begin{align*}
        \hat{A}(X) - A(X) & = \frac{\gamma \hat{\pi}(X) \hat{\mu}(X)}{d} - \frac{\gamma \pi(X) \mu(X)}{d'} = \frac{d' \gamma \hat{\pi}(X) \hat{\mu}(X) - d \gamma \pi(X) \mu(X)}{dd'} \\
        & = \frac{1}{dd'} \Big( (\gamma - 1) \mu(X) \gamma \hat{\pi}(X) \hat{\mu}(X) - (\gamma - 1) \hat{\mu}(X) \gamma \pi(X) \mu(X) + \gamma \hat{\pi}(X) \hat{\mu}(X) - \gamma \pi(X) \mu(X) \Big) \\
        & = \frac{1}{dd'} \gamma (\gamma - 1) \mu(X) \hat{\mu}(X) (\hat{\pi}(X) - \pi(X)) + \frac{1}{dd'} \gamma (\hat{\pi}(X) \hat{\mu}(X) - \pi(X) \mu(X))
    \end{align*}

    Now, we can combine \cref{first_term_1} and the combination of \cref{green_1} and \cref{green_2}, giving us that
    \begin{align*}
        & \frac{1}{\hat{P}(Y=1, D=0)} E\left[ g(X) \frac{1}{dd'} \gamma (\gamma - 1) \mu(X) \hat{\mu}(X) (\hat{\pi}(X) - \pi(X)) \right] \\
        & \qquad + \frac{1}{\hat{P}(Y=1, D=0)} E\left[ \frac{1}{(1 - \hat{g}(X))d d} (1 - g(X))\hat{g}(X) \gamma \hat{\mu}(X) (\gamma - 1) \hat{\mu}(X) (\pi(X) - \hat{\pi}(X)) \right]
    \end{align*}
    which simplifies by factoring to give that
    \begin{align*}
        \frac{1}{\hat{P}(Y=1, D=0)} E\left[ \frac{\gamma(\gamma - 1) \hat{\mu}(X)(\hat{\pi}(X) - \pi(X))}{d}  \Big( \frac{g(X)}{d'} \mu(X) - \frac{(1 - g(X)) \hat{g}(X) \hat{\mu}(X)}{(1 - \hat{g}(X)) d} \Big) \right]
    \end{align*}
    Now we can focus on the difference term, which simplifies as
    \begin{align*}
        \frac{g(X) (1 - \hat{g}(X)) d\mu(X)}{(1 - \hat{g}(X)) dd'} - \frac{(1 - g(X))\hat{g}(X) d' \hat{\mu}(X)}{(1 - \hat{g}(X)) d d'},
    \end{align*}
    and when only simplifying the numerator, we get that
    \begin{align*}
        & (g(X) - g(X)\hat{g}(X)) d\mu(X) - (\hat{g}(X) - g(X) \hat{g}(X)) d' \hat{\mu}(X) \\
        & \qquad  =  (g(X) - g(X)\hat{g}(X)) ((\gamma - 1) \hat{\mu}(X) + 1) \mu(X) - (\hat{g}(X) - g(X) \hat{g}(X)) ((\gamma - 1) \mu(X) + 1) \hat{\mu}(X) \\
        & \qquad =  (g(X) - g(X)\hat{g}(X))(\gamma - 1) \hat{\mu}(X)\mu(X) \\ 
        & \qquad \qquad + (g(X) - g(X)\hat{g}(X))\mu(X) - (\hat{g}(X) - g(X) \hat{g}(X)) (\gamma - 1) \mu(X) \hat{\mu}(X) - (\hat{g}(X) - g(X) \hat{g}(X)) \hat{\mu}(X)
    \end{align*}
    This further simplifies to give us that
    \begin{align*}
        & = (\gamma - 1) \hat{\mu}(X) \mu(X) (g(X) - g(X) \hat{g}(X) - \hat{g}(X) + g(X)\hat{g}(X)) + (g(X) - g(X) \hat{g}(X)) \mu(X) - (\hat{g}(X) - g(X) \hat{g}(X)) \hat{\mu}(X) \\
        & = (\gamma - 1) \hat{\mu}(X) \mu(X) (g(X) - \hat{g}(X)) + g(X) \mu(X) - \hat{g}(X) \hat{\mu}(X) + g(X) \hat{g}(X) \hat{\mu}(X) - g(X)\hat{g}(X) \mu(X) \\
        & = (\gamma - 1) \hat{\mu}(X) \mu(X) \Big(g(X) - \hat{g}(X) \Big) + \Big(g(X) \mu(X) - \hat{g}(X) \hat{\mu}(X)\Big) + g(X) \hat{g}(X) \Big(\hat{\mu}(X) - \mu(X) \Big)
    \end{align*}

    Then, plugging this back with the denominator and factored out term in the numerator gives us that
    \begin{align*}
        & = \frac{1}{\hat{P}(Y=1, D=0)} E\left[  \frac{\gamma(\gamma - 1) \hat{\mu}(X)}{d} (\hat{\pi}(X) - \pi(X)) \cdot  (\gamma - 1) \hat{\mu}(X) \mu(X) \Big(g(X) - \hat{g}(X) \Big) \right] \\
        & \qquad + \frac{1}{\hat{P}(Y=1, D=0)} E\left[  \frac{\gamma(\gamma - 1) \hat{\mu}(X)}{d} (\hat{\pi}(X) - \pi(X)) \cdot \Big(g(X) \mu(X) - \hat{g}(X) \hat{\mu}(X)\Big) \right] \\
         & \qquad + \frac{1}{\hat{P}(Y=1, D=0)} E\left[  \frac{\gamma(\gamma - 1) \hat{\mu}(X)}{d} (\hat{\pi}(X) - \pi(X)) \cdot   g(X) \hat{g}(X) \Big(\hat{\mu}(X) - \mu(X) \Big) \right]
    \end{align*}

    Note that in each line, we have squared terms in differences of our estimated quantities on $\hat{P}$ and $P$. In the second line, we have $(\hat{\pi} - \pi) \cdot (g\mu - \hat{g}\hat{\mu})$; the term $(\hat{g}\hat{\mu}$ is a plugin estimator, which we have previously shown in \cref{sec:plugin} to have a rate of the sum of the rates of $\hat{g}$ and $\hat{\mu}$. Thus, when multiplied by the difference $\hat{\pi} - \pi$, we still achieve squared terms. Thus, this term converges at a rate of $o_P(n^{-\frac{1}{2} })$ if our estimates of $\pi$ and $g$ and $\mu$, each converge at a rate of $o_P(n^{-\frac{1}{4} })$.

    Note that $\hat{\pi}(X)\hat{\mu}(X) - \pi(X) \mu(X) = \hat{\pi}(X)\hat{\mu}(X) - \pi(X) \mu(X) + \hat{\pi}(X)\mu(X) - \hat{\pi}(X)\mu(X) = \hat{\pi}(X)(\hat{\mu}(X)-\mu(X)) + (\hat{\pi}(X) - \pi(X))\mu(X)$. Now looking at \cref{first_term_2}, we can write it as
    \begin{align*}
        & \frac{1}{\hat{P}(Y=1,D=0)} E \left[ \frac{g(X)}{dd'}\gamma (\hat{\pi}(X)\hat{\mu}(X) - \pi(X) \mu(X)) \right] \\
        & = \underbrace{\frac{1}{\hat{P}(Y=1,D=0)} E \left[ \frac{g(X)}{dd'} \gamma(\hat{\pi}(X)(\hat{\mu}(X)-\mu(X)) \right]}_{\textcolor{purple}{(p)}} + \underbrace{\frac{1}{\hat{P}(Y=1,D=0)} E \left[ \frac{g(X)}{dd'} \gamma(\mu(X)(\hat{\pi}(X)-\pi(X)) \right]}_{\textcolor{purple}{(q)}{}}
    \end{align*}

    First, we look at the two terms with $(\hat{\mu}(X)-\mu(X))$ (which are \textcolor{purple}{(p)} and \cref{blue_2}).
    \small
    \begin{align*}
        & \frac{1}{\hat{P}(Y=1,D=0)} E \left[ \frac{g(X)}{dd'} \gamma \hat{\pi}(X)(\hat{\mu}(X)-\mu(X)) \right] - \frac{1}{\hat{P}(Y=1,D=0)} E \left[ \frac{1-\hat{g}(X)}{(1-\hat{g}(X))dd} \gamma \hat{\pi}(X)g(X)(\hat{\mu}(X) - \mu(X)) \right] \\
        & = \frac{1}{\hat{P}(Y=1,D=0)} E \left[ g \gamma \hat{\pi}(X)(\hat{\mu}(X) - \mu(X)) \left( \frac{1}{dd'} - \frac{1}{dd} \right) \right]
    \end{align*}
    \normalsize
    Looking at $\frac{1}{dd'} - \frac{1}{dd}$, we unify the denominator as follows 
    \begin{align*}
        \frac{1}{dd'} - \frac{1}{dd} &= \frac{1}{(\gamma-1)\mu(X) + 1} - \frac{1}{(\gamma-1)\hat{\mu}(X) + 1} = \frac{(\gamma-1)\hat{\mu}(X) + 1 - ((\gamma-1)\mu(X) + 1)}{d'd} \\
        &= \frac{(\gamma-1)(\hat{\mu}(X)-\mu(X))}{dd'}
    \end{align*}
    Putting this back together, the complete $\hat{\mu}-\mu$ term is
    \begin{align*}
        \frac{1}{\hat{P}(Y=1,D=0)} E \left[ g \gamma \hat{\pi}(X)(\hat{\mu}(X) - \mu(X)) \frac{(\gamma-1)(\hat{\mu}(X)-\mu(X))}{dd'} \right]
    \end{align*}

    We observe that this is in the form of squared differences of $(\hat{\mu}(X) - \mu(X))$. Thus, this term converges at a rate of $o_P(n^{-\frac{1}{2}})$ if our estimate of $\hat{\mu}(X)$ converges at a rate of $o_P(n^{-\frac{1}{4}})$.

    Now, looking at the $(\hat{\pi}(X)-\pi(X))$ terms (which are \textcolor{purple}{(q)} and \cref{blue_1}), we have
    \begin{align*}
        & \frac{1}{\hat{P}(Y=1,D=0)} E \left[ g\frac{1}{dd'}\gamma \mu(X)(\hat{\pi}(X) - \pi(X)) \right] - \frac{1}{\hat{P}(Y=1,D=0)} E \left[ \hat{g}(X)\frac{1-g(X)}{1-\hat{g}(X)}\frac{1}{dd}\gamma \hat{\mu}(X)(\hat{\pi}(X) - \pi(X)) \right] \\
        & = \frac{1}{\hat{P}(Y=1,D=0)} E \left[ \frac{\gamma}{d}(\hat{\pi}(X)-\pi(X)) \left( \frac{g(X)}{d'}\mu(X) - \frac{\hat{g}(X)}{d}\frac{1-g(X)}{1-\hat{g}(X)} \hat{\mu}(X) \right) \right]
    \end{align*}

    Looking at the terms inside the parentheses, $\frac{g(X)}{d'}\mu(X) - \frac{\hat{g}(X)}{d}\frac{1-g(X)}{1-\hat{g}(X)} \hat{\mu}(X)$, we unify the denominator as follows
    \begin{align*}
        & \frac{g(X)}{d'}\mu(X) - \frac{\hat{g}(X)}{d}\frac{1-g(X)}{1-\hat{g}(X)} \hat{\mu}(X) \\
        & = \frac{g(X)(1-\hat{g}(X))\mu(X)d - \hat{g}(X)(1-g(X))\hat{\mu}(X)d'}{d' (1-\hat{g}(X))d}
    \end{align*}
    We look at only the numerator for now,
    \begin{align*}
        & g(X)(1-\hat{g}(X))\mu(X)d - \hat{g}(X)(1-g(X))\hat{\mu}(X)d' \\
        & = g(X)\mu(X)d - g(X)\hat{g}(X)\mu(X)d - \hat{g}(X)\hat{\mu}(X)d' + g(X)\hat{g}(X)\hat{\mu}(X)d' \\
        & = (g(X)\mu(X)d - \hat{g}(X)\hat{\mu}(X)d') + (g(X)\hat{g}(X)\hat{\mu}(X)d' - g(X)\hat{g}(X)\mu(X)d) \\
        & = g(X)\mu(X)(\gamma-1)\hat{\mu}(X) + g(X)\mu(X) - (\hat{g}(X)\hat{\mu}(X)(\gamma-1)\mu(X) + \hat{g}(X)\hat{\mu}(X)) \\
        & = g(X)\hat{g}(X)\hat{\mu}(X)(\gamma-1)\mu(X) + g(X)\hat{g}(X)\hat{\mu}(X) - (g(X)\hat{g}(X)\hat{\mu}(X)(\gamma-1)\hat{\mu}(X) + g\hat{g}(X)\mu(X)) \\
        & = (\gamma-1)\mu(X)\hat{\mu}(X)(g(X) - \hat{g}(X)) + (g(X)\mu(X) - \hat{g}(X) \hat{\mu}(X)) + g(X)\hat{g}(X)(\hat{\mu}(X) - \mu(X))
    \end{align*}
    Putting this back together with the denominator, we have 
    \begin{align*}
        & \frac{1}{\hat{P}(Y=1,D=0)} E \left[ \frac{\gamma}{d} (\hat{\pi}(X) - \pi(X))(\gamma-1)\mu(X)\hat{\mu}(X)(g(X)-\hat{g}(X))\right] \\
        & + \frac{1}{\hat{P}(Y=1,D=0)} E \left[ \frac{\gamma}{d} (\hat{\pi}(X) - \pi(X)) (g(X)\mu(X) - \hat{g}(X) \hat{\mu}(X)) \right] \\
        & + \frac{1}{\hat{P}(Y=1,D=0)} E \left[ \frac{\gamma}{d} (\hat{\pi}(X) - \pi(X)) g(X)\hat{g}(X)(\hat{\mu}(X) - \mu(X)) \right]
    \end{align*}

    We observe the that first and third terms are in the form of squared differences, so they converge at a rate of $o_P(n^{-\frac{1}{2}})$, when the individual estimators converge at a rate of $o_P(n^{-\frac{1}{4}})$. We observe that the second term scales with $|g(X) - \hat{g}(X)| + |\mu(X) - \hat{\mu}(X)|$ (again by the logic in \cref{sec:plugin}) and is multiplied by $(\hat{\pi}(X) - \pi(X))$, so it converges in $o_{P}(n^{-\frac{1}{2}})$.  
\end{proof}

Now, for the following Lemmas and proofs, we let $A(X) = \frac{\gamma' \pi(X) \mu(X)}{(\gamma' - 1)\mu(X)+1}$ where $\gamma' = \frac{1}{\gamma}$. 

\begin{lemma}
    Let our estimand $\theta_3^{l, \gamma}(P)$ be given by 
    \begin{align}\label{eq:theta_3_l_gamma}
        \theta_3^{l,\gamma}(P) = \frac{1}{P(Y = 1 | D = 0)} E\left[ \frac{\frac{1}{\gamma} \pi(X) \mu(X)}{(\frac{1}{\gamma} - 1) \mu(X) + 1} | D = 0 \right]
    \end{align}
    Let $\gamma' = \frac{1}{\gamma}$. Then, we have that our influence function is given by
    \begin{align*}
    IF(\theta_3^{l, \gamma}) & = -\frac{1[Y=1, D=0]}{P{(Y=1, D=0)}^2} E_P[g(X)A'(X)]  + \frac{g(X)A'(X)}{P(Y=1, D=0)} + \frac{A'(X)(1[D=0]-g(X))}{P(Y=1, D=0)}  \\
        & \qquad + \frac{1[D=1]}{P(Y=1, D=0)} \frac{\gamma' \mu(X)}{((\gamma' - 1)\mu(X) + 1)}(T - \pi(x)) \\
         & \qquad + \frac{1[D=0]}{P(Y=1, D=0)}\frac{1 - g(X)}{g(X)} \frac{\gamma' \pi(X)}{{((\gamma' - 1)\mu(X) + 1)}^2}(Y - \mu(x)) 
\end{align*}    

where $A'(X) = A_{\gamma'}(X)$.
\end{lemma}

\begin{proof}
    This holds via a direct application for the proof for the upper bound with our sensitivity model, except using $\gamma' = \frac{1}{\gamma}$.
\end{proof}

Next, we move on to discussing our estimator of the lower bound, using this influence function. 

\begin{proposition} \label{prop:estimator_lb_gamma}
    Our one-step estimator of $\theta_3^{l, \gamma}$ is given by
    \begin{align*}
        \hat{\theta_3}^{l, \gamma} 
        & =  \frac{1}{\hat{P}(Y=1, D=0)} E_P\left[ \hat{A}'(X)(1[D=0]) \right]\\
        & \qquad + \frac{1}{\hat{P}(Y=1, D=0)} E_P\left[1[D=1] \frac{\gamma' \hat{\mu}(X)}{((\gamma' - 1)\hat{\mu}(X) + 1)}(T - \hat{\pi}(x)) \frac{\hat{g}(X)}{1 - \hat{g}(X)} \right] \\
        & \qquad +\frac{1}{\hat{P}(Y=1, D=0)} E_P\left[ 1[D=0]\frac{\gamma' \hat{\pi}(X)}{{((\gamma' - 1)\hat{\mu}(X) + 1)}^2}(Y - \hat{\mu}(x))  \right] 
    \end{align*}
\end{proposition}

\begin{proof}
This holds via a direct application for the proof for the upper bound with our sensitivity model, except using $\gamma' = \frac{1}{\gamma}$.
\end{proof}

\begin{lemma}[Error of one-step estimator of lower bound with $\gamma$]\label{lemma:error_lb_gamma}
        Let the error of our one-step estimator be given by 
        \begin{equation*}
          R(\hat{P}, P) = \theta_3^{l, \gamma}(\hat{P}) - \theta_3^{l,\gamma}(P) + E_P\left[ IF(\theta_3^{l, \gamma}(\hat{P})) \right]
        \end{equation*}
        Then, we have that
        \begin{align*}
            R(\hat{P}, P) = o_P(n^{-\frac{1}{2}}),
        \end{align*}
        when $(\hat{\pi}  - \pi), (\hat{\mu}  - \mu), (\hat{g}  - g)$ have rates of at least $o_P(n^{-\frac{1}{4}})$. 
\end{lemma}

\begin{proof}
This holds via a direct application for the proof for the upper bound with our sensitivity model, except using $\gamma' = \frac{1}{\gamma}$.
\end{proof}

\section{Margin-based Analysis and Asymptotic Normality of our Estimators} \label{sec:margin}

In this section, we provide the proof for \cref{prop:normality-max}. First, we provide a general analysis of estimating a quantity that consists of a max or min operator, showing that the resulting estimator is asymptotically normal. Next, we show that the individual components of our estimators satisfy the assumptions in our margin analysis, concluding that our resulting estimator is asymptotically normal (i.e., \cref{prop:normality-max}). 

\subsection{Preliminaries and Assumptions for \cref{prop:normality-max}}

Let $W = (X, Y, D, T)$ denote all of our observed variables.
The estimators we introduce in this paper all fit within a common framework, where we consider the general problem of estimating a bound given by either of
\begin{align}
  \psil &\coloneqq E_P\left[\max_{j = 1, \ldots, J} \thetal_{j}(W; P)\right] = E_P\left[ \thetal_{\dl(W)}(W; P) \right] & \dl(W) &\coloneqq \arg\max_{j \in 1, \ldots, J} \thetal_j(W; P) \label{eq:margin_proof_def_psi_lb} \\
\psiu &\coloneqq E_P\left[\min_{j = 1, \ldots, J} \thetau_j(W; P)\right] = E_P\left[ \thetau_{\du(W)}(W; P) \right] & \du(W) &\coloneqq \arg\min_{j \in 1, \ldots, J} \thetau_j(W; P) \label{eq:margin_proof_def_psi_ub}
\end{align}
where the $\thetau_j(W; P),\thetal_j(W;P)$ are individual bounds that can be evaluated at each sample $W$ and we wish to take the pointwise maximum (for our lower bounds) or minimum (for our upper bounds) and then marginalize over $W$.  Note that the set of individual bounds $\theta_j(W; P)$ will differ depending on whether we are estimating upper and lower bounds. Furthermore, we define for each $\theta_j(W; P)$ (regardless of whether it is a lower or upper bound) the corresponding functional
\begin{equation}\label{eq:margin_proof_def_theta}
  \theta_j \coloneqq E_P[\theta_j(W; P)].
\end{equation}
In each of the estimators we consider in this work, we have derived a plug-in estimator for each $\theta_j$, 
\begin{equation}\label{eq:margin_proof_def_plugin}
\hat{\theta}_j \coloneqq E_{\Ph}[\theta_j(W; \Ph)],
\end{equation}
and we similarly have access to a one-step (or ``debiased'') estimator for each $\theta_j$,
\begin{equation}\label{eq:margin_proof_def_onestep}
  \hat{\psi}_j \coloneqq E_{\Ph}[\theta_j(W; \Ph) + \lambda_j(W; \Ph)]
\end{equation}
where $\lambda_j(W; \Ph)$ is the influence function for $\theta_j$ in~\cref{eq:margin_proof_def_theta}, though in the following results we will only require that this additional term satisfies certain conditions (e.g., being zero-mean $E_P[\lambda_j(W; P)] = 0$) and that the plug-in estimator in~\cref{eq:margin_proof_def_plugin} and the one-step estimator in~\cref{eq:margin_proof_def_onestep} converge to $\theta_j$ at certain rates.

Our estimator of the lower bound in~\cref{eq:margin_proof_def_psi_lb} is given by the following, where we introduce the short-hand $\varphi(W; P, d)$
\begin{align}\label{eq:margin_proof_def_lower_bound_estimator}
  \psihl &=  E_{\Ph}\left[\varphil(W; \Ph, \dhl)\right] \\
  \varphil(W; \Ph, \dhl) &\coloneqq \thetal_{\dhl(W)}(W; \Ph) + \lambdal_{\dhl(W)}(W; \Ph)\nonumber \\
\dhl(W) &\coloneqq \arg\max_{j \in 1, \ldots, J} \thetal_j(W; \Ph) \nonumber
\end{align}
and our estimator of the upper bound in~\cref{eq:margin_proof_def_psi_ub} is analogously given by 
\begin{align}\label{eq:margin_proof_def_upper_bound_estimator}
  \psihu &=  E_{\Ph}\left[\varphiu(W; \Ph, \dhu)\right] \\
  \varphiu(W; \Ph, \dhu) &\coloneqq \thetau_{\dhu(W)}(W; \Ph) + \lambdau_{\dhu(W)}(W; \Ph) \nonumber \\
\dhu(W) &\coloneqq \arg\max_{j \in 1, \ldots, J} \thetau_j(W; \Ph).\nonumber
\end{align}
In words, each estimator uses the plug-in estimators $\theta_j(W; \Ph)$ to estimate which bound is tightest at each observation $W$, uses the tightest bound for each observation, and then averages the bias-corrected version of the bound at each $W$ to give the final estimate.

Our goal is to demonstrate that these estimators for the lower bound in~\cref{eq:margin_proof_def_lower_bound_estimator} and for the upper bound in~\cref{eq:margin_proof_def_upper_bound_estimator} are asymptotically normal, and to characterize the resulting asymptotic variance, so that we can provide asymptotically valid confidence bounds.  The main technical challenge is that these estimators are non-smooth, given the presence of the max/min operator.

To do so, we will require a few technical assumptions, which we state in a general form, since they apply equally whether we are considering upper or lower bounds.  Our first assumption states that our estimators are bounded.
\begin{restatable}[Boundedness]{assumption}{AsmpUniformBounded}\label{asmp:margin_proof_uniform_bounded}
  For every $j \in \{1, \ldots, J\}$, $\lambda_j(W; P)$ and $\theta_j(W; P)$ are both uniformly bounded by constants with respect to $n$. 
\end{restatable}

We will also require that for the estimator of each individual component of the bound, the chosen one-step correction $\lambda_j$ has zero mean, which will be satisfied whenever $\lambda$ is derived via an influence function-based debiasing step (related to the fact that influence functions have mean 0). 
\begin{restatable}[Zero-mean Correction Term]{assumption}{AsmpInfZeroMean}\label{asmp:margin_proof_inf_function_zero_mean}
For every $j \in \{1, \ldots, J\}$, $E_P[\lambda_j(W; P)] = 0$.
\end{restatable}

We also require a consistency assumption for the plugin estimator, although no assumption about its rate of convergence is required just yet.
\begin{restatable}[Consistent Plug-in Estimator]{assumption}{AsmpConsistentPlugin}\label{asmp:margin_proof_consistent_plugin}
  For every $j \in \{1\ldots J\}$, $||\hat{\theta}_j - \theta_j|| = o_P(1)$.
\end{restatable}

Finally, we require a technical ``margin'' condition, such that $P$ puts bounded density on the event that $\min_{j \neq \d} \theta_{\d(W)}(W) - \theta_j(W)$ is close to zero, i.e., that there are two near-optimal bounds at a given $W$. 
\begin{restatable}[Margin Condition]{assumption}{AsmpMarginCondition}\label{asmp:margin_proof_margin_condition}
   There exists some $\alpha > 0$ such that 
   \begin{equation*}
     P\left[\min_{j \neq \d(W)} |\theta_{\d(W)}(W) - \theta_j(W)| \leq t\right] \lesssim t^{\alpha}.
   \end{equation*}
\end{restatable}

\begin{restatable}[Independent Samples]{assumption}{AsmpIndependentSamples}\label{asmp:margin_proof_independent_samples}
In~\cref{eq:margin_proof_def_lower_bound_estimator,eq:margin_proof_def_upper_bound_estimator} the expectation is taken with respect to $\Ph_1$, while the estimator $\varphi(W; \Ph_2, \dh)$ uses an independent sample $\Ph_2$.
\end{restatable}

\subsection{Proof of Technical Lemmas for \cref{prop:normality-max}}

To prove \cref{prop:normality-max}, we require \cref{lemma:margin_error_bound} and \cref{thm:margin_theorem}.

\begin{restatable}{lemma}{MarginErrorBound}\label{lemma:margin_error_bound}
  Let Assumptions~\ref{asmp:margin_proof_uniform_bounded},~\ref{asmp:margin_proof_inf_function_zero_mean},~\ref{asmp:margin_proof_consistent_plugin},~\ref{asmp:margin_proof_margin_condition} and~\ref{asmp:margin_proof_independent_samples} hold.  Then 
  \begin{align*}
    \psihl - \psil =& \underbrace{E_{\Ph}[\varphil(W; P, \dl)] - E_P[\varphil(W; P, \dl)]}_{\textcolor{orange}{(a)}{}} \\
                   &+ \underbrace{O_P\left(||\thetahl_j - \thetal_j||_\infty^{1 + \alpha} + \max_{j = 1, \ldots, J}E_P[\thetal_j(W; \Ph) + \lambdal_j(W; \Ph) - \thetal_j(W; P)]\right)}_{\textcolor{orange}{(b)}{}} \\
                   &+ \underbrace{o_P(n^{-\frac{1}{2}})}_{\textcolor{orange}{(c)}{}}
  \end{align*}
  and similarly 
  \begin{align*}
    \psihu - \psiu =& E_{\Ph}[\varphiu(W; P, \du)] - E_P[\varphiu(W; P, \du)] \\
                   &+ O_P\left(||\thetahu_j - \thetau_j||_\infty^{1 + \alpha} + \max_{j = 1, \ldots, J}E_P[\thetau_j(W; \Ph) + \lambdau_j(W; \Ph) - \thetau_j(W; P)]\right) \\
                   &+ o_P(n^{-\frac{1}{2}})
  \end{align*}
  where $\psil, \psiu$ are defined in~\cref{eq:margin_proof_def_psi_lb,eq:margin_proof_def_psi_ub}, and $\psihl,\psihu$ are defined in~\cref{eq:margin_proof_def_lower_bound_estimator,eq:margin_proof_def_upper_bound_estimator}.
\end{restatable}
\begin{proof}
  Throughout the proof, we use $\psi, \varphi$ in place of e.g., $\psiu, \varphiu$ when the proof technique applies equally to either estimator $\psiu, \psil$.  We use $P_n$ and $\Ph$ to denote two independent empirical distributions (per~\cref{asmp:margin_proof_independent_samples}), where the latter is used to estimate the nuisance parameters.  With some abuse of notation, we will occasionally write $\thetah(W) \coloneqq \theta(W; \Ph)$ and $\theta(W) \coloneqq \theta(W; P)$.

  To start, we use the following standard decomposition, with the short-hand $P \varphi(W; \Ph, \dh) \coloneqq E_{P}[\varphi(W; \Ph, \dh)]$, and $(P - P_n)(\cdot) \coloneqq E_{P}[\cdot] - E_{P_n}[\cdot]$. This gives us that
\begin{align*}
\psi - \psih &= P \varphi(W; P, \d) - P_n \varphi(W; \Ph, \dh)\\ 
                   &= (P - P_n)\{\varphi(W; \Ph, \dh) - \varphi(W; P, d)\} + P\{\varphi(W; \P, \d) - \varphi(W; \Ph, \dh)\} + (P - P_n)\{\varphi(W; P, d)\}\\
    &\equiv R_1 + R_2 +  (P - P_n)\{\varphi(W; P, d)\},
\end{align*}
and we proceed by separately bounding $R_1$ and $R_2$. 

\textbf{Part 1:} (Bounding $R_1$) First, we show that under the given conditions, $R_1 = o(n^{-\frac{1}{2}})$.

  We make use of Lemma 2 of~\citet{kennedy2020sharp}, which states that this term is $O_P(\|\varphih - \varphi\| \cdot n^{-1/2})$, where we have used the shorthand $\varphih \coloneqq \varphi(W; \Ph, \dh)$ and $\varphi \coloneqq \varphi(W; P, d)$. Therefore, the entire term is $o(n^{-1/2})$ if the following condition holds:
    \begin{align*}
        E_P\left[{\left(\varphi(W; \Ph, \dh)  - \varphi(W; P, d)\right)}^2\right] = o_P(1).
    \end{align*}
    We first bound 
    \begin{align}
      &E_P\left[{\left(\varphi(W; \Ph, \dh)  - \varphi(W; P, d)\right)}^2\right] \nonumber \\
      &= E_P\left[{\left(\varphi(W; \Ph, \dh) - \varphi(W; P, \dh) + \varphi(W; P, \dh) - \varphi(W; P, d)\right)}^2\right] \nonumber\\
                                                                                        & \lesssim E_P\left[{\left(\varphi(W; \Ph, \dh)  - \varphi(W; P, \dh)\right)}^2\right] +  E_P\left[{\left(\varphi(W; P, \dh)  - \varphi(W; P, d)\right)}^2\right] \label{eq:margin_proof_lemma_1_terms}
    \end{align}
    where we simply add and subtract $\varphi(W; P, \dh)$ in the second line, and the third line follows from the inequality that ${(a + b)}^2 = a^2 + 2ab + b^2 \leq 2(a^2 + b^2)$, since ${(a - b)}^2 \geq 0 \implies 2ab \leq a^2 + b^2$.  Note that we absorb the constant factor into $\lesssim$.  We now bound the two terms on the right-hand side of~\cref{eq:margin_proof_lemma_1_terms}, showing that both are $O_P(1)$. The first term on the right-hand side of~\cref{eq:margin_proof_lemma_1_terms} satisfies 
    \begin{align*}
      E_P[{(\varphi(W; \Ph, \dh)  - \varphi(W; P, \dh))}^2] \leq \sum_{j = 1}^J E_P\left[{\left(\theta_{j}(W; \Ph) + \lambda_j(W; \Ph) - \theta_j(W; P) -  \lambda_j(W; P)\right)}^2\right] = o_P(1)
    \end{align*}
    via consistency of the underlying estimator for each bound in~\cref{asmp:margin_proof_consistent_plugin}. The second term on the right-hand side of~\cref{eq:margin_proof_lemma_1_terms} satisfies
    \begin{align*}
        E_P\left[{\left(\varphi(W; P, \dh)  - \varphi(W; P, d)\right)}^2\right]  &=\sum_{j = 1}^J E_P\left[\left\lvert 1[\dh(W) = j] - 1[\d(W) = j]\right\rvert {\left(\lambda_j(W; P) + \theta_j(W)\right)}^2\right]\\
        &\lesssim P(\theta_{\dh(W)} \neq \theta_{\d(W)})
    \end{align*}
    where we write $\theta_{\d(W)} \coloneqq \theta_{\d(W)}(W; P)$ to simplify notation, and 
    where the last step uses that $\theta_j$ and $\lambda_j$ are uniformly bounded per~\cref{asmp:margin_proof_uniform_bounded}, and that $J$ is fixed. Next, we will show that $ P(\theta_{\dh(W)} \neq \theta_{\d(W)}) = o_P(1)$ by using the consistency of $\dh$, combined with the margin condition from~\cref{asmp:margin_proof_margin_condition}. For any $t > 0$, we have that
    \begin{align}
        P(\theta_{\dh(W)} \neq \theta_{\d(W)}) &=  P\left(\theta_{\dh(W)} \neq \theta_{\d(W)}, \min_{j \neq \d(W)}|\theta_{\d(W)} - \theta_{j}| \leq t\right) \nonumber \\ 
                                               &+ P\left(\theta_{\dh(W)} \neq \theta_{\d(W)}, \min_{j \neq \d(W)}|\theta_{\d(W)} - \theta_{j}| > t\right)\nonumber \\
                                               &\leq P\left(\min_{j \neq \d(W)}|\theta_{\d(W)} - \theta_{j}| \leq t\right) + P\left(|\theta_{\d(W)} - \theta_{\dh(W)}| > t\right)\label{eq:margin_proof_margin_terms}
    \end{align}
    where the last line uses that whenever $\theta_{\dh(W)} \neq \theta_{\d(W)}$, it must hold that $\dh(W) \neq \d(W)$ and hence $\abs{\theta_{\dh(W)} - \theta_{\d(W)}} \geq \min_{j \neq \d(W)}|\theta_{\d(W)} - \theta_{j}|$.

    Note that, if we are considering $\dl$, then $\theta_{\dl(W)} - \theta_{\dhl(W)} \geq 0$, since we take a maximum over $\theta_j$ when considering $\dl$, and similarly $\thetah_{\dhl(W)} - \thetah_{\dl(W)} \geq 0$, since $\dhl$ considers the maximum over $\thetah_j$.  As a result, we can write that 
    \begin{align}
      \abs{\theta_{\dl(W)} - \theta_{\dhl(W)}} &= \theta_{\dl(W)} - \theta_{\dhl(W)} \nonumber \\
                                             &\leq \theta_{\dl(W)} - \thetah_{\dl(W)} + \thetah_{\dhl(W)} - \theta_{\dhl(W)}\nonumber \\
                                             &\leq \abs{\theta_{\dl(W)} - \thetah_{\dl(W)}} + \abs{\thetah_{\dhl(W)} - \theta_{\dhl(W)}} \label{eq:margin_proof_bound_dl}
    \end{align}
    and if we are considering $\du$, then $\theta_{\du(W)} - \theta_{\dhu(W)} \leq 0$, and $\thetah_{\dhu(W)} - \thetah_{\du(W)} \leq 0$, and by similar logic
    \begin{align}
      \abs{\theta_{\du(W)} - \theta_{\dhu(W)}} &= \theta_{\dhu(W)} - \theta_{\du(W)} \nonumber \\
                                             &\leq \theta_{\dhu(W)} - \thetah_{\dhu(W)} + \thetah_{\du(W)} - \theta_{\du(W)} \nonumber \\
                                             &\leq \abs{\theta_{\du(W)} - \thetah_{\du(W)}} + \abs{\thetah_{\dhu(W)} - \theta_{\dhu(W)}} \label{eq:margin_proof_bound_du} 
    \end{align}
    Returning to~\cref{eq:margin_proof_margin_terms}, let $C$ be the universal constant from the margin condition in~\cref{asmp:margin_proof_margin_condition}.
    Using the margin condition, and our reasoning above, coupled with the fact that for $a \leq b$, $P(a > t) \leq P(b > t)$, we continue to bound as follows  
    \begin{align*}
         P(\theta_{\dh(W)} \neq \theta_{\d(W)}) &\leq Ct^\alpha + P\left(|\theta_{\d(W)} - \hat{\theta}_{\d(W)}(W)| + |\hat{\theta}_{\dh(W)}(X) - \theta_{\dh(W)}| > t\right)\\
         &\leq Ct^\alpha + P\left(\sum_{j = 1}^J 2|\theta_j - \hat{\theta}_j| > t\right) \\
         &\leq Ct^\alpha + \frac{2}{t} \sum_{i = 1}^J E_P[|\theta_j - \hat{\theta}_j|] \quad \text{(using Markov's inequality and linearity of expectation)}\\
         &\leq Ct^\alpha + \frac{2}{t}\sum_{i = 1}^J ||\hat{\theta}_j - \theta_j||_2.
    \end{align*}

    Now, we obtain the desired result by using consistency of the underlying plug-in estimators (\cref{asmp:margin_proof_consistent_plugin}) that for each $j$,  $||\hat{\theta}_j - \theta_j||_2 = o_P(1)$. For any $\epsilon > 0$, set $t_\epsilon = {\left(\frac{\epsilon}{C}\right)}^{\frac{1}{\alpha}}$ so that $Ct_{\epsilon}^\alpha = \epsilon$. Next, define the sequences $X_n = P(\theta_{\dh(W)} \neq \theta_{\d(W)})$ and $Z_n^\epsilon = \frac{2}{t_\epsilon}\sum_{i = 1}^J ||\hat{\theta}_j - \theta_j||_2$. Since $|X_n| \leq \epsilon + Z_n^\epsilon$ and $Z_n^\epsilon = o_P(1)$, $X^n = o_P(1)$ as well, concluding the proof of the bound on $R_1$.
    
\textbf{Part 2:} (Bounding $R_2$) We now show that 
\begin{equation*}
|R_2|  = O_P\left(\max_{j=1, \ldots, J} ||\hat{\theta}_j - \theta_j||_\infty^{1 + \alpha} + \max_{j = 1,\ldots, J}E_P[\theta_{j}(W; \Ph) + \lambda_j(W; \Ph) - \theta_j(W; P)]\right)
\end{equation*}

Our goal is to bound $R_2 \equiv E_P[\varphi(W; \Ph, \dh) - \varphi(W; P, d)]$. We decompose this as 
\begin{align}
    R_2 =  E_P\left[\lambda_{\dh(W)}(W; \Ph) + \theta_{\dh(W)}(W; \Ph) - \theta_{\dh(W)}(W; P)\right] + E_P\left[\theta_{\dh(W)}(W; P) - \theta_{\d(W)}(W; P)\right]\label{eq:margin_proof_lemma_2_terms}
\end{align}
noting that $E_P[\lambda_j(W; P)] = 0$ by~\cref{asmp:margin_proof_inf_function_zero_mean}. For the first term of~\cref{eq:margin_proof_lemma_2_terms},  we have that
\begin{align*}
     \abs{E_P\left[\lambda_{\dh(W)}(W; \Ph) + \theta_{\dh(W)}(W; \Ph) - \theta_{\dh(W)}(W; P)\right]} &\leq \sum_{j = 1}^J \abs{E_P\left[\lambda_{j}(W; \Ph) + \theta_{j}(W; \Ph) - \theta_{j}(W; P)\right]}\\
     &\lesssim \max_{j = 1,\ldots,J} \abs{E_P\left[\lambda_{j}(W; \Ph) + \theta_{j}(W; \Ph) - \theta_{j}(W; P)\right]},
\end{align*}
which gives us the second term in the desired expression for $\abs{R_2}$.
For the second term of~\cref{eq:margin_proof_lemma_2_terms}, we use the margin condition. 
First, since this difference is equal to zero whenever $\dh(W) = \d(W)$, we can write the absolute value of this expression as 
\small
\begin{align}
    &\left\lvert E_P[\theta_{\dh(W)}(W; P) - \theta_{\d(W)}(W; P)] \right\rvert \nonumber \\
    &\qquad \leq E_P\left[\abs{\theta_{\dh(W)}(W; P) - \theta_{\d(W)}(W; P)}\right]\nonumber \\
    &\qquad = E_P\left[1[\d(W)\neq \dh(W)] \cdot \abs{\theta_{\d(W)}(W; P) - \theta_{\dh(W)}(W; P)}\right] \nonumber \\
    &\qquad = E_P\left[1\left[\min_{j \neq \d(W)}\abs{\theta_{\d(W)}(W; P) - \theta_{j}(W; P)} \leq \abs{\theta_{\d(W)}(W; P) - \theta_{\dh(W)}(W; P)}\right] \cdot \abs{\theta_{\d(W)}(W; P) - \theta_{\dh(W)}(W; P)}\right] \label{eq:margin_proof_final_term}
\end{align}
\normalsize
where the indicator follows from the simple fact that $\dh(W) \neq \d(W)$. Recall from~\cref{eq:margin_proof_bound_dl} and~\cref{eq:margin_proof_bound_du} that regardless of whether we are using $\dl,\du$, we can write that 
\begin{align}
  \abs{\theta_{\d(W)}(W; P) - \theta_{\dh(W)}(W; P)} &\leq \abs{\theta_{\d(W)}(W; P) - \theta_{\d(W)}(W; \Ph)} + \abs{\theta_{\dh(W)}(W; P) - \theta_{\dh(W)}(W; \Ph)} \nonumber \\
                                                     &\leq 2 \max_{j = 1, \ldots, J} \norm{\theta_{j}(W; P) - \theta_{j}(W; \hat{P})}_{\infty}\label{eq:margin_proof_general_dldu_bound}
\end{align}
Moreover, we can observe that
\begin{equation}
  Y \leq Z \implies E_P[1[X \leq Y] \cdot Y] \leq E_P[1[X \leq Z] \cdot Z]. \label{eq:margin_proof_general_rule}
\end{equation}

Putting it all together, we observe that~\cref{eq:margin_proof_final_term}, combined with~\cref{eq:margin_proof_general_dldu_bound} and~\cref{eq:margin_proof_general_rule}, gives us the desired result, where we use the shorthand $\theta_{j} \coloneqq \theta_{j}(W; P)$ and $\thetah_{j} \coloneqq \theta_{j}(W; \Ph)$ for simplicity
\begin{align*}
  \left\lvert E_P[\theta_{\dh(W)}(W; P) - \theta_{\d(W)}(W; P)] \right\rvert &\leq E_P\left[1\left[\min_{j \neq \d(W)}\abs{\theta_{\d(W)} - \theta_{j}} \leq 2 \max_{j = 1, \ldots, J} \norm{\theta_{j} - \thetah_{j}}_{\infty}\right]\cdot 2 \max_{j = 1, \ldots, J} \norm{\theta_{j} - \thetah_{j}}_{\infty} \right]\\
                                                                             &=  P\left(\min_{j \neq \d(W)}\abs{\theta_{\d(W)} - \theta_{j}} \leq 2 \max_{j = 1, \ldots, J} \norm{\theta_{j} - \thetah_{j}}_{\infty}\right)\cdot 2 \max_{j = 1, \ldots, J} \norm{\theta_{j} - \thetah_{j}}_{\infty} \\
                                                                             &\lesssim   \max_{j = 1,\ldots,J}||\hat{\theta}_j - \theta_j||_\infty^{1 + \alpha}
\end{align*}
using the margin condition. Combining the bounds on the two individual components of $R_2$ yields the result. 

We can now plug in the bounds derived in Parts 1 and 2 (for $R_1$ and $R_2$) to conclude our result for \cref{lemma:margin_error_bound}. 
\end{proof}

Now, to show that an estimator is asymptotically normal (\cref{thm:margin_theorem}), we make two more assumptions, which require that each of the one-step estimators converge at a sufficiently fast rate. 

\begin{restatable}{assumption}{AsmpPluginConvergence}\label{asmp:margin_proof_plugin_convergence_rate}
$||\hat{\theta_j} - \theta_j||_\infty^{1 + \alpha} = o_P(n^{-\frac{1}{2}})$, where $\alpha$ is defined in~\cref{asmp:margin_proof_margin_condition}
\end{restatable}
\begin{restatable}{assumption}{AsmpOneStepConvergence}\label{asmp:margin_proof_one_step_convergence_rate}
  For each $j \in \{1\ldots J\}$, $E_P[\theta_j(W; \Ph) + \lambda_j(W; \Ph) - \theta_j(W; P)] = o_P(n^{-\frac{1}{2}})$ 
\end{restatable}

\begin{restatable}{lemma}{MarginTheorem}\label{thm:margin_theorem}
  Under the conditions of~\cref{lemma:margin_error_bound}, as well as Assumptions~\ref{asmp:margin_proof_plugin_convergence_rate} and~\ref{asmp:margin_proof_one_step_convergence_rate}, then 
   \begin{equation*}
       \sqrt{n}\left(\psihl - \psil\right) \to N(0, \text{Var}(\varphil(W; P, \dl)))
   \end{equation*}
   and 
   \begin{equation*}
       \sqrt{n}\left(\psihu - \psiu\right) \to N(0, \text{Var}(\varphiu(W; P, \du)))
   \end{equation*}
\end{restatable}

\begin{proof}

    Under Assumptions~\ref{asmp:margin_proof_plugin_convergence_rate} and~\ref{asmp:margin_proof_one_step_convergence_rate}, the second and third terms (\textcolor{orange}{(b)} and \textcolor{orange}{(c)}) in \cref{lemma:margin_error_bound} vanish asymptotically at fast rates, so only the first term \textcolor{orange}{(a)} remains. The desired result directly follows from an application of the Central Limit Theorem on the remaining first term.
\end{proof}

Thus, we have shown that a general estimator is asymptotically normal, given that it satisfies the aformentioned assumptions. We will now show that our estimators satisfy these assumptions.

\subsection{Verifying Assumptions for Our Estimators}\label{sec:asymp_normal}

\subsubsection{Setup and Notation}

First, let $W = (X, Y, D, T)$ denote all observed variables, as in the previous section. Let $\theta_1(W; P)$ be defined as the quantity that gives us the upper bound based on the partial identification, i.e., 
\begin{equation}
E_{P}[\thetau_1(W; P)] = \psiu
\end{equation}
and let $\theta_2(W) = 1$, the constant function.  Furthermore, let $\theta_3(W; P)$ be defined similarly as the quantity that gives us the upper bound based on $\gamma$, i.e., 
\begin{equation}
E_{P}[\thetau_3(W; P)] = \psiu_{\gamma}
\end{equation}
where $J = 2$ for the partial identification bound, and $J = 3$ for the bound that includes $\gamma$.  Furthermore, we define our estimators as follows 

\begin{align}
   \thetau_1(W; \hat{P}) &\coloneqq \frac{1}{\hat{P}(Y=1, D=0)} \hat{g}(X) \hat{\pi}(X)  \\
   \thetal_1(W; \hat{P}) &\coloneqq  \frac{1}{\hat{P}(Y=1, D=0)} \hat{g}(X) (\hat{\pi}(X) + \hat{\mu}(X) - 1)\\
   \lambdau_1(W; \hat{P}) &\coloneqq \frac{1}{\hat{P}(Y=1, D=0)} \Bigg( -\frac{1[Y=1, D=0]}{\hat{P}(Y=1, D=0)} E_{\hat{P}} [\hat{g}(X) \hat{\pi}(X)] + \hat{g}(X)\hat{\pi}(X) \\
   & \qquad \qquad + 1[D=1](T - \hat{\pi}(X)) \frac{\hat{g}(X)}{1 - \hat{g}(X)}\Bigg) \\
   \lambdal_1(W; \hat{P}) & \coloneqq \lambdau_1(W; \hat{P}) + \frac{1}{\hat{P}(Y=1, D=0)} \Bigg(  -\frac{1[Y=1, D=0]}{\hat{P}(Y=1, D=0)} E_{\hat{P}} [\hat{g}(X) (\hat{\mu}(X) - 1)] \\
   & \qquad \qquad + 1[D=0](Y - 1)\Bigg)
\end{align}
and, with letting 
\begin{equation*}
    \hat{A}_{\gamma}(X) = \frac{\gamma \hat{\mu}(X) \hat{\pi}(X)}{(\gamma - 1) \hat{\mu}(X) + 1}
\end{equation*}
we have that for the sensitivity model bounds,
\begin{align}
   \thetau_3(W; \hat{P}) &\coloneqq  \frac{1}{\hat{P}(Y=1, D=0)} \hat{g}(X)\left(\hat{A}_{\gamma}(X) \right) \\
   \thetal_3(W; \hat{P}) &\coloneqq \frac{1}{\hat{P}(Y=1, D=0)} \hat{g}(X)\left( \hat{A}_{\frac{1}{\gamma}}(X)\right) \\
   \lambdau_3(W; \hat{P}) &\coloneqq - \frac{1[Y=1, D=0]}{\hat{P}(Y=1, D=0)^2} E_{\hat{P}}[\hat{g}(X) \hat{A}_\gamma(X)] + \frac{\hat{g}(X) \hat{A}_\gamma(X)}{\hat{P}(Y=1, D=0)} + \frac{\hat{A}_\gamma(X) (1[D=0] - \hat{g}(X))}{\hat{P}(Y=1, D=0)} \\
   & \qquad \qquad + \frac{1[D=1]}{\hat{P}(Y=1, D=0)} \frac{\gamma \hat{\mu}(X)}{(\gamma - 1) \hat{\mu}(X) + 1} (T - \hat{\pi}(X)) \frac{\hat{g}(X)}{1 - \hat{g}(X)} \\
   & \qquad \qquad +  \frac{1[D=0]}{\hat{P}(Y=1, D=0)} \frac{\gamma \hat{\pi}(X)}{((\gamma - 1) \hat{\mu}(X) + 1)^2} (Y - \hat{\mu}(X)) \\
   \lambdal_3(W; \hat{P}) & \coloneqq - \frac{1[Y=1, D=0]}{\hat{P}(Y=1, D=0)^2} E_{\hat{P}}[\hat{g}(X) \hat{A}_{\frac{1}{\gamma}}(X)] + \frac{\hat{g}(X) \hat{A}_{\frac{1}{\gamma}}(X)}{\hat{P}(Y=1, D=0)} + \frac{\hat{A}_{\frac{1}{\gamma}}(X) (1[D=0] - \hat{g}(X))}{\hat{P}(Y=1, D=0)} \\
   & \qquad \qquad + \frac{1[D=1]}{\hat{P}(Y=1, D=0)} \frac{\frac{1}{\gamma} \hat{\mu}(X)}{(\frac{1}{\gamma} - 1) \hat{\mu}(X) + 1} (T - \hat{\pi}(X)) \frac{\hat{g}(X)}{1 - \hat{g}(X)} \\
   & \qquad \qquad +  \frac{1[D=0]}{\hat{P}(Y=1, D=0)} \frac{\frac{1}{\gamma} \hat{\pi}(X)}{((\frac{1}{\gamma} - 1) \hat{\mu}(X) + 1)^2} (Y - \hat{\mu}(X)) \\
\end{align}
Then our desired quantity to estimate, and the combined estimator, is defined as in the previous section (see~\cref{eq:margin_proof_def_lower_bound_estimator,eq:margin_proof_def_upper_bound_estimator}).

Our goal is now to demonstrate that the conditions of~\cref{thm:margin_theorem} hold.  If so, then we can conclude that our estimator is asymptotically normal, with variance given by $\text{Var}(\varphi(W, P, d))$, which we can in turn estimate from data. Let us discuss each condition in turn.

\subsubsection{Verifying Assumptions for the Partial Identification Bound}

Throughout, we will assume that there exists some $\alpha$ such that~\cref{asmp:margin_proof_margin_condition} holds.  With this in hand, we will verify that the remaining assumptions of~\cref{thm:margin_theorem} hold for our estimators. For each assumption, we re-state the assumption for ease of reading, then discuss whether or not it is satisfied in our case.

\AsmpUniformBounded*

In all cases, this assumption holds. For each $\theta_j$ and $\lambda_j$, we have a denominator that contains $\hat{P}(Y=1, D=0)$. While this can be zero, we assume that our dataset contains instances of $Y = 1$ in our pre-treatment dataset, which makes this value nonzero. Similarly, we also have that $1 - \hat{g}(X)$ is in the denominator as well; given that our observed training data for our models has non-zero support on pre-treatment data, this will also be greater than zero.

\AsmpInfZeroMean*

In our error analysis, we have shown that the correction functions that we derived have error terms that are second order in differences in quantities estimated on $\hat{P}$ and $P$. Therefore, by an application of the results in the work of \citet{kennedy2021semiparametric}, we have that our correction functions are efficient influence functions (and that our usage of the discretization trick in deriving this influence functions is valid). As influence functions have zero mean, this assumption is satisfied.

\AsmpConsistentPlugin*
This assumption is directly implied by~\cref{asmp:margin_proof_one_step_convergence_rate} below, so we defer discussion until then.

\AsmpMarginCondition*
As discussed above, we assume this condition, rather than verifying it directly, since it depends on the underlying data-generating process.

\AsmpIndependentSamples*
As we use cross-fitting to estimate our nuisance functions, this assumption holds by construction.

\AsmpPluginConvergence*

For our values of $\theta_j$, we have that our estimates converge at a rate of $o_P(n^{-\frac{1}{2}})$. We can consider an estimate of $\thetau_1(W; \hat{P})$. This is given by
\small
\begin{align*}
    E_{\hat{P}}\left[ \frac{\hat{g}(X)\hat{\pi}(X)}{\hat{P}(Y=1, D=0)}\right] - E_P \left[\frac{g(X)\pi(X)}{P(Y=1, D=0)} \right] & = \sqrt{\frac{\sigma^2}{n}} + \left| E_{P}\left[ \frac{\hat{g}(X)\hat{\pi}(X)}{\hat{P}(Y=1, D=0)}\right] - E_P \left[\frac{g(X)\pi(X)}{P(Y=1, D=0)} \right] \right|
\end{align*}
\normalsize
where $\sigma^2$ is the variance of our estimator, following the steps in \cref{sec:plugin}. Given that our estimator is bounded and, thus, has finite variance, we observe that the variance term has a rate of $o_P(n^{-\frac{1}{2}})$. The remaining error term is on the same order as the sum of the individual estimators' error terms. 
In other words, $\hat{\pi}$ and $\hat{g}$ must converge at a rate of $o_P(n^{-\frac{1}{2(1 + \alpha)}})$, which is easily satisfied with $\alpha = 1$.

\AsmpOneStepConvergence*

We have shown in \cref{lemma:error_ub} and \cref{lemma:error_lb} that given estimators of $\mu, \pi, g$ that converge at rates of $o_P(n^{-\frac{1}{4}}),$ then our one step corrected estimator converges at a $o_P(n^{-\frac{1}{2}})$ rate.

\subsubsection{Verifying Assumptions for the Sensitivity Analysis Bound}

We repeat the same discussion for our sensitivity analysis bound. Throughout, we will assume that there exists some $\alpha$ such that~\cref{asmp:margin_proof_margin_condition} holds.  With this in hand, we will verify the remaining assumptions of~\cref{thm:margin_theorem}. For each assumption, we re-state the assumption for ease of reading, then discuss whether or not it is satisfied in our case.

\AsmpUniformBounded*

In all cases, this assumption holds. For each $\theta_j$ and $\lambda_j$, we have a denominator that contains $\hat{P}(Y=1, D=0)$. While this can be zero, we assume that our dataset contains instances of $Y = 1$ in our pre-treatment dataset, which makes this value nonzero. Similarly, we also have that $1 - \hat{g}(X)$ is in the denominator as well; given that our observed training data for our models has non-zero support on pre-treatment data, this will also be greater than zero. In our sensitvity analysis bound, we have an additional term of $\frac{1}{(\gamma - 1) \hat{\mu}(X) + 1}$, but this is always greater than 0 because of the 1 that is added in the denominator.

\AsmpInfZeroMean*

In our error analysis, we have shown that the correction functions that we derived have error terms that are second order in differences in quantities estimated on $\hat{P}$ and $P$. Therefore, by the results in the work of \citet{kennedy2021semiparametric}, we have that our correction functions are efficient influence functions. As influence functions have zero mean, this assumption is satisfied.

\AsmpConsistentPlugin*
This assumption is directly implied by~\cref{asmp:margin_proof_one_step_convergence_rate} below, so we defer discussion until then.

\AsmpMarginCondition*
As discussed above, we assume this condition, rather than verifying it directly, since it depends on the underlying data-generating process.

\AsmpIndependentSamples*
As we use cross-fitting to estimate our nuisance functions, this assumption holds by construction.

\AsmpPluginConvergence*

For our values of $\theta_j$, we have that our estimates converge at a rate of $o_P(n^{-\frac{1}{2}})$. We can consider an estimate of $\thetau_1(W; \hat{P})$. This is given by
\small
\begin{align*}
    E_{\hat{P}}\left[ \frac{\hat{g}(X)\hat{\pi}(X)}{\hat{P}(Y=1, D=0)}\right] - E_P \left[\frac{g(X)\pi(X)}{P(Y=1, D=0)} \right] & = \sqrt{\frac{\sigma^2}{n}} + \left| E_{P}\left[ \frac{\hat{g}(X)\hat{A}_{\gamma}(X)}{\hat{P}(Y=1, D=0)}\right] - E_P \left[\frac{g(X)A_{\gamma}(X)}{P(Y=1, D=0)} \right] \right|
\end{align*}
\normalsize
where $\sigma^2$ is the variance of our estimator, following the steps in \cref{sec:plugin}. Given that our estimator is bounded and, thus, has finite variance, we observe that the variance term has a rate of $o_P(n^{-\frac{1}{2}})$. The remaining error term is on the same order of the sum of the individual estimators' error terms. 
We can again argue that $\hat{A}$ is a plugin estimator for $A$, which gives us that it also converges at a sum of the rates of $\hat{\mu}$ and $\hat{\pi}$. 
In other words, these must converge at a rate of $o_P(n^{-\frac{1}{2(1 + \alpha)}})$, which is easily satisfied.

\AsmpOneStepConvergence*

We have shown in \cref{lemma:error_ub_gamma} and \cref{lemma:error_lb_gamma} that given estimators of $\mu, \pi, g$ that converge at rates of $o_P(n^{-\frac{1}{4}}),$ then our one step corrected estimator converges at a $o_P(n^{-\frac{1}{2}})$ rate.

\subsection{Proof of \cref{prop:normality-max}}

Having proven the required technical lemmas and having demonstrated that our estimators indeed satisfy the required assumptions, we can now derive \cref{prop:normality-max}.

\Normality*

\begin{proof}

    We can directly apply \cref{lemma:margin_error_bound} and \cref{thm:margin_theorem}, given that our estimators satisfy all the given assumptions (as discussed in \cref{sec:asymp_normal}), to prove this result.
\end{proof}

\section{Dataset and Cohort Details}\label{cohort}

\subsection{Dataset Consent and Acknowledgement Statement}
The analyses described in this publication were conducted with data or tools accessed through the NCATS N3C Data Enclave \url{https://covid.cd2h.org} and N3C Attribution \& Publication Policy \texttt{v1.2-2020-08-25b} supported by NCATS \texttt{U24 TR002306}, Axle Informatics Subcontract: \texttt{NCATS-P00438-B}. This research was possible because of the patients whose information is included within the data and the organizations (\url{https://ncats.nih.gov/n3c/resources/data-contribution/data-transfer-agreement-signatories}) and scientists who have contributed to the on-going development of this community resource: \url{https://doi.org/10.1093/jamia/ocaa196}.

\subsection{Cohort Details}
\label{cohort_details}
The patient cohort is filtered out based on the following eligibility requirements:
\begin{itemize}
    \item Satisfy all FDA-approved Paxlovid eligibility requirements~\citep{fdafile}
    \item Not taking any medications, where coadministration with Nirmatralvir-Ritonavir is contraindicated~\citep{marzolini2022recommendations, larkin2022paxlovid}
    \item First COVID-19 diagnosis visits are between 22 December 2021 
    (date of FDA approval for Paxlovid) and 31 May 2023
    \item From sites with at least a 10\% treatment rate---to 
    exclude sites where treatment is potentially underreported.
\end{itemize}

This filtering results in our final cohort of 87,584 COVID-positive patients from the pre-availability period and 77,412 patients from the post-availability period. The covariates $X$ used are: age, presence of diabetes, presence of obesity, history of congestive heart failure, history of chronic obstructive pulmonary disease, history of peripheral vascular disease, and history of myocardial infarction.

The base treatment rates and pre-availability baseline mortality rates for each group in the cohort are as follows:
\begin{table}[h!]
\centering
\begin{tabular}{lcc}
\toprule
Group & Base Treatment Rate & Baseline Mortality Rate \\
\midrule
White Patients & 0.127445 & 0.012054 \\
Black Patients & 0.080853 & 0.017073 \\
Asian Patients & 0.156817 & 0.009747 \\
\bottomrule
\end{tabular}
\caption{Base Treatment Rates and Pre-Availability Baseline Mortality Rates per Group}
\label{tab:rates}
\end{table}

\end{document}